\documentclass[11pt]{article}
\usepackage[toc,page]{appendix}

\usepackage[]{amsmath}
\usepackage[]{amsthm}
\usepackage{mathtools}
\numberwithin{equation}{section}
\usepackage[]{amssymb}
\usepackage[margin=1in]{geometry}
\usepackage[]{soul}
\usepackage[]{bbm}
\usepackage[]{cancel}

\usepackage[]{enumitem}
\usepackage{mathrsfs}

\usepackage{graphicx}
\usepackage{subfigure}

\usepackage{booktabs}

\usepackage{hyperref}
\hypersetup{
    pdfauthor={Shakil Rafi},
    pdftitle={Dissertation},
    pdfkeywords={neural-networks, stochastic-processes},
    colorlinks = true,
    filecolor = magenta,
    urlcolor = cyan
}

\usepackage[capitalise]{cleveref}
\usepackage{natbib}
\usepackage{neuralnetwork}
\usepackage{witharrows}
\usepackage{stmaryrd}
\usepackage{stackengine,amssymb,graphicx}

\usepackage{graphicx}

\DeclareMathAlphabet{\mymathbb}{U}{BOONDOX-ds}{m}{n}

\DeclareMathOperator{\unif}{Unif}

\usepackage{tikz-cd}

\DeclareMathOperator{\rect}{\mathfrak{r}}
\DeclareMathOperator{\param}{\mathsf{P}}
\DeclareMathOperator{\inn}{\mathsf{I}}
\DeclareMathOperator{\out}{\mathsf{O}}
\DeclareMathOperator{\neu}{\mathsf{NN}}
\DeclareMathOperator{\hid}{\mathsf{H}}
\DeclareMathOperator{\lay}{\mathsf{L}}
\DeclareMathOperator{\dep}{\mathsf{D}}
\DeclareMathOperator{\we}{Weight}
\DeclareMathOperator{\bi}{Bias}
\DeclareMathOperator{\aff}{\mathsf{Aff}}
\DeclareMathOperator{\act}{\mathfrak{a}}
\DeclareMathOperator{\real}{\mathfrak{I}}
\DeclareMathOperator{\id}{\mathsf{Id}}

\DeclareMathOperator{\wid}{\mathsf{W}}
\DeclareMathOperator{\sm}{\mathsf{Sum}}

\DeclareMathOperator{\tun}{\mathsf{Tun}}
\DeclareMathOperator{\cpy}{\mathsf{Cpy}}

\DeclareMathOperator{\relu}{\mathsf{ReLU}}

\DeclareMathOperator{\sqr}{\mathsf{Sqr}}
\DeclareMathOperator{\prd}{\mathsf{Prd}}
\DeclareMathOperator{\pwr}{\mathsf{Pwr}}
\DeclareMathOperator{\xpn}{\mathsf{Xpn}}

\DeclareMathOperator{\tay}{\mathsf{Tay}}

\DeclareMathOperator{\nrm}{\mathsf{Nrm}}
\DeclareMathOperator{\mxm}{\mathsf{Mxm}}
\DeclareMathOperator{\trp}{\mathsf{Trp}}
\DeclareMathOperator{\etr}{\mathsf{Etr}}

\DeclareMathOperator{\csn}{\mathsf{Csn}}
\DeclareMathOperator{\sne}{\mathsf{Sne}}

\DeclareMathOperator{\pnm}{\mathsf{Pnm}}
\DeclareMathOperator{\inst}{\mathfrak{I}}

\newcommand{\R}{\mathbb{R}}
\newcommand{\Z}{\mathbb{Z}}
\newcommand{\N}{\mathbb{N}}

\newcommand{\lp}{\left(}
\newcommand{\rp}{\right)}
\newcommand{\rb}{\right]}
\newcommand{\lb}{\left[}

\newcommand{\ve}{\varepsilon}
\newcommand{\les}{\leqslant}
\newcommand{\ges}{\geqslant}

\newcommand{\DDiamond}{%
  \begin{tikzpicture}[baseline={(0,-0.2ex)}]
    \draw[rotate=45] (0,0) rectangle (0.2,0.2);
    \draw[rotate=45] (0,0.2) -- (0.2,0);
  \end{tikzpicture}%
}

\newtheorem{theorem}{Theorem}[section]
\newtheorem{corollary}{Corollary}[theorem]
\newtheorem{lemma}[theorem]{Lemma}
\newtheorem{definition}[theorem]{Definition}
\newtheorem{remark}[theorem]{Remark}

\begin{document}
	\title{Towards an Algebraic Framework For Approximating Functions Using Neural Network Polynomials}
	\author{
    Shakil Rafi\textsuperscript{1,2} \\
    \footnotesize{\textsuperscript{1} Department of Mathematical Sciences, University of Arkansas}\\
    \footnotesize{\textsuperscript{2} Department of Data Science, Sam M. Walton College of Business, University of Arkansas}\\
    \small{Fayetteville, AR, USA, e-mail: \texttt{sarafi@uark.edu}}
    \and
    Joshua Lee Padgett\textsuperscript{3} \\
    \small{Toyota Financial Services, Plano, TX, 75024}\\
    \small{USA, e-mail: \texttt{josh.padgett@toyota.com}}
    \and
    Ukash Nakarmi\textsuperscript{4} \\
    \small{Department of Computer Science and Computer Engineering, University of Arkansas}\\
    \small{Fayetteville, AR, USA. e-mail: \texttt{unakarmi@uark.edu}}
}

	\maketitle
\begin{abstract}
We make the case for neural network objects and extend an already existing neural network calculus explained in detail in Chapter 2 on \cite{bigbook}. Our aim will be to show that, yes, indeed, it makes sense to talk about neural network polynomials, neural network exponentials, sine, and cosines in the sense that they do indeed approximate their real number counterparts subject to limitations on certain of their parameters, $q$, and $\ve$. While doing this, we show that the parameter and depth growth are only polynomial on their desired accuracy (defined as a 1-norm difference over $\mathbb{R}$), thereby showing that this approach to approximating, where a neural network in some sense has the structural properties of the function it is approximating is not entire intractable.
\end{abstract}
\tableofcontents
\section{Introduction and Motivation}
\label{submission}

This paper represents the first in an series of papers that the authors have undertaken to provide a unified framework for neural network objects. Whereas neural networks have shown great promise in applications as far-reaching as protein-folding in \cite{tsaban_harnessing_2022}, gravitational wave detection \cite{zhao_space-based_2023}, and knot theory in \cite{davies_signature_2021}, a growing need to understand neural networks as first-class mathematical objects is still needed. 

This paper thus follows in the footsteps of a body of research going back to \cite{mcculloch_logical_1943}, \cite{cybenko_approximation_1989}, \cite{Hornik1991ApproximationCO}, \cite{Lagaris_1998}, and more recently \cite{KNOKE2021100035}. 

Our approach differs from standard neural network orthodoxy, where a neural network is essentially seen to be constructing piecewise functions approximating a function given sample points. In this sense, these neural networks are extrapolants and somewhat ``blind'' to the function underneath. Our neural networks are, in a concrete sense, structurally the same as the functions they approximate.

Indeed, our approach envisions fully connected feedforward neural networks as abstract mathematical objects with properties much like real numbers. We posit and go on to prove they can be made to exhibit properties of exponentiation such as seen in $\pwr^{q,\ve}_n$ in Definition \ref{def:pwr} and Lemma \ref{power_prop}, polynomials such as seen in $\pnm^{q,\ve}_n$ in Definition \ref{def:nn_poly} Lemma \ref{pnm_prop}, exponentials such as in $\xpn_n^{q,\ve}$ in Definition \ref{def:xpn} and Lemma \ref{xpn_properties}, cosines and sines such as in $\csn^{q,\ve}_n$, and $\sne^{q,\ve}_n$ in Definition \ref{def:csn}, Definition \ref{def:sne}, Lemmas \ref{csn_properties}, and \ref{sne_properties}, respectively, and finally neural network approximants for $\int_a^b e^x dx$ such as in Definition \ref{def:E} and in Lemma \ref{mathsfE} all of which will require the use of the $\tun_d$ network as defined in Definition \ref{def:tun} whose properties are proven in Lemma \ref{tun_prop}.

In this sense, this work is also in the spirit of a body of recent literature, seeking to ``algebrify'' neural networks, in, for instance, Section 2 of \cite{gunnar_carlsson}, and \cite{shiebler2021category}.  

In summary, our contributions are as follows:
\begin{enumerate}
	\item Introduce new neural networks such as $\tun_n, \pwr^{q,\ve}_n$, $\pnm_{n}^{q,\ve}$, $\xpn_n^{q,\ve}$, $\csn_n^{q,\ve}$, $\sne_n^{q,\ve}$, $\trp^h$, $\etr^{N,h}$, and $\mathsf{E}^{N,h,q,\ve}_n$, as tools for approximating common functions.
	\item We exhibit upper bounds for, among other things, depths, parameter counts, and accuracies.
	\item We will make use of $\tun_n$ to redefine what it means to stack unequal depth neural networks, although in practice we will not use this for later proofs.
	\item We provide parameter bounds for an interpolation scheme found in Section 4.2 of \cite{bigbook}.
	\item We introduce neural network diagrams, which are common in computer science literature but have not been applied to this description of neural networks. They derive straightforwardly from diagrams found in well-known literature such as \cite{arik2} or \cite{vaswani2}.
\end{enumerate}

Our work derives primarily from several works previously done in \cite{petersen_optimal_2018}, \cite{grohsetal},\cite{Grohs_2022},  \cite{grohs2019spacetime} and \cite{bigbook} but extends this framework much farther.

We will spend the first four pages exploring this framework as it stands, and this will serve as the springboard for our work in the next four pages. Because the proofs are lengthy, we will relegate these to the Appendix, which will be substantial.

\section{Building up towards $\mathsf{E}^{N,h,q,\ve}_n$}
The first part is dedicated to architecting and building towards entirely neural network approximations for equations of the form $\int_a^b e^x dx$.
\subsection{Basic Definition of Artificial Neural Networks}
Our definition is derived from Definition 1.3.1 in \cite{bigbook}. 
\begin{definition}
Let $\mathsf{NN}$ be the set given by:
\begin{align}
		\neu = \bigcup_{L\in \N} \bigcup_{l_0,l_1,...,l_L \in \N^{L+1}} ( \bigtimes^L_{k=1} [ \R^{l_k \times l_{k-1}} \times \R^{l_k}]  )
\end{align}
An artificial neural network is a tuple $\lp \nu, \param, \dep, \inn, \out, \hid, \lay, \wid \rp   $ where $\nu \in \neu$ and is equipped with the following functions (referred to as auxiliary functions) satisfying for all $\nu \in \lp \bigtimes^L_{k=1} \lb \R^{l_k \times l_{k-1}} \times \R^{l_k}\rb  \rp$:
	\begin{enumerate}
		\item $\param: \neu \rightarrow \N$ denoting the number of parameters of $\nu$, given by:
		\begin{align}\label{paramdef}
			\param(\nu) = \sum^L_{k=1}l_k ( l_{k-1}+1 ) 
		\end{align}
		\item $\dep: \neu \rightarrow \N$ denoting the number of layers of $\nu$ other than the input layer given by:
		\begin{align}
			\dep(\nu) = L
		\end{align}
		\item $\inn:\neu \rightarrow \N$ denoting the width of the input layer, given by:
		\begin{align}
			\inn(\nu) = l_0
		\end{align}
		\item $\out: \neu \rightarrow \N$ denoting the width of the output layer, given by:
		\begin{align}
			\out(\nu) = l_L
		\end{align}
		\item $\hid: \neu \rightarrow \N_0$ denoting the number of hidden layers (i.e. layers other than the input and output), given by:
		\begin{align}
			\hid(\nu) = L-1
		\end{align}
		\item $\lay: \neu \rightarrow \bigcup_{L \in \N} \N^L$ denoting the width of layers as an $(L+1)$-tuple, given by:
		\begin{align}
			\lay(\nu) = ( l_0,l_1,l_2,...,l_L )
		\end{align}
		We will sometimes refer to this as the layer configuration or layer architecture of $\nu$.
		\item $\wid_i: \neu \rightarrow \N_0$ denoting the width of layer $i$, given by:
		\begin{align} \label{widthdef}
			\wid_i(\nu) = \begin{cases}
				l_i & i \leqslant L \\
				0 & i > L
			\end{cases} 
		\end{align}
	\end{enumerate}
\end{definition}

\begin{remark}
	We will often use just $\nu$ to represent this neural network when we really mean the tuple $\lp \nu, \param, \dep, \inn, \out, \hid, \lay, \wid \rp   $. This is analogous to when we say that $X$ is a topological space when we actually mean the pair $\lp X,\tau\rp$ or probability space when we mean the triple $\lp \Omega, \mathcal{F},\mathbb{P}\rp$.
\end{remark}

In and of themselves neural networks are not quite helpful, they become continuous functions once they are \textit{instantiated}. We will denote by $\real_{\act}: \neu \rightarrow C (\R^{\inn(\nu)},\R^{\out(\nu)})$, a mapping called \textit{instantiation}, where $\act \in C \lp \R, \R\rp$. For all our cases we will consider $\rect$, the $\relu$. Instantiation is defined as follows:

\begin{definition}[Instantiation with an activation function]
	Let $\act \in C ( \R, \R )$, we denote by $\real_{\act}: \neu \rightarrow ( \bigcup_{k,l \in \N} C ( \R^k, \R^l ) )$ the function satisfying for all $L \in \N$, $l_0,l_1,...,l_L \in \N$, $\nu = ( ( W_1, b_1 ) , ( W_2, b_2) ,...,( W_L, b_L ) ) \in ( \bigtimes^L_{k=1} [ \R^{l_k \times l_{k-1}} \times \R^{l_k}]  )$, $x_0 \in \R^{l_0}, x_1 \in \R^{l_1},...,x_{L-1} \in \R^{l_L-1}$ and with $\forall k \in \N \cap (0,L):x_k = \act ( [ W_kx_k+b_k ]_{*,*} $\footnote{Given $f \in C(\R,\R)$, and vector $x \in \R^d$, we will denote by $f([x]_{*,*})$ as the component-wise application of $f$ to vector $x$.} such that:
	\begin{align}\label{5.1.11}
		\real_{\act}( \nu ) \in C ( \R^{l_0}, \R^{l_L} ) \: \text{ and } \\ \: ( \real_{\act}( \nu) ) ( x_0 ) = W_Lx_{L-1}+b_L
	\end{align}
\end{definition}
\begin{remark}
\label{rem:single_layer_realizations}
    Crucially note that we hit all layers of this neural network with the activation function, except the last layer, meaning that for a neural network with one layer, we simply map $((W,b)) \rightarrow W(\cdot) + b$, without any activation function.
\end{remark}

\begin{remark}
	As Definition \ref{def:comp}, and Proposition 2.6 of \cite{grohs2019spacetime} will show, instantiation is sufficiently functorial in that it respects composition. A full discussion of the abstract algebraic properties of instantiation is outside the scope of this paper and is possibly future work.
\end{remark}

\subsection{Composition}

In composition, we envisage that the last layer of the first function ``overlaps'' with the first layer of the second function being composed.

\begin{definition}[Compositions of ANNs]\label{def:comp}
	We denote by $( \cdot ) \bullet ( \cdot ): \{ ( \nu_1,\nu_2 ) \in \neu \times \neu: \inn(\nu_1) = \out (\nu_1) \} \rightarrow \neu$ the function satisfying for all $L,M \in \N, l_0,l_1,...,l_L, m_0, m_1,...,m_M \in \N$, $\nu_1 = ( ( W_1, b_1 ), ( W_2, b_2 ),...,( W_L,b_L ) ) \in ( \bigtimes^L_{k=1} [ \R^{l_k \times l_{k-1}} \times \R^{l_k}]  )$, and $\nu_2 = ( ( W'_1, b'_1 ), ( W'_2, b'_2 ),... ( W'_M, b'_M ) ) \in (  \bigtimes^M_{k=1} [ \R^{m_k \times m_{k-1}} \times \R^{m_k}]  )$ with $l_0 = \inn(\nu_1)= \out(\nu_2) = m_M$ and :
	\begin{align}
		&\nu_1 \bullet \nu_2 = \nonumber\\ &\begin{cases}
			(( W'_1,b'_1 ), ( W'_2,b'_2 ), \hdots,( W'_{M-1}, b'_{M-1}),\\ ( W_1W'_M, W_1b'_{M} + b_1), (W_2, b_2 )\\,\hdots, ( W_L,b_L )) \: :( L> 1 ) \land ( M > 1 ) \\
			((W_1W'_1,W_1b'_1+b_1),(W_2,b_2), (W_3,b_3)\\,\hdots,(W_Lb_L)) \: :(L>1) \land (M=1) \\
			((W'_1, b'_1),(W'_2,b'_2), \hdots , \\(W'_{M-1}, b'_{M-1})(W_1, b'_M + b_1)) \: :(L=1) \land (M>1) \\
			((W_1W'_1, W_1b'_1+b_1)) \: :(L=1) \land (M=1)
		\end{cases}
	\end{align}\nonumber
	
\end{definition}
Neural network composition has the following properties:
\begin{lemma}\label{lem: comp_prop}
	Let $\nu_1, \nu_2 \in \neu$ and suppose $\out( \nu_1) = \inn ( \nu_2)$. Then we have: $\dep ( \nu_1 \bullet \nu_2 ) = \dep( \nu_1) + \dep ( \nu_2) -1$, $\lay(\nu_1 \bullet \nu_2) = ( \wid_1( \nu_2), \wid_2 ( \nu_2),\hdots, \\\wid_{\hid( \nu_2)},\wid_1( \nu_1), \wid_2( \nu_1),\hdots, \wid_{\dep( \nu_1)}( \nu_1))$, $\hid ( \nu_1 \bullet \nu_2) = \hid ( \nu_1) + \hid( \nu_2)$, $\param ( \nu_1 \bullet \nu_2) \les \param( \nu_1) + \param ( \nu_2) + \wid_1 ( \nu_1)\cdot \wid_{\hid( \nu_2)}( \nu_2)$, for all $\act \in C ( \R, \R)$ that $\real_{\act}( \nu_1 \bullet \nu_2) ( x ) \in C ( \R^{\inn ( \nu_2)},\R^{\out( \nu_1)})$ and further: $\real_{\act} ( \nu_1 \bullet \nu_2) = \lb \real_{\act}( \nu_1)\rb \circ \lb \real_{\act}( \nu_2 )\rb$
\end{lemma}
\begin{proof}
    The first two assertions are straightforward from the description of composition. For a full proof, see Proposition 2.6 in \cite{grohs2019spacetime}. 
\end{proof}

\subsection{Affine Networks, $\cpy$, and $\sm$}
As noted in Remark \ref{rem:single_layer_realizations} neural networks of just one layer are a crucial class of neural networks. We will call them \textit{affine neural networks}. 
\begin{definition}
	Let $m,n \in \N$, $W \in \R^{m \times n}$, $b \in \R^m$.We denote by $\aff_{W,b} \in \lp \R^{m\times n} \times \R^m \rp \subseteq \neu$ the neural network given by $\aff_{W,b} = ((W,b))$.
\end{definition}
Of these the following two constitute two of the most important affine functions. 

\begin{definition}[The $\cpy$ Network]\label{def:cpy}
	We define the neural network, $\cpy_{n,k} \in \neu$ for $n,k\in \N$ as the neural network given by\footnote{We will denote the identity matrix of size $d$ as $\mathbb{I}_d$ and a zero vector of the same size as $\mymathbb{0}_d$.}:
	\begin{align}
		\cpy_{n,k} = \aff_{\underbrace{\lb \mathbb{I}_{k} \: \mathbb{I}_k \: \cdots \: \mathbb{I}_k \rb^T}_{n-\text{many}},\mymathbb{0}_{nk}}
	\end{align} 
	Where $k$ represents the dimensions of the vectors being copied and $n$ is the number of copies of the vector being made.
\end{definition}

\begin{definition}[The $\sm$ Network]\label{def:sm}
	We define the neural network $\sm_{n,k}$ for $n,k \in \N$ as the neural network given by:
	\begin{align}
		\sm_{n,k} = \aff_{\underbrace{\lb \mathbb{I}_k \: \mathbb{I}_k \: \cdots \: \mathbb{I}_k\rb}_{n-\text{many}}, \mymathbb{0}_{k}}
	\end{align}
	Where $k$ represents the dimensions of the vectors being added and $n$ is the number of vectors being added.
\end{definition}

Of great imporatnce to us are neural networks, by dint of their structure, end up as scalar multiplication upon instantiation. The following two neural networks do just that.

\begin{definition}[Scalar left-multiplication with an ANN]\label{slm}
	Let $\lambda \in \R$. We will denote by $(\cdot ) \triangleright (\cdot ): \R \times \neu \rightarrow \neu$ the function that satisfy for all $\lambda \in \R$ and $\nu \in \neu$ that $\lambda \triangleright \nu = \aff_{\lambda \mathbb{I}_{\out(\nu)},0} \bullet \nu$. 
\end{definition}

\begin{definition}[Scalar right-multiplication with an ANN]
	Let $\lambda \in \R$. We will denote by $(\cdot) \triangleleft (\cdot): \neu \times \R \rightarrow \neu$ the function satisfying for all $\nu \in \neu$ and $\lambda \in \R$ that $\nu \triangleleft \lambda = \nu \bullet \aff_{\lambda \mathbb{I}_{\inn(\nu)},0}$. 
\end{definition}
They instantiate in quite predictable ways:
\begin{theorem}
	Let $\lambda \in \R$. Let $\nu \in \neu$. For all $\act \in C(\R,\R)$, and $x \in \R^{\inn(\nu)}$, it is the case that:
	\begin{align}
		\real_{\act}\lp \lambda \triangleright \nu\rp = \lambda \cdot \real_{\act}(\nu)(x)
	\end{align}
	and:
	\begin{align}
		\real_{\act}(\nu \triangleleft \lambda)(x) = \real_{\act}(\nu)(\lambda \cdot x)
	\end{align}
\end{theorem}

\begin{proof}
	Lemma \ref{lem: comp_prop} tells us that:
	\begin{align}
	\real_{\act} (\lambda \triangleright \nu) &= \real_{\act}(\aff_{\lambda\mathbb{I}_{\out(\nu)},\mymathbb{0}_{\out(\nu)}}\bullet \nu)(x) \nonumber \\
	&= \real_{\act} (\aff_{\lambda\mathbb{I}_{\out(\nu)},\mymathbb{0}_{\out(\nu)}})\bullet \real_{\act} (\nu)(x) \nonumber \\
	&= \lambda\mathbb{I}_{\out(\nu)} \cdot \real_{\act}(\nu)(x) = \lambda \real_{\act}(\nu)(x)
	\end{align}
	and that:
	\begin{align}
	\real_{\act} (\nu \triangleleft \lambda) &= \real_{\act}(\nu \bullet \aff_{\lambda\mathbb{I}_{\out(\nu)},\mymathbb{0}_{\out(\nu)}})(x) \nonumber \\
	&=  \real_{\act} (\nu)  \bullet \real_{\act} (\aff_{\lambda\mathbb{I}_{\out(\nu)},\mymathbb{0}_{\out(\nu)}})(x) \nonumber \\
	&= \real_{\act}(\nu)(\lambda \mathbb{I}_{\out(\nu)}\cdot x) = \real_{\act}(\lambda \nu)
	\end{align}
\end{proof}

\subsection{Stacking and Neural Network Sums}

Sometimes we will need to "stack" neural networks. Stacking is done as follows:

\begin{definition}[Stacking of ANNs of same depth]\label{5.2.5}\label{def:stk}
	Let $n\in \N$, let $\{\lp \nu_1, \nu_2,...,\nu_n \rp \in \neu^n$ such that $\dep(\nu_1) = \dep(\nu_2) = ... = \dep(\nu_n)$ we then denote by: 
	\begin{align}
		\boxminus^n_{i=1}:  \neu^n \rightarrow \neu 
	\end{align}
	the function satisfying for all $L \in \N$, $\nu_1,\nu_2,...,\nu_n \in \neu$ and $L = \dep(\nu_1) = \dep(\nu_2) = ... = \dep(\nu_n)$ that:
	\begin{align*}\label{5.4.2}
		\boxminus^n_{i=1} \nu_i= \lp \lp
		\begin{bmatrix}
			\we_{1,\nu_1} & 0 & 0 & 0\\
			0 & \we_{1,\nu_2} & & \vdots \\
			\vdots & & \ddots &\\
			0 & \hdots & & \we_{1,\nu_n} 
		\end{bmatrix},
		\begin{bmatrix}
			\bi_{1,\nu_1} \\ \bi_{1,\nu_2} \\ \vdots \\ \bi_{1,\nu_n}
		\end{bmatrix}
		  \rp, \right. \nonumber\\
		\left. 
		\lp \begin{bmatrix}
			\we_{2,\nu_1} & 0 & 0 & 0\\
			0 & \we_{2,\nu_2} & & \vdots \\
			\vdots & & \ddots &\\
			0 & \hdots & & \we_{2,\nu_n} 
		\end{bmatrix},
		\begin{bmatrix}
			\bi_{2,\nu_1} \\ \bi_{2,\nu_2} \\ \vdots \\ \bi_{2,\nu_n}
		\end{bmatrix} \rp,..., \right. \nonumber\\
		\left. 
		\lp \begin{bmatrix}
			\we_{L,\nu_1} & 0 & 0 & 0\\
			0 & \we_{L,\nu_2} & & \vdots \\
			\vdots & & \ddots &\\
			0 & \hdots & & \we_{L,\nu_n} 
		\end{bmatrix},
		\begin{bmatrix}
			\bi_{L,\nu_1} \\ \bi_{L,\nu_2} \\ \vdots \\ \bi_{L,\nu_n}
		\end{bmatrix} \rp \rp 
	\end{align*}
	For the case where two neural networks $\nu_1,\nu_2$ are stacking it is convenient to write $\nu_1 \boxminus \nu_2$.
\end{definition}

For unequal depth neural networks it is convenient to introduce ''padding" via what we will call tunneling neural networks. 

\begin{definition}[Identity Neural Network]\label{7.2.1}
	We will denote by $\id_d \in \neu$ the neural network satisfying for all $d \in \N$ that:
	\item \begin{align}
		\id_1 &= \lp \lp \begin{bmatrix}
			1 \\
			-1
		\end{bmatrix}, \begin{bmatrix}
			0 \\
			0
		\end{bmatrix}\rp \lp \begin{bmatrix}
			1 \quad -1
		\end{bmatrix},\begin{bmatrix} 0\end{bmatrix}\rp \rp \nonumber\\ &\in  \lp \lp \R^{2 \times 1} \times \R^2 \rp \times \lp \R^{1\times 2} \times \R^1 \rp \rp 
	\end{align}
	and
	\item \begin{align}\label{7.2.2}
		\id_d = \boxminus^d_{i=1} \id_1
	\end{align}
	for $d>1$.
\end{definition}
We refer the reader to Lemma 2.2.2, Proposition 2.2.3, Proposition 2.2.4, and Corollary 2.2.5 in \cite{bigbook} with $\boxminus_1^{n} \curvearrowleft \mathbf{P}_n$, $\wid \curvearrowleft \mathbb{D}$, $\lay \curvearrowleft \mathcal{D}$, $\param \curvearrowleft \mathcal{P}$ , $\hid \curvearrowleft \mathcal{H}$, $\out\curvearrowleft \mathcal{O}$, $\inn \curvearrowleft \mathcal{I}$, $\id \curvearrowleft \mathfrak{I}$, stacking $\curvearrowleft$ paralleliztion, and finally instantiation $\curvearrowleft$ realization.
\begin{remark}
	Moving forward, the above will be our substitution scheme whenever we refer to \cite{bigbook} or \cite{grohs2019spacetime}.
\end{remark}
A tunneling neural network is essentially multiple $\id_1$ networks composed together.
\begin{definition}[The Tunneling Neural Networks]\label{def:tun}
	We define the tunneling neural network, denoted as $\tun_n$ for $n\in \N$ by:
	\begin{align}
		\tun_n \coloneqq \begin{cases}
			\aff_{1,0} &:n= 1 \\
			\id_1 &: n=2 \\
			\bullet^{n-2} \id_1 & n \in \N \cap [3,\infty)
		\end{cases}
	\end{align}	
	Where $\id_1$ is as in Definition \ref{7.2.1}.
\end{definition}

For properties, see Lemma \ref{tun_prop}.

Thus we may stack neural networks of unequal depth by introducing tunneling networks at the end of the shorter neural networks, thereby introducing a form of ``padding''.

\begin{definition}
	Let $n \in \N$, and $\nu_1,\nu_2,...,\nu_n \in \neu$. We will define the stacking of unequal length neural networks, denoted $\DDiamond^n_{i=1}\nu_i$ as the neural network given by:
	\begin{align}
		\DDiamond^n_{i=1}\nu_i \coloneqq 
		\boxminus^n_{i=1} \lb \tun_{\max_i \left\{\dep \lp \nu_i \rp\right\} +1 - \dep \lp \nu_i\rp} \bullet \nu_i \rb 
	\end{align}
\end{definition}

Once we are able to stack neural networks we are now finally ready to introduce neural network sums. Essentially we make two copies of our input, run the copies through the two summand networks and take their sum on the other side.

\begin{definition}[Sum of ANNs of the same depth and same end widths\footnote{The beginning layer width $l_0$ and end layer width $l_L$ will collectively be called ``end-widths''. Where the beginning and ending width are the same, we may also seek to call them ``square'' neural networks by analogy with }]
	Let $u,v \in \Z$ with $u \leqslant v$. Let $\nu_u,\nu_{u+1},...,\nu_v \in \neu$ satisfy for all $i \in \N \cap [u,v]$ that $\dep(\nu_i) = \dep(\nu_u)$, $\inn(\nu_i) = \inn(\nu_u)$, and $\out(\nu_i) = \out(\nu_u)$. We then denote by $\oplus^n_{i=u} \nu_i$ or alternatively $\nu_u \oplus\nu_{u+1} \oplus \hdots \oplus\nu_v$ the neural network given by:
	\begin{align}\label{5.4.3}
		&\oplus^v_{i=u}\nu_i \nonumber\\&\coloneqq \lp \sm_{v-u+1,\out(\nu_2)} \bullet \lb \boxminus^v_{i=u}\nu_i \rb \bullet \cpy_{(v-u+1),\inn(\nu_1)} \rp \nonumber
	\end{align}
\end{definition}
Similarly, for unequal depth neural networks, we have the following.

\begin{definition}[Sum of ANNs of unequal depths but same end widths]
	Let $n\in \N$. Let $\nu_1,\nu_2,...,\nu_n \in \neu$ such that they have the same end widths. We define the neural network $\DDiamond_{i=1}^n\nu_i \in \neu$, the neural network sum of neural networks of unequal depth as:
	\begin{align}
		\DDiamond^n_{i=1}\nu_i  \coloneqq \lp \sm_{n,\out(\nu_2)} \bullet \lb \DDiamond^v_{i=u}\nu_i \rb \bullet \cpy_{n,\inn(\nu_1)} \rp
	\end{align}
\end{definition}

\subsection{Neural Networks for Squaring and Products}

We will define neural networks for squaring on $[0,1]$, squaring on $\mathbb{R}$, and product operations for $x, y \in \mathbb{R}$. Detailed proofs of their accuracy, parameters, and depth will be provided in the Appendix in Lemmas \ref{lem:phi_k}, and Corollary \ref{cor:phi_network},  and can also be found in the literature, particularly in Section 3.2.1 in \cite{grohs2019spacetime}.
\begin{definition}[The $\mathfrak{i}_d$ Network]\label{def:mathfrak_i}
	For all $d \in \N$ we will define the following set of neural networks as ``activation neural networks'' denoted $\mathfrak{i}_d$ as:
	\begin{align}
		\mathfrak{i}_d = \lp \lp \mathbb{I}_d, \mymathbb{0}_d\rp, \lp \mathbb{I}_d, \mymathbb{0}_d\rp \rp 
	\end{align}
\end{definition}
\begin{definition}[The $\Phi_k$ network]\label{phi_k network}
	Let $\lp c_k \rp _{k \in \N} \subseteq \R$, $\lp A_k \rp _{k \in \N} \in \R^{4 \times 4},$ $B\in \R^{4 \times 1}$, $\lp C_k \rp _{k\in \N}$ satisfy for all $k \in \N$ that:
	\begin{align}\label{(6.0.1)}
		A_k = \begin{bmatrix}
			2 & -4 &2 & 0 \\
			2 & -4 & 2 & 0\\
			2 & -4 & 2 & 0\\
			-c_k & 2c_k & -c_k & 1
		\end{bmatrix} \quad B=\begin{bmatrix}
			0 \\ -\frac{1}{2} \\ -1 \\ 0
		\end{bmatrix} \nonumber\\ C_k = \begin{bmatrix}
			-c_k & 2c_k &-c_k & 1
		\end{bmatrix} \quad c_k = 2^{1-2k}
	\end{align}
	Let $\xi_k \in \neu$, $k\in \N$ satisfy for all $k \in [2,\infty) \cap \N$ that $\xi_1 = \lp \aff_{C_1,0} \bullet \mathfrak{i}_4 \rp \bullet \aff_{\mymathbb{e}_4,B}$. Note that for all $d \in \N$, $\mathfrak{i}_d = \lp \lp \mathbb{I}_d, \mymathbb{0}_d \rp, \lp \mathbb{I}_d, \mymathbb{0}_d \rp \rp$ (explained in detail in Definition \ref{actnn}), and that:
	\begin{align}
		\Phi_k =\lp \aff_{C_k,0}\bullet \mathfrak{i}_4 \rp \bullet \lp \aff_{A_{k-1},B} \bullet \mathfrak{i}_4\rp \bullet \cdots\nonumber\\ \bullet \lp  \aff_{A_1,B} \bullet \mathfrak{i}_4 \rp \bullet \aff_{\mymathbb{e}_4,B} 
	\end{align}
\end{definition}
We will want to be able to reverse-engineer a suitable $k$ given a certain epsilon. Hence we introduce $M \in \N$ and a neural network $\Phi$ as such.
\begin{definition}[$\Phi$ Network ]\label{def:Phi}
Let $\ve \in \lp 0,\infty\rp$, $M= \min \{ \frac{1}{2}\log_2 \lp \ve^{-1} \rp -1,\infty\}\cap \N$, $\lp c_k\rp_{k \in \N} \subseteq \R$, $\lp A_k\rp_{k\in\N} \subseteq \R^{4 \times 4}$, $B \in \R^{4\times 1}$, $\lp C_k\rp_{k\in \N}$ satisfy for all $k \in \N$ that:
	\begin{align}
		A_k = \begin{bmatrix}
			2&-4&2&0 \\
			2&-4&2&0\\
			2&-4&2&0\\
			-c_k&2c_k & -c_k&1
		\end{bmatrix}, \quad B = \begin{bmatrix}
			0\\ -\frac{1}{2}\quad \\ -1 \\ 0
		\end{bmatrix}\nonumber\\ C_k = \begin{bmatrix}
			-c_k &2c)_k&-c_k&1
		\end{bmatrix} \quad c_k = 2^{1-2k}
	\end{align}
	and let $\Phi \in \neu$ be defined as:
	\begin{align}
		\Phi = \begin{cases}
			\lb \aff_{C_1,0}\bullet \mathfrak{i}_4\rb \bullet \aff_{\mymathbb{e}_4,B} & M=1 \\
			\lb \aff_{C_M,0} \bullet \mathfrak{i}_4\rb\bullet \lb \aff_{A_{M-1},0} \bullet \mathfrak{i}_4 \rb \bullet \cdots \nonumber\\ \bullet \lb \aff_{A_1,B}\bullet \mathfrak{i}_4\rb \bullet \aff_{\mymathbb{e}_4,B} & M \in \lb 2,\infty \rp \cap \N
		\end{cases}
	\end{align}
\end{definition}
Once we are able to square on $[0,1]$, it is a simple matter to extend it to the entire $\R$ via pre and post multiplying with $\lp \frac{\ve}{2}\rp^{\frac{1}{q-2}}$ for $\ve \in (0,\infty)$ and $q\in(2,\infty)$, more precisely we may define the neural network $\sqr^{q,\ve}$ as:
\begin{definition}\label{def:sqr}
	Let $\delta,\epsilon \in (0,\infty)$, $\alpha \in (0,\infty)$, $q\in (2,\infty)$, $\Phi \in \neu$ satisfy that $\delta = 2^{\frac{-2}{q-2}}\ve ^{\frac{q}{q-2}}$, $\alpha = \lp \frac{\ve}{2}\rp^{\frac{1}{q-2}}$. Let $\Phi$ be as in Definition \ref{def:Phi}, we will then define the neural network $\sqr \in \neu$ as the neural network define as follows:
	\begin{align}
		 &\sqr \nonumber\\ &\coloneqq \lp \aff_{\alpha^{-2},0} \bullet \Phi \bullet \aff_{\alpha,0} \rp \bigoplus\lp \aff_{\alpha^{-2},0} \bullet \Phi \bullet \aff_{-\alpha,0}\rp \nonumber
	\end{align}
\end{definition}

Now that we are able to square over all of the real line it is clear to see that for all $x,y \in \R$ it is the case that $xy = \frac{1}{2}(x+y)^2 - \frac{1}{2}x^2 - \frac{1}{2}y^2$. Whence we get the neural network $\prd^{q,\ve}$ defined as such:
\begin{definition}\label{def:prd}
	Let $\delta,\ve \in \lp 0,\infty \rp $, $q\in \lp 2,\infty \rp$, $A_1,A_2,A_3 \in \R^{1\times 2}$, let $\delta = \ve \lp 2^{q-1} +1\rp^{-1}$, $A_1 = \lb 1 \quad 1 \rb$, $A_2 = \lb 1 \quad 0 \rb$, $A_3 = \lb 0 \quad 1 \rb$, let $\sqr^{q,\ve}$ be as defined in Definition \ref{def:sqr}. We will then define the neural network $\prd^{q,\ve}$ as such:
		\begin{align}
			&\prd \nonumber\\ &\coloneqq \lp \frac{1}{2}\triangleright \lp \Phi \bullet \aff_{A_1,0} \rp \rp \bigoplus \lp \lp -\frac{1}{2}\rp \triangleright\lp \Phi \bullet \aff_{A_2,0} \rp \rp\nonumber\\ &\bigoplus\lp \lp -\frac{1}{2}\rp \triangleright \lp \Phi \bullet \aff_{A_3,0} \rp \rp \nonumber
		\end{align}
\end{definition} 
See Lemmas 2.1, 3.1, and 4.1 in \cite{grohs2019spacetime}
This is a straightforward neural network representation of the identity: $xy = \frac{1}{2}(x+y)^2 - \frac{1}{2}x^2 - \frac{1}{2}y^2$. As we go through this paper this theme will be repeated again and again, and indeed is the core theme of this paper. 
\subsection{The $\pwr_n^{q,\ve}$ networks}
Once we know how to multiply two numbers together, the next logical step is to raise a real number to a power. This is done via a recursive application of $\prd^{q,\ve}$. We will define the family of $\pwr^{q,\ve}_n$ networks as follows:
\begin{definition}[The Power Neural Network]\label{def:pwr}
	Let $n\in \N$. Let $\delta,\ve \in \lp 0,\infty \rp $, $q\in \lp 2,\infty \rp$, satisfy that $\delta = \ve \lp 2^{q-1} +1\rp^{-1}$. We define the power neural networks $\pwr_n^{q,\ve} \in \neu$, denoted for $n\in \N_0$ as:
	\begin{align}
		\pwr_n^{q,\ve} = \begin{cases}
			\aff_{0,1} & :n=0\\
			\prd^{q,\ve} \bullet \\ \lb \tun_{\dep(\pwr_{n-1}^{q,\ve})} \boxminus \pwr_{n-1}^{q,\ve} \rb \bullet \cpy_{2,1} & :n \in \N
		\end{cases} \nonumber 
	\end{align}
\end{definition}

For a full proof of properties, including depth counts, parameter counts, and accuracy, see Lemma \ref{power_prop} in the Appendix.

\subsection{Neural Network Polynomials}
Indeed once we have a definition of raising to a power for neural networks, the next logical extension is to introduce the concept of neural network polynomials, i.e. neural networks of the form:
\begin{definition}[Neural Network Polynomials]\label{def:nn_poly}
		Let $\delta,\ve \in \lp 0,\infty \rp $, $q\in \lp 2,\infty \rp$ and $\delta = \ve \lp 2^{q-1} +1\rp^{-1}$. For fixed $q,\ve$, fixed $n \in \N_0$, and for $C = \{c_0,c_1,\hdots, c_n \} \in \R^{n+1}$ (the set of coefficients), we will define the following objects as neural network polynomials:
	\begin{align*}
		&\pnm^{q,\ve}_{n,C}\nonumber \\ &\coloneqq  \bigoplus^n_{i=0} \lp c_i \triangleright\lb \tun_{\max_i \left\{\dep \lp \pwr_i^{q,\ve} \rp\right\} +1 - \dep \lp \pwr^{q,\ve}_i\rp} \bullet \pwr_i^{q,\ve}\rb \rp  
	\end{align*} 
\end{definition}

Note the striking resemblance to polynomials. Indeed these are algebraic objects that are equivalent to the standard polynomials in $\R \lb x\rb$. A full discussion of the ring-like properties of neural networks, as defined, is beyond the scope of this paper.
\subsection{$\xpn_n^{q,\ve}$, $\csn_n^{q,\ve}$, $\sne_n^{q,\ve}$, and neural network exponentiation, cosines and sines}

\begin{definition}[The $\xpn_n^{q,\ve}$ Networks]\label{def:xpn}
	Let $\delta,\ve \in \lp 0,\infty \rp $, $q\in \lp 2,\infty \rp$ and $\delta = \ve \lp 2^{q-1} +1\rp^{-1}$. We define, for all $n\in \N_0$, the family of neural networks $\xpn_n^{q,\ve} as$: 
	\begin{align*}
		&\xpn_n^{q,\ve} \\
		&\coloneqq \bigoplus^n_{i=0} \lb \frac{1}{i!} \triangleright\lb \tun_{\max_i \left\{\dep \lp \pwr_i^{q,\ve} \rp\right\} +1 - \dep \lp \pwr^{q,\ve}_i\rp} \bullet \pwr_i^{q,\ve}\rb \rb
	\end{align*}
\end{definition}
It is straightforward to see that this is the equivalent of the Taylor approximation of $e^x$ centered around $0$.
For a full proof of properties, including depth counts, parameter counts, and accuracy, see Lemma \ref{xpn_properties} in the Appendix.

\begin{definition}[The $\csn_n^{q,\ve}$ Networks]\label{def:csn}.
	Let $\delta,\ve \in \lp 0,\infty \rp $, $q\in \lp 2,\infty \rp$ and $\delta = \ve \lp 2^{q-1} +1\rp^{-1}$. Let $\pwr^{q,\ve}$ be a neural network defined in Definition \ref{def:pwr}. We will define the neural network $\mathsf{Csn}_{n,q,\ve}$ as:
	\begin{align*}
		&\mathsf{Csn}_n^{q,\ve} \coloneqq \\
		&\bigoplus^n_{i=0} \lb \frac{(-1)^i}{2i!}\triangleright\lb \tun_{\max_i \left\{\dep \lp \pwr_i^{q,\ve} \rp\right\} +1 - \dep \lp \pwr^{q,\ve}_i\rp} \bullet \pwr_{2i}^{q,\ve}\rb \rb
	\end{align*}
\end{definition}
It is straightforward to see that this is the equivalent of the Taylor approximation of $\cos\lp x \rp$ centered around $0$. For a full proof of properties

\begin{definition}[The $\mathsf{Sne}_n^{q,\ve}$ Newtorks]\label{def:sne}
	Let $\delta,\ve \in \lp 0,\infty \rp $, $q\in \lp 2,\infty \rp$ and $\delta = \ve \lp 2^{q-1} +1\rp^{-1}$. Let $\pwr^{q,\ve}$ be a neural network defined in Definition \ref{def:pwr}. We will define the neural network $\mathsf{Csn}_{n,q,\ve}$ as:
	\begin{align}
		\mathsf{Sne}_n^{q,\ve} \coloneqq \csn^{q,\ve} \bullet \aff_{1, -\frac{\pi}{2}}
	\end{align}
\end{definition}
It is straightforward to see that this is the equivalent of the Taylor approximation of $\cos\lp x \rp$ centered around $0$. In partical note the parameter savings in defining $\sin \lp x\rp = \cos \lp x - \frac{\pi}{2}\rp$

For a full proof of properties, including depth counts, parameter counts, and accuracy, see Lemma \ref{xpn_properties}, Lemma \ref{csn_properties}, and Lemma \ref{sne_properties} in the Appendix.
\subsection{Trapezoidal Rule}
Our next course of action is to develop a one-dimensional trapezoidal rule. This is straightforward. Essentially we need a neural network that instantiates with a continuous activation function as $\R^{N+1} \rightarrow \R$, that is to say, converts $N+1$ mesh points to an area approximant. 
\begin{definition}[The $\trp^h$ neural network]
	Let $h \in \lp 0,\infty\rp$. We define the $\trp^h \in \neu$ neural network as:
	\begin{align}
		\trp^h \coloneqq \aff_{\lb \frac{h}{2} \: \frac{h}{2}\rb,0}
	\end{align} 
\end{definition}
This has the following properties.
\begin{lemma}\label{trp_prop}
	Let $h\in \lp 0, \infty\rp$. It is then the case that:
	\begin{enumerate}
		\item for $x = \{x_1,x_2\} \in \R^2$ that $\lp \real_{\rect} \lp \trp^h \rp \rp \lp x \rp \in C \lp \R^2, \R \rp$
		\item for $x = \{x_1,x_2 \} \in \R^2$ that $\lp \real_{\rect} \lp \trp^h \rp \rp \lp x \rp = \frac{1}{2}h \lp x_1+x_2 \rp$
		\item $\dep \lp \trp^h \rp = 1$
		\item $\param\lp \trp^h \rp = 3$
		\item $\lay \lp \trp^h \rp = \lp 2,1 \rp$
	\end{enumerate}
\end{lemma}
\begin{proof}
	This a straightforward consequence of Lemma 2.3.2 in \cite{bigbook}.
\end{proof}
And its larger sibling:
\begin{lemma}[The $\etr^{N,h}$ Networks]\label{etr_prop}
	Let $n\in \N$. Let $x_0 \in \lp -\infty, \infty \rp$, and $x_n \in \lb x_0, \infty \rp$. Let $ x = \lb x_0 \: x_1 \:...\: x_n\rb  \in \R^{n+1}$ and $h\in \lp -\infty, \infty\rp$ such that for all $i \in \{0,1,...,n\}$ it is the case that $x_i = x_0+i\cdot h$. It is then the case that:
	\begin{enumerate}
		\item $\lp \real_{\rect} \lp \etr^{n,h} \rp \rp \lp x \rp \in C \lp \R^n, \R \rp$
		\item $\lp \real_{\rect} \lp \etr^{n,h} \rp \rp \lp x \rp = \frac{h}{2} \cdot x_0+h\cdot x_1 + \cdots + h\cdot x_{n-1} + \frac{h}{2}\cdot x_n$
		\item $\dep \lp \etr^{n,h} \rp = 1$
		\item $\param\lp \etr^{n,h} \rp = n+2$
		\item $\lay \lp \etr^{n,h} \rp = \lp n,1 \rp$
	\end{enumerate}
\end{lemma}
\begin{proof}
	This a straightforward consequence of Lemma 2.3.2 in. \cite{bigbook}.
\end{proof}

\subsection{The $\mathsf{E}^{N,h,q,\ve}_n$ network}
Now that we have a sufficient framework for expressing $e^x$ and trapezoidal integration, we may work our way towards a neural network expression of $\int_a^b e^x dx$ where $a$ and $b$ are lower and upper bound, respectively. This appears as an important class of equations in solving partial differential equations via Feynman-Kac.
	\begin{definition}\label{def:E}
	Let $n, N\in \N$ and $h \in \lp 0,\infty\rp$. Let $\delta,\ve \in \lp 0,\infty \rp $, $q\in \lp 2,\infty \rp$, satisfy that $\delta = \ve \lp 2^{q-1} +1\rp^{-1}$. Let $a\in \lp -\infty,\infty \rp$, $b \in \lb a, \infty \rp$. Let $f:[a,b] \rightarrow \R$ be continuous and have second derivatives almost everywhere in $\lb a,b \rb$. Let $a=x_0 \les x_1\les \cdots \les x_{N-1} \les x_N=b$ such that for all $i \in \{0,1,...,N\}$ it is the case that $h = \frac{b-a}{N}$, and $x_i = x_0+i\cdot h$ . Let $x = \lb x_0 \: x_1\: \cdots x_N \rb$ and as such let $f\lp\lb x \rb_{*,*} \rp = \lb f(x_0) \: f(x_1)\: \cdots \: f(x_N) \rb$. Let $\mathsf{E}^{N,h,q,\ve}_{n} \in \neu$ be the neural network given by:
	\begin{align}
		\mathsf{E}^{N,h,q,\ve}_n = \xpn_n^{q,\ve} \bullet \etr^{N,h}
	\end{align}
	\end{definition}
	For a full proof of properties, including depth counts, parameter counts, and accuracy, see Lemma \ref{mathsfE}

\subsection{The $\nrm$, $\mxm$, and $1$-D interpolation}
To build up towards a sufficiently sophisticated version of $1$-D interpolation, we must first introduce networks that find the $1$-norm and maximum of a set. We will define the following networks $\nrm$ and $\mxm$
\begin{definition}[The $\nrm_1^d$ neural network]
	We denote by $\lp \nrm_1^d \rp _{d\in \N} \subseteq \neu$ the family of neural networks that satisfy:
	\begin{enumerate}
		\item for $d=1$:\begin{align}\label{(9.7.1)}
			\nrm^1_1 &= \lp \lp \begin{bmatrix}
				1 \\ -1
			\end{bmatrix}, \begin{bmatrix}
				0 \\ 0
			\end{bmatrix}\rp, \lp \begin{bmatrix}
				1 && 1
			\end{bmatrix}, \begin{bmatrix}
				0
			\end{bmatrix}\rp \rp \nonumber\\ &\in \lp \R^{2 \times 1} \times \R^2 \rp \times \lp \R^{1 \times 2} \times \R^1 \rp
		\end{align}
		\item for $d \in \{2,3,...\}$: \begin{align}
			\nrm_1^d = \sm_{d,1} \bullet \lb \boxminus_{i=1}^d \nrm_1^1 \rb 
		\end{align} 
	\end{enumerate}
\end{definition}
With the following properties:
\begin{lemma}\label{9.7.2}
	Let $d \in \N$. It is then the case that:
	\begin{enumerate}
		\item $\lay \lp \nrm^d_1 \rp = \lp d,2d,1 \rp$
		\item $\lp \real_{\rect} \lp \nrm^d_1\rp \rp \lp x \rp \in C \lp \R^d,\R \rp$
		\item that for all $x \in \R^d$ that $\lp \real_{\rect}\lp \nrm^d_1 \rp \rp \lp x \rp = \left\| x \right\|_1$
		\item it holds $\hid\lp \nrm^d_1\rp=1$
		\item it holds that $\param \lp \nrm_1^d \rp = 4d^2+6d+1$ 
	\end{enumerate}
\end{lemma}
\begin{proof}
	For a full proof of properties, including depth counts, parameter counts, and accuracy, see Lemma \ref{lem:nrm_prop} in the Appendix.
\end{proof}
\begin{definition}[Maxima ANN representations]
	Let $\lp \mxm ^d\rp_{d \in \N} \subseteq \neu$ represent the neural networks that satisfy:
	\begin{enumerate}
		\item for all $d \in \N$ that $\inn \lp \mxm^d \rp = d$
		\item for all $d \in \N$ that $\out\lp \mxm^d \rp = 1$
		\item that $\mxm^1 = \aff_{1,0} \in \R^{1 \times 1} \times \R^1$  
		\item that:
		\begin{align}\label{9.7.6}
			\mxm^2 = \lp \lp \begin{bmatrix}
				1 & -1 \\ 0 & 1 \\ 0 & -1
			\end{bmatrix}, \begin{bmatrix}
				0 \\ 0 \\0
			\end{bmatrix}\rp, \lp \begin{bmatrix}
				1&1&-1
			\end{bmatrix}, \begin{bmatrix}
				0
			\end{bmatrix}\rp\rp
		\end{align}
		\item it holds for all $d \in \{2,3,...\}$ that $\mxm^{2d} = \mxm^d \bullet \lb \boxminus_{i=1}^d \mxm^2\rb$, and
		\item it holds for all $d \in \{ 2,3,...\}$ that $\mxm^{2d-1} = \mxm^d \bullet \lb \lp \boxminus^d_{i=1} \mxm^2 \rp \boxminus \id_1\rb$.
	\end{enumerate}
\end{definition}
While it is straightforward to see that (\ref{9.7.6}) will give us the maximum of two numbers under instantiation with $\rect$, we may extend this to find the maximum of an arbitrary vector of numbers. If it is even, we can eliminate half the numbers. For odd, we may introduce a dummy network $\id_1$ whose sole job it will be to push our unpaired number to the next layer until it, too, is ``pruned''.

\begin{lemma}
	Let $d \in \N$, it is then the case that:
	\begin{enumerate}
		\item $\hid \lp \mxm^d \rp = \lceil \log_2 \lp x \rp \rceil $
		\item for all $i \in \N$ that $\wid_i \lp \mxm^d \rp \les 3 \left\lceil \frac{d}{2^i} \right\rceil$
		\item $\real_{\rect} \lp \mxm^d\rp \in C \lp \R^d, \R \rp$ and
		\item for all $x = \lp x_1,x_2,...,x_d \rp \in \R^d$ we have that $\lp \real_{\rect} \lp \mxm^d \rp \rp \lp x \rp = \max \{x_1,x_2,...,x_d \}$.
		\item $\param \lp \mxm^d \rp \les \left\lceil \lp \frac{2}{3}d^2+3d\rp \lp 1+\frac{1}{2}^{2\lp \left\lceil \log_2\lp d\rp\right\rceil+1 \rp}\rp + 1 \right\rceil$
		\item $\dep \lp \mxm^d\rp = \left\lceil \log_2 \lp d\rp \right\rceil + 1$
	\end{enumerate}
\end{lemma}
\begin{proof}
	See Lemma \ref{lem:mxm_prop} in Appendix.
\end{proof}

Let $N \in \N$. Let $f: [a,b] \rightarrow \R$ be a continuous bounded function with Lipschitz constant $L$. Let $x_i$ for $i \in \{1,2,\hdots, N\}$ be a set of sample points within $[a,b]$, with it being possibly the case that that for all $i \in \{0,1,\hdots, N\}$, that $x_i \sim \unif([a,b])$ and i.i.d. For all $i \in \{0,1,\hdots, N\}$, define a function $f_i: [a,b] \rightarrow \R$, as such:
\begin{align}
	f_i = f(x_i) - L \cdot \left| x-x_i\right|
\end{align} 
We will call the approximant $\max_{i \in \{0,1,\hdots, N\}}\{ f_i\lp x \rp\}$, the \textit{maximum convolution approximation}.
\begin{definition}
Let $d,N\in \N$, $L\in \lb 0,\infty \rp$, $x_1,x_2,\hdots, x_N \in \R^d$, $y = \lp y_1,y_2,\hdots,y_N  \rp \in \R^N$ and $\mathsf{MC} \subseteq \neu$ satisfy that:
	\begin{align*}\label{mc_def}
		&\mathsf{MC}^{N,d}_{x,y}\\ 
		&\coloneqq \mxm^N \bullet \aff_{-L\mathbb{I}_N,y} \bullet \lp \boxminus_{i=1}^N \lb \nrm^d_1  \bullet \aff_{\mathbb{I}_d,-x_i} \rb \rp\\ &\bullet \cpy_{N,d}
	\end{align*}
	
	The above is the neural network for maximum convolutions. For a full treatment of its properties please see Lemma \ref{lem:mc_prop}, Lemma \ref{capF_lemma}, and Lemma \ref{max_conv_converges}
	
\end{definition}
	\section{Future Work}
	There are two avenues for future research:
	\begin{enumerate}
		\item A fruitful avenue of research is exploring the higher algebraic properties of neural networks under this framework. We have already mentioned in passing on the functoriality of $\inst$, but note that the mapping described in Definition 1.3.5 of \cite{bigbook} has the trappings of a forgetful functor, although with no obvious left adjoint. Inventing the notion of a ``free'' neural network seems daunting yet intriguing.
		\item The parameter bounds, as given, may be considered rather crude. Refining these bounds and constructing neural network analogues for more functions could be a further avenue of research. 
	\end{enumerate}

\newpage
\appendix
\onecolumn
\section{Proofs}
\subsection{The Tunneling Networks}
\begin{lemma}\label{6.2.2}\label{tun_prop}
	Let $n\in \N$, $x \in \R$ and $\tun_n \in \neu$. For all $n\in \N$ and $x\in \R$, it is then the case that:
	\begin{enumerate}
		\item $\real_{\rect} \lp \tun_n \rp \in C \lp \R, \R \rp$
		\item $\dep \lp \tun_n \rp =n$
		\item $\lp \real_{\rect} \lp \tun_n \rp \rp \lp x \rp  = x$
 		\item $\param \lp \tun_n \rp = \begin{cases}
 			2 &:n=1 \\
 			7+6(n-2) &:n \in \N \cap [2,\infty)
 		\end{cases}$ 
 		\item $\lay \lp \tun_n \rp = \lp l_0, l_1,...,l_{L-1}, l_L \rp = \lp 1,2,...,2,1 \rp $ 
	\end{enumerate}
\end{lemma}
\begin{proof}
	Note that $\aff_{0,1} \in C \lp \R, \R\rp$ by Lemma 2.3.2 in \cite{bigbook} and by Lemma 2.2.7 in \cite{bigbook} we have that $\id_1 \in C\lp \R, \R\rp$. Finally, the composition of continuous functions is continuous, hence $\tun_n \in C\lp \R, \R\rp$ for $n \in \N \cap \lb 2,\infty\rp$. This proves Item (i).
	
	Note that by Lemma 2.3.2 in \cite{bigbook} it is the case that $\dep\lp \aff_{1,0} \rp = 1$ and by Definition \ref{7.2.1} it is the case that $\dep \lp \id_1 \rp = 2$. 
	Assume now that for all $n \les N$ that $\dep\lp \tun_n \rp = n$, then for the inductive step, by Proposition 2.6 in \cite{grohs2019spacetime} we have that:
	\begin{align}
		\dep \lp \tun_{n+1} \rp  &= \dep \lp \bullet^{n-1} \id_1 \rp  \nonumber \\
		&= \dep \lp \lp \bullet^{n-2} \id_1 \rp \bullet \id_1 \rp \nonumber \\
		&=n+2-1 = n+1
	\end{align}
	This completes the induction and proves Item (i)\textemdash(iii). 
	Note next that by (\ref{5.1.11}) we have that:
	\begin{align}
		\lp \real_{\rect} \lp \aff_{1,0} \rp \rp \lp x \rp = x
	\end{align}
	Lemma 2.2.7 in \cite{bigbook}, Item (ii) also tells us that:
	\begin{align}
		\lp \real_{\rect} \lp \id_1 \rp \rp \lp x \rp = \rect(x) - \rect(-x) = x
	\end{align}
	Assume now that for all $n\les N$ that $\tun_n \lp x \rp = x$. For the inductive step, by Lemma 2.27 in \cite{bigbook}, Item (iii), and we then have that:
	\begin{align}
		\lp \real_{\rect} \lp \tun_{n+1} \rp \rp \lp x \rp &= \lp \real_{\rect} \lp \bullet^{n-1} \id_1 \rp \rp \lp x \rp \lp x \rp \nonumber\\
		&= \lp \real_{\rect} \lp \lp   \bullet^{n-2} \id_1 \rp \bullet \id_1 \rp \rp \nonumber\\
		&= \lp \lp \real_{\rect} \lp \bullet^{n-2} \id_1 \rp \rp \circ \lp \real_{\rect} \lp \id_1 \rp \rp \rp \lp x \rp \nonumber \\
		&= \lp \lp \real_{\rect} \lp \tun_n \rp \rp \circ \lp \real_{\rect} \lp \id_1 \rp \rp \rp \lp x \rp \nonumber \\
		&= x
	\end{align}
	This proves Item (ii). Next note that $\param\lp \tun_1\rp = \param\lp \aff_{1,0}\rp = 2$. Note also that:
	\begin{align}
		\param\lp \tun_2\rp = \param \lp \id_1 \rp &= \param \lb \lp \lp \begin{bmatrix}
			1 \\ -1
		\end{bmatrix}, \begin{bmatrix}
			0 \\ 0
		\end{bmatrix}\rp, \lp \begin{bmatrix}
			1 & -1
		\end{bmatrix}, \begin{bmatrix}
			0
		\end{bmatrix}\rp \rp \rb \nonumber \\
		&= 7 \nonumber
	\end{align}
	And that by definition of composition:
	\begin{align}
		\param \lp \tun_3 \rp &= \param \lb \lp \lp \begin{bmatrix}
			1 \\ -1
		\end{bmatrix}, \begin{bmatrix}
			0 \\ 0
		\end{bmatrix}\rp, \lp \begin{bmatrix}
			1 & -1
		\end{bmatrix}, \begin{bmatrix}
			0
		\end{bmatrix}\rp \rp \bullet \lp \lp \begin{bmatrix}
			1 \\ -1
		\end{bmatrix}, \begin{bmatrix}
			0 \\ 0
		\end{bmatrix}\rp, \lp \begin{bmatrix}
			1 & -1
		\end{bmatrix}, \begin{bmatrix}
			0
		\end{bmatrix}\rp \rp \rb \nonumber \\
		&= \param \lb \lp \lp \begin{bmatrix}
			1 \\ -1
		\end{bmatrix}, \begin{bmatrix}
			0 \\ 0
		\end{bmatrix} \rp, \lp \begin{bmatrix}
			1 & -1 \\ -1 & 1
		\end{bmatrix}, \begin{bmatrix}
			0 \\ 0
		\end{bmatrix}\rp, \lp \begin{bmatrix}
			1&-1
		\end{bmatrix},\begin{bmatrix}
			0
		\end{bmatrix}\rp \rp \rb \nonumber \\
		&=13 \nonumber
	\end{align}
	Now for the inductive step assume that for all $n\les N\in \N$, it is the case that $\param\lp \tun_n \rp = 7+6(n-2)$. For the inductive step, we then have:
	\begin{align}
		&\param \lp \tun_{n+1} \rp = \param \lp \tun_n \bullet \id_1 \rp \nonumber\\
		&=\param \lb \lp \lp \begin{bmatrix}
			1 \\ -1
		\end{bmatrix}, \begin{bmatrix}
			0 \\ 0
		\end{bmatrix}\rp, \lp \begin{bmatrix}
			1 & -1 \\ -1 & 1
		\end{bmatrix}, \begin{bmatrix}
			0 \\0
		\end{bmatrix}\rp, \cdots, \lp \begin{bmatrix}
			1 & -1
		\end{bmatrix}, \begin{bmatrix}
			0
		\end{bmatrix}\rp \rp \bullet \id_1  \rb \nonumber \\
		&= \param \lb \lp \lp \begin{bmatrix}
			1 \\ -1
		\end{bmatrix}, \begin{bmatrix}
			0 \\ 0
		\end{bmatrix}\rp, \lp \begin{bmatrix}
			1 & -1 \\ -1 & 1
		\end{bmatrix}, \begin{bmatrix}
			0 \\0
		\end{bmatrix}\rp, \cdots, \lp \begin{bmatrix}
			1 & -1 \\ -1 & 1
		\end{bmatrix}, \begin{bmatrix}
			0 \\ 0
		\end{bmatrix} \rp, \lp \begin{bmatrix}
			1 & -1
		\end{bmatrix}, \begin{bmatrix}
			0
		\end{bmatrix}\rp \rp  \rb \nonumber \\
		&=7+6(n-2)+6 = 7+6\lp \lp n+1 \rp -2 \rp  
	\end{align}
	This proves Item (iv).
	
	Note finally that Item (v) is a consequence of Lemma 2.2.7, Item (i), in \cite{bigbook} and Proposition 2.6, in \cite{grohs2019spacetime}. This completes the proof of the Lemma.
\end{proof} 
\subsection{The $\pwr^{q,\ve}_n$ Networks and the Build Up to Them.}
\begin{definition}[The $\mathfrak{i}_d$ Network]\label{def:mathfrak_i}
	For all $d \in \N$ we will define the following set of neural networks as ``activation neural networks'' denoted $\mathfrak{i}_d$ as:
	\begin{align}
		\mathfrak{i}_d = \lp \lp \mathbb{I}_d, \mymathbb{0}_d\rp, \lp \mathbb{I}_d, \mymathbb{0}_d\rp \rp 
	\end{align}
\end{definition}
\begin{lemma}\label{lem:mathfrak_i}
	Let $d \in \N$. It is then the case that:
	\begin{enumerate}
		\item $\real_{\rect} \lp \mathfrak{i}_4\rp \in C \lp \R^d, \R^d\rp$.
		\item $\lay \lp \mathfrak{i}_d\rp = \lp d,d,d\rp$
		\item $\param \lp \mathfrak{i}_4\rp = 2d^2+2d$
	\end{enumerate}
\end{lemma}
\begin{proof}
	Item (i) is straightforward from the fact that for all $d \in \N$ it is the case that  $\real_{\rect} \lp \mathfrak{i}_d\rp = \mathbb{I}_d\lp \real_{\rect} \lp \lb \mathbb{I}_d\rb_*\rp + \mymathbb{0}_d\rp + \mymathbb{0}_d$. Item (ii) is straightforward from the fact that $\mathbb{I}_d \in \R^{d \times d}$. We realize Item (iii) by observation.
\end{proof}

\begin{lemma}[The $\Phi_k$ network]\label{lem:6.1.1}\label{lem:phi_k}
	Let $\lp c_k \rp _{k \in \N} \subseteq \R$, $\lp A_k \rp _{k \in \N} \in \R^{4 \times 4},$ $\mathbb{B}\in \R^{4 \times 1}$, $\lp C_k \rp _{k\in \N}$ satisfy for all $k \in \N$ that:
	\begin{align}\label{(6.0.1)}
		A_k = \begin{bmatrix}
			2 & -4 &2 & 0 \\
			2 & -4 & 2 & 0\\
			2 & -4 & 2 & 0\\
			-c_k & 2c_k & -c_k & 1
		\end{bmatrix} \quad B=\begin{bmatrix}
			0 \\ -\frac{1}{2} \\ -1 \\ 0
		\end{bmatrix} \quad C_k = \begin{bmatrix}
			-c_k & 2c_k &-c_k & 1
		\end{bmatrix}
	\end{align}
	and that:
	\begin{align}
		c_k = 2^{1-2k}
	\end{align}
	Let $\Phi_k \in \neu$, $k\in \N$ satisfy for all $k \in [2,\infty) \cap \N$ that $\Phi_1 = \lp \aff_{C_1,0} \bullet \mathfrak{i}_4 \rp \bullet \aff_{\mymathbb{e}_4,B}$, that for all $d \in \N$, $\mathfrak{i}_d = \lp \lp \mathbb{I}_d, \mymathbb{0}_d \rp, \lp \mathbb{I}_d, \mymathbb{0}_d \rp \rp$  and that:
	\begin{align}
		\Phi_k =\lp \aff_{C_k,0}\bullet \mathfrak{i}_4 \rp \bullet \lp \aff_{A_{k-1},B} \bullet \mathfrak{i}_4\rp \bullet \cdots \bullet \lp  \aff_{A_1,B} \bullet \mathfrak{i}_4 \rp \bullet \aff_{\mymathbb{e}_4,B} 
	\end{align}
	It is then the case that:
	\begin{enumerate}
	\item for all $k \in \N$, $x \in \R$ we have $\real_{\rect}\lp \Phi_k\rp\lp x \rp \in C \lp \R, \R \rp $
	\item for all $k \in \N$ we have $\lay \lp \Phi_k \rp = \lp 1,4,4,...,4,1 \rp \in \N^{k+2}$
	\item for all $k \in \N$, $x \in \R \setminus \lb 0,1 \rb $ that $\lp \real_{\rect} \lp \Phi_k \rp \rp \lp x \rp = \rect \lp x \rp$
	\item for all $k \in \N$, $x \in \lb 0,1 \rb$, we have $\left| x^2 - \lp \real_{\rect} \lp \xi_k \rp \rp \lp x \rp \right| \les 2^{-2k-2}$, and 
	\item for al $k \in \N$ , we have that $\param \lp \Phi_k \rp = 20k-7$ 
	\end{enumerate} 
\end{lemma}
\begin{proof}
	Let $g_k: \R \rightarrow \lb 0,1 \rb$, $k \in \N$ be the functions defined as such, satisfying for all $k \in \N$, $x \in \R$ that:
	\begin{align}\label{(6.0.3)}
		g_1 \lp x \rp &= \begin{cases}
			2x & : x \in \lb 0,\frac{1}{2} \rp \\
			2-2x &: x\in \lb \frac{1}{2},1\rb \\
			0 &: x \in \R \setminus \lb 0,1 \rb 
		\end{cases} \\
		g_{k+1} &= g_1(g_{k}) \nonumber
	\end{align}
	and let $f_k: \lb 0,1 \rb \rightarrow \lb 0,1 \rb$, $k \in \N_0$ be the functions satisfying for all $k \in \N_0$, $n \in \{0,1,...,2^k-1\}$, $x \in \lb \frac{n}{2^k}, \frac{n+1}{2^k} \rp$ that $f_k(1)=1$ and:
	\begin{align}\label{(6.0.4.2)}
		f_k(x) = \lb \frac{2n+1}{2^k} \rb x-\frac{n^2+n}{2^{2k}}
	\end{align} 
	and let $r_k = \lp r_{k,1},r_{k,2},r_{k,3},r_{k,4} \rp: \R \rightarrow \R^4$, $k \in \N$ be the functions which which satisfy for all $x \in \R$, $k \in \N$ that:
	\begin{align}\label{(6.0.5)}
		r_1\lp x \rp &= \begin{bmatrix}
			r_{1,1}(x) \\ r_{2,1}(x) \\ r_{3,1}(x) \\ r_{4,1}(x)
		\end{bmatrix}= \rect  \lp \begin{bmatrix}
			x \\ x-\frac{1}{2} \\ x-1 \\ x 
		\end{bmatrix} \rp  \\
		r_{k+1} &= A_{k+1}r_k(x) \nonumber
	\end{align}
	Note that since it is the case that for all $x \in \R$ that $\rect(x) = \max\{x,0\}$, (\ref{(6.0.3)}) and (\ref{(6.0.5)}) shows that it holds for all $x \in \R$ that:
	\begin{align}\label{6.0.6}
		2r_{1,1}(x) -4r_{2,1}(x) + 2r_{3,1}(x) &= 2 \rect(x) -4\rect \lp x-\frac{1}{2}\rp+2\rect\lp x-1\rp \nonumber \\
		&= 2\max\{x,0\} -4\max\left\{x-\frac{1}{2} ,0\right\}+2\max\{x-1,0\} \nonumber \\
		&=g_1(x) 
	\end{align}
	Note also that combined with (\ref{(6.0.4.2)}), the fact that for all $x\in [0,1]$ it holds that $f_0(x) = x = \max\{x,0\}$ tells us that for all $x \in \R$:
	\begin{align}\label{6.0.7}
		r_{4,1}(x) = \max \{x,0\} = \begin{cases}
			f_0(x) & :x\in [0,1] \\
			\max\{x,0\}& :x \in \R \setminus \lb 0,1\rb 
		\end{cases}
	\end{align} 
	We next claim that for all $k \in \N$, it is the case that:
	\begin{align}\label{6.0.8}
		\lp \forall x \in \R : 2r_{1,k}(x)-4r_{2,k}(x) + 2r_{3,k}(x) =g(x) \rp 
	\end{align}
	and that:
	\begin{align}\label{6.0.9}
		\lp \forall x \in \R: r_{4,k} (x) = \begin{cases}
			f_{k-1}(x) & :x \in \lb 0,1 \rb \\
			\max\{x,0\} & : x \in \R \setminus \lb 0,1\rb 
		\end{cases} \rp 
	\end{align}
	We prove (\ref{6.0.8}) and (\ref{6.0.9}) by induction. The base base of $k=1$ is proved by (\ref{6.0.6}) and (\ref{6.0.7}). For the induction step $\N \ni k \rightarrow k+1$ assume there does exist a $k \in \N$ such that for all $x \in \R$ it is the case that:
	\begin{align}
		2r_{1,k}(x) - 4r_{2,k}(x) + 2r_{3,k}(x) = g_k(x)
	\end{align}
	and:
	\begin{align}\label{6.0.11}
		r_{4,k}(x) = \begin{cases}
			f_{k-1}(x) & : x \in [0,1] \\
			\max\{x,0\} &: x \in \R \setminus \lb 0,1 \rb 
		\end{cases}
	\end{align}
	Note that then (\ref{(6.0.3)}),(\ref{(6.0.5)}), and (\ref{6.0.6}) then tells us that for all $x \in \R$ it is the case that:
	\begin{align}\label{6.0.12}
		g_{k+1}\lp x \rp &= g_1(g_k(x)) = g_1(2r_{1,k}(x)+4r_{2,k}(x) + 2r_{3,k}(x)) \nonumber \\
		&= 2\rect \lp 2r_{1,k}(x)) + 4r_{2,k} +2r_{3,k}(x) \rp \nonumber \\
		&-4\rect \lp 2r_{1,k}\lp x \rp -4r_{2,k}+2r_{3,k}(x) - \frac{1}{2} \rp \nonumber \\
		&+ 2\rect \lp 2r_{1,k} (x) - 4r_{2,k}(x) + 2r_{3,k}(x)-1 \rp \nonumber \\
		&=2r_{1,k+1}(x) -4r_{2,k+1}(x) + 2r_{3,k+1}(x)
	\end{align} 
	In addition note that (\ref{(6.0.4.2)}), (\ref{(6.0.5)}), and (\ref{6.0.7}) tells us that for all $x \in \R$:
	\begin{align}\label{6.0.13}
		r_{4,k+1}(x) &= \rect \lp \lp -2 \rp ^{3-2 \lp k+1 \rp }r_{1,k} \lp x \rp + 2^{4-2 \lp k+1 \rp}r_{2,k} \lp x \rp  + \lp -2 \rp^{3-2\lp k+1\rp }r_{3,k} \lp x \rp  + r_{4,k} \lp x\rp \rp \nonumber \\
		&= \rect \lp \lp -2 \rp ^{1-2k}r_{1,k} \lp x \rp + 2^{2-2k}r_{k,2}\lp x \rp + \lp -2 \rp ^{1-2k}r_{3,k} \lp x \rp + r_{4,k}\lp x \rp \rp \nonumber \\
		&=\rect \lp 2^{-2k} \lb -2r_{1,k}\lp x \rp + 2^2r_{2,k} \lp x \rp -2r_{3,k} \lp x \rp \rb +r_{4,k}\lp x \rp \rp \nonumber \\
		&= \rect \lp - \lb 2^{-2k} \rb \lb 2r_{1,k}\lp x \rp -4r_{2,k} \lp x \rp +2r_{3,k}\lp x \rp \rb +r_{4,k}\lp x \rp \rp \nonumber \\
		&= \rect\lp -\lb 2^{-2k} \rb g_k \lp x \rp +r_{4,k}\lp x \rp \rp 
	\end{align}
	This and the fact that for all $x\in \R$ it is the case that $\rect \lp x \rp = \max\{x,0\}$, that for all $x\in \lb 0 ,1 \rb$ it is the case that $f_k \lp x \rp \ges 0$, (\ref{6.0.11}), shows that for all $x \in \lb 0,1 \rb$ it holds that:
	\begin{align}\label{6.0.14}
		r_{4,k+1}\lp x \rp &= \rect \lp -2 \lb 2^{-2k} g_k \rb + f_{k-1}\lp x \rp  \rp = \rect \lp -2 \lp 2^{-2k}g_k \lp x \rp \rp +x-\lb \sum^{k-1}_{j=1} \lp 2^{-2j}g_j \lp x \rp \rp \rb \rp \nonumber \\
		&= \rect \lp x - \lb \sum^k_{j=1}2^{-2j}g_j \lp x \rp \rb \rp = \rect \lp f_k \lp x \rp \rp =f_k \lp x \rp 
	\end{align}
	Note next that (\ref{6.0.11}) and (\ref{6.0.13}) then tells us that for all $x\in \R \setminus \lb 0,1\rb$:
	\begin{align}
		r_{4,k+1}\lp x \rp = \max \left\{ -\lp 2^{-2k}g_x \lp x \rp \rp + r_{4,k}\lp x \rp \right\} = \max\{\max\{x,0\},0\} = \max\{x,0\}
	\end{align}
	Combining (\ref{6.0.12}) and (\ref{6.0.14}) proves (\ref{6.0.8}) and (\ref{6.0.9}). Note that then (\ref{(6.0.1)}) and (\ref{6.0.8}) assure that for all $k\in \N$, $x\in \R$ it holds that $\real_{\rect} \lp \Phi_k \rp \in C \lp \R,\R \rp$  and that:
	\begin{align}\label{(6.0.17)}
		&\lp \real_{\rect} \lp \Phi_k \rp \rp \lp x \rp \nonumber \\
		&= \lp \real_{\rect} \lp \lp \aff_{C_k,0} \bullet \mathfrak{i}_4 \rp \bullet \lp \aff_{A_{k-1},B} \bullet \mathfrak{i}_4 \rp \bullet \cdots \bullet\lp \aff_{A_1,B} \bullet \mathfrak{i}_4 \rp \bullet \aff_{\mymathbb{e}_4,B} \rp \rp \lp x \rp \nonumber \\
		&= \lp -2\rp^{1-2k}r_{1,k}\lp x \rp + 2^{2-2k} r_{2,k} \lp x \rp + \lp -2 \rp ^{1-2k} r_{3,k} \lp x \rp + r_{4,k} \lp x \rp \nonumber \\
		&=\lp -2 \rp ^{2-2k} \lp \lb \frac{r_{1,k}\lp x \rp +r_{3,k} \lp x \rp }{-2} \rb + r_{2,k}\lp x \rp \rp +r_{4,k}\lp x \rp \nonumber \\
		&=2^{2-2k} \lp \lb \frac{r_{1,k}\lp x \rp+r_{3,k} \lp x \rp }{-2} \rb + r_{2,k} \lp x \rp \rp +r_{4,k} \lp x \rp \nonumber \\
		&=2^{-2k}\lp 4r_{2,k} \lp x \rp -2r_{1,k}\lp x \rp -2r_{3,k} \lp x \rp \rp +r_{4,k} \lp x \rp \nonumber \\
		&=-\lb 2^{-2k} \rb \lb 2r_{1,k} \lp x \rp -4r_{2,k} \lp x \rp +2r_{3,k} \lp x \rp \rb +r_{4,k} \lp x \rp = -\lb 2^{-2k} \rb g_k \lp x \rp + r_{4,k} \lp x \rp 
	\end{align}
	This and (\ref{6.0.9}) tell us that:
	\begin{align}
		\lp \real_{\rect} \lp \Phi_k \rp \rp \lp x \rp = - \lp 2^{-2k}g_k \lp x \rp \rp +f_{k-1}\lp x \rp &= -\lp 2^{-2k}g_k \lp x \rp \rp +x-\lb \sum^{k-1}_{j=1} 2^{-2j}g_j \lp x \rp \rb \nonumber \\
		&=x-\lb \sum^k_{j=1}2^{-2j}g_j \lp x \rp \rb =f_k\lp x\rp \nonumber
	\end{align}
	Which then implies for all $k\in \N$, $x \in \lb 0,1\rb$ that it holds that:
	\begin{align}
		\left\| x^2-\lp \real_{\rect} \lp \Phi_k \rp \rp \lp x \rp \right\| \les 2^{-2k-2}
	\end{align}
	This, in turn, establishes Item (i). 
	
	Finally observe that (\ref{(6.0.17)}) then tells us that for all $k\in \N$, $x \in \R \setminus \lb 0,1\rb$ it holds that:
	\begin{align}
		\lp \real_{\rect} \lp \Phi_k \rp \rp \lp x \rp = -2^{-2k}g_k \lp x \rp +r_{4,k} \lp x \rp =r_{4,k} \lp x \rp = \max\{x,0\} = \rect(x)
	\end{align}
	This establishes Item(iv). Note next that Item(iii) ensures for all $k\in \N$ that $\dep\lp \xi_k \rp = k+1$, and:
	\begin{align}
		\param \lp \Phi_k \rp = 4(1+1) + \lb \sum^k_{j=2} 4 \lp 4+1\rp \rb + \lp 4+1 \rp =8+20\lp k-1\rp+5 = 20k-7
	\end{align}
	This, in turn, proves Item(vi). The proof of the lemma is thus complete. 
\end{proof}

\begin{corollary}\label{6.1.1.1}\label{cor:phi_network}
	Let $\ve \in \lp 0,\infty\rp$, $M= \min \{ \frac{1}{2}\log_2 \lp \ve^{-1} \rp -1,\infty\}\cap \N$, $\lp c_k\rp_{k \in \N} \subseteq \R$, $\lp A_k\rp_{k\in\N} \subseteq \R^{4 \times 4}$, $B \in \R^{4\times 1}$, $\lp C_k\rp_{k\in \N}$ satisfy for all $k \in \N$ that:
	\begin{align}
		A_k = \begin{bmatrix}
			2&-4&2&0 \\
			2&-4&2&0\\
			2&-4&2&0\\
			-c_k&2c_k & -c_k&1
		\end{bmatrix}, \quad B = \begin{bmatrix}
			0\\ -\frac{1}{2}\quad \\ -1 \\ 0
		\end{bmatrix}\quad C_k = \begin{bmatrix}
			-c_k &2c)_k&-c_k&1
		\end{bmatrix}
	\end{align}
	where:
	\begin{align}
		c_k = 2^{1-2k}
	\end{align}
	and let $\Phi \in \neu$ be defined as:
	\begin{align}
		\Phi = \begin{cases}\label{def:Phi}
			\lb \aff_{C_1,0}\bullet \mathfrak{i}_4\rb \bullet \aff_{\mymathbb{e}_4,B} & M=1 \\
			\lb \aff_{C_M,0} \bullet \mathfrak{i}_4\rb\bullet \lb \aff_{A_{M-1},0} \bullet \mathfrak{i}_4 \rb \bullet \cdots \bullet \lb \aff_{A_1,B}\bullet \mathfrak{i}_4\rb \bullet \aff_{\mymathbb{e}_4,B} & M \in \lb 2,\infty \rp \cap \N
		\end{cases}
	\end{align}
	it is then the case that:
	\begin{enumerate}
		\item $\real_{\rect} \lp \Phi\rp \in C \lp \R,\R\rp$
		\item $\lay \lp \Phi\rp = \lp 1,4,4,...,4,1\rp \in \N^{M+2} $
		\item it holds for all $x \in \R \setminus\lb 0,1 \rb$ that $\lp \real_{\rect} \lp \Phi\rp\rp \lp x \rp = \rect(x)$
		\item it holds for all $x \in \lb 0,1 \rb$ that $\left| x^2 - \lp \real_{\rect} \lp \Phi \rp \rp\lp x \rp \right| \les 2^{-2M-2} \les \ve$
		\item $\dep \lp \Phi \rp \les M+1 \les \max\{ \frac{1}{2}\log_2 \lp \ve^{-1}\rp+1,2\}$, and
		\item $\param \lp \Phi\rp = 20M-7 \les \max\left\{ 10\log_2 \lp \ve^{-1}\rp-7,13\right\}$
	\end{enumerate}
	\end{corollary}
	\begin{proof}
		Items (i)--(iii) are direct consequences of Lemma \ref{lem:6.1.1}, Items (i)--(iii). Note next the fact that $M = \min \left\{\N \cap \lb \frac{1}{2} \log_2 \lp \ve^{-1}\rp-1\rb,\infty\right\}$ ensures that:
		\begin{align}
			M = \min \left\{ \N \cap \lb \frac{1}{2}\log_2\lp \ve^{-1}\rp-1\rb, \infty\right\} \ges \min \left\{ \lb\max \left\{ 1,\frac{1}{2}\log_2 \lp\ve^{-1} \rp-1\right\},\infty \rb\right\} \ges \frac{1}{2}\log_2 \lp \ve^{-1}\rp-1
		\end{align}
		This and Item (v) of Lemma \ref{lem:6.1.1} demonstrate that for all $x\in \lb 0,1\rb$ it then holds that:
		\begin{align}
			\left| x^2 - \lp \real_{\rect}\lp \Phi\rp\rp \lp x\rp \right| \les 2^{-2M-2} = 2^{-2(M+1)} \les 2^{-\log_2\lp\ve^{-1} \rp} = \ve
		\end{align}
		Thus establishing Item (iv). The fact that $M = \min \left\{ \N \cap \lb \frac{1}{2}\log_2 \lp \ve^{-1}\rp -1,\infty\rb\right\}$ and Item (ii) of Lemma \ref{lem:6.1.1} tell us that:
		\begin{align}
			\dep \lp \Phi \rp = M+1 \les \max \left\{ \frac{1}{2} \log_2 \lp \ve^{-1}\rp+1,2\right\}
		\end{align}
		Which establishes Item(v). This and Item (v) of Lemma \ref{lem:6.1.1} then tell us that:
		\begin{align}
			\param \lp \Phi_M\rp \les 20M-7 \les 20 \max\left\{ \frac{1}{2}\log_2\lp\ve^{-1}\rp,2\right\}-7 = \max\left\{ 10\log_2 \lp\ve^{-1} \rp-7,13\right\}
		\end{align}
		This completes the proof of the corollary. 
	\end{proof}
\begin{lemma}\label{6.0.3}\label{lem:sqr_network}
	Let $\delta,\epsilon \in (0,\infty)$, $\alpha \in (0,\infty)$, $q\in (2,\infty)$, $ \Phi \in \neu$ satisfy that $\delta = 2^{\frac{-2}{q-2}}\ve ^{\frac{q}{q-2}}$, $\alpha = \lp \frac{\ve}{2}\rp^{\frac{1}{q-2}}$, $\real{\rect}\lp\Phi\rp \in C\lp \R,\R\rp$, $\dep(\Phi) \les \max \left\{\frac{1}{2} \log_2(\delta^{-1})+1,2\right\}$, $\param(\Phi) \les \max\left\{10\log_2\lp \delta^{-1}\rp-7,13\right\}$, $\sup_{x \in \R \setminus [0,1]} | \lp \real_{\rect} \lp \Phi \rp -\rect(x) \right| =0$, and $\sup_{x\in \lb 0,1\rb} |x^2-\lp \real_{\rect} \lp \Phi \rp \rp \lp x\rp | \les \delta$, let $\Psi \in \neu$ be the neural network given by:
	\begin{align}
		 \Psi = \lp \aff_{\alpha^{-2},0} \bullet \Phi \bullet \aff_{\alpha,0} \rp \bigoplus\lp \aff_{\alpha^{-2},0} \bullet \Phi \bullet \aff_{-\alpha,0}\rp 
	\end{align}
	\begin{enumerate}
		\item it holds that $\real_{\rect} \lp \Psi \rp \in C \lp \R,\R \rp$.
		\item it holds that $\lp \real_{\rect} \lp \Psi \rp \rp \lp 0\rp=0$
		\item it holds for all $x\in \R$ that $0\les \lp \real_{\rect} \lp \Psi \rp \rp \lp x \rp  \les \ve + |x|^2$
		\item it holds for all $x \in \R$ that $|x^2-\lp \real_{\rect} \lp \Psi  \rp \rp \lp x \rp |\les \ve \max\{1,|x|^q\}$
		\item it holds that $\dep (\Psi)\les \max\left\{1+\frac{1}{q-2}+\frac{q}{2(q-2)}\log_2 \lp \ve^{-1} \rp,2\right\}$, and
		\item it holds that $\param\lp \Psi \rp \les \max\left\{ \lb \frac{40q}{q-2} \rb \log_2 \lp \ve^{-1} \rp +\frac{80}{q-2}-28,52 \right\}$ 
	\end{enumerate}
\end{lemma}
\begin{proof}
	Note that for all $x\in \R$ it is the case that:
	\begin{align}\label{6.0.21}
		\lp \real_{\rect}\lp \Psi  \rp \rp\lp x \rp &= \lp \real_{\rect} \lp \lp  \aff_{\alpha^{-2}}\bullet \Phi \bullet \aff_{\alpha,0}\rp  \oplus\lp \aff_{\alpha^{-2},0} \bullet \Phi \bullet \aff_{-\alpha,0} \rp \rp \rp \lp x \rp \nonumber\\
		&=  \lp \real_{\rect}\lp \aff_{\alpha^{-2},0} \bullet \Phi \bullet \aff_{\alpha,0} \rp \rp \lp x\rp + \lp \real_{\rect}\lp \aff_{\alpha^{-2},0} \bullet  \Phi \bullet \aff_{-\alpha,0}\rp \rp \lp x\rp \nonumber \\
		&= \frac{1}{\alpha^2}\lp \real_{\rect}\lp \Phi \rp \rp \lp \alpha x\rp + \frac{1}{\alpha^2}\lp \real_{\rect} \lp \Phi \rp \rp \lp -\alpha x\rp \nonumber\\
		&= \frac{1}{\lp \frac{\ve}{2}\rp^{\frac{2}{q-2}}}\lb \lp \real_{\rect}\lp \Phi \rp \rp \lp \lp \frac{\ve}{2}\rp ^{\frac{1}{q-2}}x \rp + \lp \real_{\rect}\lp \Phi \rp \rp \lp -\lp \frac{\ve}{2}\rp^{\frac{1}{q-2}}x\rp \rb 
	\end{align}
	This and the assumption that $\Phi \in C\lp \R, \R \rp$ along with the assumption that $\sup_{x\in \R \setminus \lb 0,1\rb } | \lp \real_{\rect} \lp \Phi \rp \rp \lp x \rp -\rect\lp x\rp | =0$ tells us that for all $x\in \R$ it holds that:
	\begin{align}
		\lp \real_{\rect}\lp \Psi \rp \rp \lp 0 \rp &= \lp \frac{\ve}{2}\rp^{\frac{-2}{q-2}}\lb \lp \real_{\rect}\lp \Phi \rp \rp  \lp 0 \rp +\lp \real_{\rect} \lp \Phi\rp \rp \lp 0 \rp \rb \nonumber \\
		&=\lp \frac{\ve}{2}\rp ^{\frac{-2}{q-2}} \lb \rect (0)+\rect(0) \rb \nonumber \\
		&=0
	\end{align} 
	This, in turn, establishes Item (i)--(ii). Observe next that from the assumption that $\real_{\rect} \lp \Phi \rp \in C\lp \R,\R \rp$ and the assumption that $\sup_{x\in \R \setminus \lb 0,1\rb} | \lp \real_{\rect}\lp \Phi \rp \rp \lp x \rp -\rect(x) |=0$ ensure that for all $x\in \R \setminus \lb -1,1 \rb$ it holds that:
	\begin{align}\label{6.0.23}
		\lb \real_{\rect}\lp \Phi \rp \rb \lp x\rp + \lb \real_{\rect}\lp \Phi \rp  \lp -x \rp\rb   = \rect\lp x\rp +\rect(-x) &= \max\{x,0\}+\max\{-x,0\} \nonumber\\
		&=|x|
	\end{align}
	The assumption that for all $\sup_{x\in \R \setminus \lb 0,1\rb }|\lp \real_{\rect} \lp \Phi \rp \rp \lp x\rp -\rect\lp x\rp |=0$ and the assumption that $\sup_{x\in\lb 0,1\rb} |x^2-\lp \real_{\rect} \lp \Phi \rp \rp \lp x\rp |\les \delta$ show that:
	\begin{align}\label{6.0.24}
		&\sup_{x \in \lb -1,1\rb} \left|x^2 - \lp \lb \real_{\rect}\lp \Phi \rp \rb \lp x\rp +\lb \real_{\rect}\lp \Phi \rp \lp x \rp \rb \rp \right| \nonumber \\
		&= \max\left\{ \sup_{x\in \lb -1,0 \rb} \left| x^2-\lp \rect(x)+ \lb \real_{\rect}\lp \Phi \rp \rb \lp -x \rp \rp \right|,\sup _{x\in \lb 0,1 \rb} \left| x^2-\lp \lb \real_{\rect} \lp \Phi \rp \rb \lp x \rp + \rect \lp -x \rp \rp \right| \right\} \nonumber\\
		&= \max\left\{\sup_{x\in \lb -1,0 \rb}\left|\lp -x \rp^2 - \lp \real_{\rect}\lp \Phi \rp \rp \lp -x \rp \right|, \sup_{x\in \lb 0,1\rb} \left| x^2-\lp \real_{\rect} \lp \Phi \rp \rp \lp x \rp \right| \right\} \nonumber \\
		&=\sup_{x\in \lb 0,1 \rb}\left| x^2 - \lp \real_{\rect}\lp \Phi \rp \rp \lp x\rp \right| \les \delta 
	\end{align}
	Next observe that (\ref{6.0.21}) and (\ref{6.0.23}) show that for all $x \in \R \setminus \lb -\lp \frac{\ve}{2} \rp^{\frac{-1}{q-2}}, \lp \frac{\ve}{2}\rp ^{\frac{-1}{q-2}} \rb$ it holds that:
	\begin{align}\label{6.0.25}
		0 \les \lb \real_{\rect} \lp \Psi \rp \rb \lp x \rp &= \lp \frac{\ve}{2} \rp ^{\frac{-2}{q-2}}\lp \lb \real_{\rect} \lp \Phi \rp \rb \lp \lp \frac{\ve}{2}\rp ^{\frac{1}{q-2}}x \rp + \lb \real_{\rect} \lp \Phi \rp \rb \lp -\lp \frac{\ve}{2}\rp^{\frac{1}{q-2}} x\rp \rp \nonumber \\
		&= \lp \frac{\ve}{2} \rp ^{\frac{-2}{q-2}} \left| \lp \frac{\ve}{2} \rp^{\frac{1}{q-2}}x \right| = \lp \frac{\ve}{2} \rp^{\frac{-1}{q-2}|x|} \les |x|^2
	\end{align}
	The triangle inequality then tells us that for all $x\in \R \setminus \lb - \lp \frac{\ve}{2} \rp^{\frac{-1}{q-2}}, \lp \frac{\ve}{2} \rp^{\frac{-1}{q-2}} \rb$ it holds that:
	\begin{align} \label{6.0.25}
		\left| x^2- \lp \real_{\rect} \lp \Psi \rp \rp \lp x \rp \right| &= \left| x^2 - \lp \frac{\ve}{2} \rp ^{\frac{-1}{q-2}}\left|x\right| \right| \les \lp \left|x \right|^2 + \lp \frac{\ve}{2} \rp ^{\frac{-1}{q-2}} \left| x \right| \rp   \nonumber\\
		&= \lp \left| x \right|^q \left|x\right|^{-(q-2)} + \lp \frac{\ve}{2} \rp^{\frac{-1}{q-2}} \left| x \right|^q\left| x \right|^{-(q-1)} \rp \nonumber \\
		&\les \lp \left| x \right|^q \lp \frac{\ve}{2} \rp^{\frac{q-2}{q-2}} + \lp \frac{\ve}{2} \rp ^{\frac{-1}{q-2}} \left| x \right|^q \lp \frac{\ve}{2} \rp ^{\frac{q-1}{q-2}} \rp \nonumber \\
		&= \lp \frac{\ve}{2}+ \frac{\ve}{2} \rp \left| x \right|^q = \ve \left| x \right|^q \les \ve \max \left\{ 1, \left| x \right|^q \right\} 
	\end{align}

Note that (\ref{6.0.24}), (\ref{6.0.21}) and the fact that $\delta = 2^{\frac{-2}{q-2}}\ve^{\frac{q}{q-2}}$ then tell for all $x \in \lb -\lp \frac{\ve}{2} \rp ^{\frac{-1}{q-2}}, \lp \frac{\ve}{2} \rp ^{\frac{-1}{q-2}} \rb$ it holds that:
\begin{equation}
\begin{aligned}\label{6.0.26}
	&\left| x^2-\left( \real_{\rect} (\Phi) \right) (x) \right| \\
	&= \left( \frac{\varepsilon}{2} \right)^{\frac{-2}{q-2}} \left| \left( \left( \frac{\varepsilon}{2} \right) ^{\frac{1}{q-2}}x \right)^2 - \left( \left[ \real_{\rect} (\Phi) \right] \left( \left( \frac{\varepsilon}{2} \right) ^{\frac{1}{q-2}}x \right) + \left[ \real_{\rect} (\Phi) \right] (-y) \right) \right| \\
	&\les \left( \frac{\varepsilon}{2} \right)^{\frac{-2}{q-2}} \left[ \sup_{y \in \left[-1,1\right]} \left| y^2 - \left[ \real_{\rect} (\Phi) \right] (y) + \left[ \real_{\rect} (\Phi) \right] (-y) \right| \right] \\
	&\les \lp \frac{\ve}{2} \rp^{\frac{-2}{q-2}} \delta = \lp \frac{\ve}{2} \rp^{\frac{-2}{q-2}} 2^{\frac{-2}{q-2}} \ve^{\frac{q}{q-2}} = \ve \les \ve \max \{ 1, \left| x \right|^q \}
\end{aligned}
\end{equation}
Now note that this and (\ref{6.0.25}) tells us that for all $x\in \R$ it is the case that:
\begin{align}
	\left| x^2-\lp \real_{\rect} \lp \Psi \rp \rp \lp x \rp \right| \les \ve \max\{1,|x|^q \}
\end{align}
This establishes Item (v). Note that, (\ref{6.0.26}) tells that for all $x \in \lb - \lp \frac{\ve}{2} \rp ^{\frac{-1}{q-2}}, \lp \frac{\ve}{2} \rp ^{\frac{1}{q-2}} \rb $ it is the case that:
\begin{align}
	\left| \lp \real_{\rect} \lp \Psi \rp \rp \lp x \rp \right| \les \left| x^2 - \lp \real_{\rect} \lp \Psi \rp \rp \lp x \rp \right| + \left| x \right|^2 \les \ve + \left| x \right| ^2
\end{align}
This and (\ref{6.0.25}) tells us that for all $x\in \R$:
\begin{align}
	\left| \lp \real_{\rect} \rp \lp x \rp \right| \les \ve + |x|^2  
\end{align}
This establishes Item (iv). 

Note next that by Corollary 2.9 in \cite{grohs2019spacetime}, the hypothesis, and the fact that $\delta = 2^{\frac{-2}{q-2}}\ve ^{\frac{q}{q-2}}$ tells us that:
\begin{align}
	\dep \lp \Psi \rp = \dep \lp \Phi \rp &\les \max \left\{\frac{1}{2} \log_2(\delta^{-1})+1,2\right\} \nonumber \\
	&= \max \left\{  \frac{1}{q-2} + \lb \frac{q}{q-2}\rb\log_2 \lp \ve \rp +1,2\right\}
\end{align}
This establishes Item (v). 

Notice next that the fact that $\delta = 2^{\frac{-2}{q-2}}\ve^{\frac{q}{q-2}}$ tells us that:
\begin{align}
	\log_2 \lp \delta^{-1} \rp = \log_2 \lp 2^{\frac{2}{q-2}} \ve^{\frac{-q}{q-2}}\rp = \frac{2}{q-2} + \lb \lb \frac{q}{q-2}\rb \log_2 \lp \ve^{-1}\rp  \rb
\end{align}
Note that by , Corollary 2.9 in \cite{grohs2019spacetime}, we have that:
\begin{align}
	\param \lp \Phi \bullet \aff_{-\alpha,0} \rp &\les \lb \max\left\{ 1, \frac{\inn \lp \aff_{-\alpha,0}\rp+1}{\inn\lp \Phi\rp+1}\right\}\rb \param \lp \Phi\rp = \param \lp \Phi\rp 
\end{align}
and further that:
\begin{align}
	\param \lp \aff_{\alpha^{-2},0} \bullet \Phi \bullet \aff_{-\alpha,0} \rp &= \lb \max\left\{ 1, \frac{\out \lp \aff_{-\alpha^2,0}\rp}{\out\lp \Phi \bullet \aff_{-\alpha,0}\rp}\right\}\rb \param \lp \Phi \bullet \aff_{-\alpha,0}\rp \nonumber\\
	&\les \param \lp \Phi\rp
\end{align}
By symmetry note also that $ \param \lp \aff_{\alpha^{-2},0} \bullet \Phi \bullet \aff_{\alpha,0}\rp = \param \lp \aff_{\alpha^{-2},0} \bullet \Phi \bullet \aff_{-\alpha,0}\rp $ and also that $ \lay \lp \aff_{\alpha^{-2},0} \bullet \Phi \bullet \aff_{\alpha,0}\rp = \lay \lp \aff_{\alpha^{-2},0} \bullet \Phi \bullet \aff_{-\alpha,0}\rp $. Thus Lemma \ref{paramsum}, Corollary \ref{cor:sameparal}, and the hypothesis tells us that:
\begin{align}\label{(6.1.42)}
	\param \lp \Psi \rp &= \param \lp \Phi \boxminus \Phi \rp \nonumber \\
	&\les 4\param \lp \Phi\rp \nonumber \\
	&= 4\max\left\{10\log_2\lp \delta^{-1}\rp-7,13\right\}
\end{align}
This, and the fact that $\delta = 2^{\frac{-2}{q-2}}\ve ^{\frac{q}{q-2}}$ renders (\ref{(6.1.42)}) as:
\begin{align}
	4\max\left\{10\log_2\lp \delta^{-1}\rp-7,13\right\} &= 4\max\left\{10\log_2\lp \delta^{-1}\rp-7,13\right\} \nonumber\\
	&= 4\max \left\{ 10 \lp \frac{2}{q-2} +\frac{q}{q-2}\log_2 \lp \ve^{-1}\rp\rp-7,13\right\} \nonumber \\
	&=\max \left\{ \lb \frac{40q}{q-2}\rb \log_2 \lp \ve^{-1}\rp + \frac{80}{q-2}-28,52\right\}
\end{align}
\end{proof}
\begin{remark}
	We will often find it helpful to refer to this network for fixed $\ve \in \lp 0, \infty \rp$ and $q \in \lp 2,\infty\rp$ as the $\sqr^{q,\ve}$ network.
\end{remark}

We are finally ready to give neural network representations of arbitrary products of real numbers. However, this representation differs somewhat from those found in the literature, especially \cite{grohs2019spacetime}, where parallelization (stacking) is used instead of neural network sums. This will help us calculate $\wid_1$ and the width of the second to last layer.
\begin{lemma}\label{prd_network}
	Let $\delta,\ve \in \lp 0,\infty \rp $, $q\in \lp 2,\infty \rp$, $A_1,A_2,A_3 \in \R^{1\times 2}$, $\Psi \in \neu$ satisfy for all $x\in \R$ that $\delta = \ve \lp 2^{q-1} +1\rp^{-1}$, $A_1 = \lb 1 \quad 1 \rb$, $A_2 = \lb 1 \quad 0 \rb$, $A_3 = \lb 0 \quad 1 \rb$, $\real_{\rect} \in C\lp \R, \R \rp$, $\lp \real_{\rect} \lp \Psi \rp \rp \lp 0\rp = 0$, $0\les \lp \real_{\rect} \lp \Psi \rp \rp \lp x \rp \les \delta+|x|^2$, $|x^2-\lp \real_{\rect}\lp \Psi \rp \rp \lp x \rp |\les \delta \max \{1,|x|^q\}$, $\dep\lp \Psi \rp \les \max\{ 1+\frac{1}{q-2}+\frac{q}{2(q-2)}\log_2 \lp \delta^{-1} \rp ,2\}$, and $\param \lp \Psi \rp \les \max\left\{\lb \frac{40q}{q-2} \rb \log_2\lp \delta^{-1} \rp +\frac{80}{q-2}-28,52\right\}$, then:
	\begin{enumerate}
		\item there exists a unique $\Gamma \in \neu$ satisfying:
		\begin{align}
			\Gamma = \lp \frac{1}{2}\triangleright \lp \Psi \bullet \aff_{A_1,0} \rp \rp \bigoplus \lp \lp -\frac{1}{2}\rp \triangleright\lp \Psi \bullet \aff_{A_2,0} \rp \rp \bigoplus\lp \lp -\frac{1}{2}\rp \triangleright \lp \Psi \bullet \aff_{A_3,0} \rp \rp 
		\end{align}
		\item it  that $\real_{\rect} \lp \Gamma \rp \in C \lp \R^2,\R \rp$
		\item it holds for all $x\in \R$ that $\lp \real_{\rect}\lp \Gamma \rp \rp \lp x,0\rp = \lp \real_{\rect}\lp \Gamma \rp \rp \lp 0,y\rp  =0$
		\item it holds for any $x,y \in \R$ that $\left|xy - \lp \real_{\rect} \lp \Gamma \rp \rp \lp \begin{bmatrix}
			x \\
			y
		\end{bmatrix} \rp \right| \les \ve \max \{1,|x|^q,|y|^q \}$
		\item it holds that $\param(\Gamma) \les \frac{360q}{q-2} \lb \log_2 \lp \ve^{-1} \rp +q+1 \rb -252$
		\item it holds that $\dep\lp \Gamma \rp \les \frac{q}{q-2} \lb \log_2 \lp \ve^{-1}\rp +q \rb $ 
		\item it holds that $\wid_1 \lp \Gamma \rp=24$
		\item it holds that $\wid_{\hid \lp\Gamma\rp} = 24$
	\end{enumerate}
\end{lemma}
\begin{proof}
	Note that:
	\begin{align}
		&\lp \real_{\rect} \lp \Gamma \rp \rp \lp \begin{bmatrix}
			x\\y
		\end{bmatrix} \rp =  \real_{\rect} \lp \lp \frac{1}{2}\triangleright \lp \Psi \bullet \aff_{A_1,0} \rp \rp \bigoplus \lp \lp -\frac{1}{2}\rp \triangleright\lp \Psi \bullet \aff_{A_2,0} \rp \rp \bigoplus \right. \\
		&\left. \lp \lp -\frac{1}{2}\rp \triangleright \lp \Psi \bullet \aff_{A_3,0} \rp \rp \rp \nonumber \lp \begin{bmatrix}
			x \\ y
		\end{bmatrix} \nonumber\rp\\
		&= \real_{\rect} \lp \frac{1}{2}\triangleright \lp \Psi \bullet \aff_{A_1,0} \rp \rp \lp \begin{bmatrix}
			x\\y 
		\end{bmatrix} \rp  + \real_{\rect}\lp \lp -\frac{1}{2}\rp \triangleright\lp \Psi \bullet \aff_{A_2,0} \rp \rp \lp \begin{bmatrix}
			x \\ y
		\end{bmatrix} \rp \nonumber \\
		&+\real_{\rect}\lp \lp -\frac{1}{2}\rp \triangleright \lp \Psi \bullet \aff_{A_3,0} \rp \rp \lp \begin{bmatrix}
			x\\y
		\end{bmatrix} \rp \nonumber \\
		&= \frac{1}{2} \lp \real_{\rect} \lp \Psi \rp \rp \lp \begin{bmatrix}
			1 && 1
		\end{bmatrix} \begin{bmatrix}
			x \\ y
		\end{bmatrix}\rp - \frac{1}{2} \lp \real_{\rect} \lp \Psi  \rp \rp \lp \begin{bmatrix}
			1 && 0
		\end{bmatrix} \begin{bmatrix}
			x \\ y
		\end{bmatrix} \rp \nonumber\\
		&-\frac{1}{2} \lp \real_{\rect}\lp \Psi \rp \rp \lp \begin{bmatrix}
			0 && 1
		\end{bmatrix} \begin{bmatrix}
			x \\y
		\end{bmatrix} \rp 	\nonumber \\
		&=\frac{1}{2} \lp \real_{\rect}\lp \Psi \rp \rp \lp x+y \rp -\frac{1}{2} \lp \real_{\rect}\lp \Psi \rp \rp \lp x \rp - \frac{1}{2} \lp \real_{\rect}\lp \Psi \rp \rp \lp y \rp \label{6.0.33}
	\end{align}
	Note that this, and the assumption that $\lp \real_{\rect} \lp \Psi \rp \rp \lp x \rp  \in C \lp \R, \R \rp$ and that $\lp \real_{\rect}\lp \Psi \rp \rp  \lp 0 \rp = 0$ ensures:
	\begin{align}
		\lp \real_{\rect} \lp \Gamma \rp \rp \lp \begin{bmatrix}
			x \\0
		\end{bmatrix} \rp &= \frac{1}{2} \lp \real_{\rect} \lp \Psi \rp \rp \lp x+0 \rp -\frac{1}{2} \lp \real_{\rect} \lp \Psi \rp \rp \lp x \rp - \frac{1}{2} \lp \real_{\rect} \lp \Psi \rp \rp \lp 0 \rp \nonumber \\
		&= 0 \nonumber\\
		&= \frac{1}{2} \lp \real_{\rect} \lp \Psi \rp \rp \lp 0+y \rp -\frac{1}{2} \lp \real_{\rect} \lp \Psi \rp \rp \lp 0 \rp  - \frac{1}{2}\lp \real_{\rect} \lp \Psi \rp \rp \lp y \rp \nonumber \\
		&=\lp \real_{\rect} \lp \Gamma \rp \rp \lp \begin{bmatrix}
			0 \\y 
		\end{bmatrix} \rp 
	\end{align}
	Next, observe that since by assumption it is the case for all $x,y\in \R$ that $|x^2 - \lp \real_{\rect} \lp \Psi \rp \rp \lp x \rp | \les \delta \max\{1,|x|^q\}$, $xy = \frac{1}{2}|x+y|^2-\frac{1}{2}|x|^2-\frac{1}{2}|y|^2$,  triangle Inequality and from (\ref{6.0.33}) we have that:
	\begin{align}
		&\left| \lp \real_{\rect} \lp \Gamma\rp\lp x,y \rp  \rp -xy\right| \nonumber\\
		&=\left|\frac{1}{2}\lb \lp \real_{\rect} \lp \Psi \rp \rp \lp x + y \rp - \left|x+y\right|^2 \rb - \frac{1}{2} \lb \lp \real_{\rect} \lp \Psi \rp \rp \lp x \rp -\left| x \right|^2\rb  - \frac{1}{2} \lb \lp \real_{\rect} \lp \Psi\rp \rp \lp x \rp -\left|y\right|^2\rb \right| \nonumber \\
		&\les \left|\frac{1}{2}\lb \lp \real_{\rect} \lp \Psi \rp \rp \lp x + y \rp - \left|x+y\right|^2 \rb + \frac{1}{2} \lb \lp \real_{\rect} \lp \Psi \rp \rp \lp x \rp -\left| x \right|^2\rb  + \frac{1}{2} \lb \lp \real_{\rect} \lp \Psi\rp \rp \lp x \rp -\left|y\right|^2\rb \right| \nonumber \\
		&\les \frac{\delta}{2} \lb \max \left\{ 1, |x+y|^q\right\} + \max\left\{ 1,|x|^q\right\} + \max \left\{1,|y|^q \right\}\rb\nonumber
	\end{align}
	Note also that since for all $\alpha,\beta \in \R$ and $p \in \lb 1, \infty \rp$ we have that $|\alpha + \beta|^p \les 2^{p-1}\lp |\alpha|^p + |\beta|^p \rp$ we have that:
	\begin{align}
		&\left| \lp \real_{\rect} \lp \Psi \rp \rp \lp x \rp - xy \right| \nonumber \\
		&\les \frac{\delta}{2} \lb \max \left\{1, 2^{q-1}|x|^q+ 2^{q-1}\left| y\right|^q\right\} + \max\left\{1,\left|x\right|^q \right\} + \max \left\{1,\left| y \right|^q \right\}\rb \nonumber \\
		&\les \frac{\delta}{2} \lb \max \left\{1, 2^{q-1}|x|^q \right\}+ 2^{q-1}\left| y\right|^q + \max\left\{1,\left|x\right|^q \right\} + \max \left\{1,\left| y \right|^q \right\}\rb \nonumber \\
		&\les \frac{\delta}{2} \lb 2^q + 2\rb \max \left\{1, \left|x\right|^q, \left| y \right|^q \right\} = \ve \max \left\{ 1,\left| x \right|^q, \left| x \right|^q\right\} \nonumber
	\end{align} 
	This proves Item (iv). 	
	
	By symmetry it holds that $\param \lp \frac{1}{2}\triangleright \lp \Psi \bullet \aff_{A_1,0} \rp \rp = \param  \lp -\frac{1}{2}\triangleright \lp \Psi \bullet \aff_{A_2,0} \rp \rp = \param \lp -\frac{1}{2}\triangleright \lp \Psi \bullet \aff_{A_3,0} \rp \rp$ and further that $\lay \lp \frac{1}{2}\triangleright \lp \Psi \bullet \aff_{A_1,0} \rp \rp = \lay  \lp -\frac{1}{2}\triangleright \lp \Psi \bullet \aff_{A_2,0} \rp \rp = \lay \lp -\frac{1}{2}\triangleright\lp \Psi \bullet \aff_{A_3,0} \rp \rp$. 
	Note also that Corollary 2.9 in \cite{grohs2019spacetime}, tells us that for all $i \in \{1,2,3\}$ and $a \in \{ \frac{1}{2},-\frac{1}{2}\}$ it is the case that:
	\begin{align}
		\param \lp a \triangleright \lp \Psi \bullet \aff_{A_i,0}\rp \rp = \param \lp \Psi \rp
	\end{align}
	This, together with Corollary 2.21 in \cite{grohs2019spacetime} indicates that:
	\begin{align}\label{(6.1.49)}
		\param \lp \Gamma \rp &\les 9\param\lp \Psi \rp \nonumber \\
		&\les 9\max\left\{\lb \frac{40q}{q-2} \rb \log_2\lp \delta^{-1} \rp +\frac{80}{q-2}-28,52\right\}
	\end{align}
	Combined with the fact that $\delta =\ve \lp 2^{q-1} +1\rp^{-1}$, this is then rendered as:
	\begin{align}\label{(6.1.50)}
		&9\max\left\{\lb \frac{40q}{q-2} \rb \log_2\lp \delta^{-1} \rp +\frac{80}{q-2}-28,52\right\} \nonumber \\
		&= 9\max \left\{ \lb \frac{40q}{q-2}\rb \lp \log_2 \lp \ve^{-1}\rp  +\log_2 \lp 2^{q-1}+1\rp\rp + \frac{80}{q-2}-28,52 \right\}
	\end{align}
	Note that:
	\begin{align}
		\log_2 \lp 2^{q-1}+1\rp &= \log_2\lp 2^{q-1}+1\rp - \log_2 \lp 2^q\rp + q \nonumber\\
		&=\log_2 \lp \frac{2^{q-1}+1}{2^q}\rp + q = \log_2 \lp 2^{-1}+2^{-q}\rp +q\nonumber \\
		&\les \log_2 \lp 2^{-1} + 2^{-2}\rp + q = \log_2 \lp \frac{3}{4}\rp + q = \log_2 \lp 3\rp-2+q
	\end{align}
	Combine this with the fact that for all $q\in \lp 2,\infty\rp$ it is the case that $\frac{q(q-1)}{q-2} \ges 2$ then gives us that:
	\begin{align}
		\lb \frac{40q}{q-2}\rb \log_2 \lp 2^{q-1}+1\rp -28\ges \lb \frac{40q}{q-2}\rb \log_2 \lp 2^{q-1}\rp -28= \frac{40q(q-1)}{q-2}-28 \ges 52
	\end{align}
	This then finally renders (\ref{(6.1.50)}) as:
	\begin{align}
		&9\max \left\{ \lb \frac{40q}{q-2}\rb \lp \log_2 \lp \ve^{-1}\rp  +\log_2 \lp 2^{q-1}+1\rp\rp + \frac{80}{q-2}-28,52 \right\} \nonumber \\
		&\les 9 \lb \lb \frac{40q}{q-2}\rb \lp \log_2\lp \ve^{-1}\rp + \log_2\lp 3\rp-2+q\rp +\frac{80}{q-2}-28\rb \nonumber\\
		&= 9 \lb \lb \frac{40q}{q-2}\rb \lp \log_2\lp \ve^{-1}\rp + \log_2\lp 3\rp-2+\frac{2}{q}\rp-28\rb \nonumber\\
		&\les 9 \lb \lb \frac{40q}{q-2}\rb \lp \log_2\lp \ve^{-1}\rp + \log_2\lp 3\rp-1\rp -28\rb \nonumber\\
		&= \frac{360q}{q-2}\lb \log_2 \lp \ve^{-1} \rp +q+\log_2\lp 3\rp-1\rb -252
	\end{align}
	Note that Lemma \ref{lem:sqr_network}, the hypothesis, and the fact that $\delta = \ve \lp 2^{q-1} +1\rp^{-1}$ tell us that:
	\begin{align}
		\dep \lp \Gamma \rp = \dep\lp \Psi \rp &\les \max\left\{ 1+\frac{1}{q-2}+\frac{q}{2(q-2)}\log_2 \lp \delta^{-1} \rp ,2\right\} \nonumber\\
		&= \max \left\{1+\frac{1}{q-2} +\frac{q}{2(q-2)}\lb \log_2\lp \ve^{-1}\rp + \log_2 \lp 2^{q-1}+1\rp\rb,2 \right\} \nonumber\\
		&= \max \left\{ 1+\frac{1}{q-2}+\frac{q}{2(q-2)} \lp \log_2\lp \ve^{-1}\rp +q-1\rp,2\right\}
	\end{align}
	Since it is the case that $\frac{q(q-1)}{2(q-2)} > 2$ for $q \in \lp 2, \infty \rp$ we have that:
	\begin{align}
		& \max \left\{ 1+\frac{1}{q-2}+\frac{q}{2(q-2)} \lp \log_2\lp \ve^{-1}\rp +q-1\rp,2\right\} \nonumber \\
		&=  1+\frac{1}{q-2}+\frac{q}{2(q-2)} \lp \log_2\lp \ve^{-1}\rp +q-1\rp \nonumber \\
		&\les \frac{q-1}{q-2} +\frac{q}{2\lp q-2\rp} \lp \log_2 \lp \ve^{-1}\rp+q\rp \nonumber \\
		&
	\end{align}
	
Observe next that for $q\in \lp 0,\infty\rp$, $\ve \in \lp 0,\infty \rp$, $\Gamma$ consists of, among other things, three stacked $\lp \Psi \bullet \aff_{A_i,0}\rp$ networks where $i \in \{1,2,3\}$. Definition \ref{def:stk} tells us therefore, that $\wid_1\lp \Gamma\rp = 3\cdot \wid_1 \lp \Psi \rp$. On the other hand, note that each $\Psi$ networks consist of, among other things, two stacked $\Phi$ networks, which by Corollary \ref{cor:phi_network} and Lemma \ref{lem:sqr_network}, yields that $\wid_1 \lp \Gamma\rp = 6 \cdot \wid_1 \lp \Phi\rp$. Finally from Corollary \ref{cor:phi_network}, and Corollary 2.9 in \cite{grohs2019spacetime}, we see that the only thing contributing to the $\wid_1\lp \Phi\rp$ is $\wid_1 \lp \mathfrak{i}_4\rp$, which was established from Lemma \ref{lem:mathfrak_i} as $4$. Whence we get that $\wid_1\lp \Gamma\rp = 6 \cdot 4 = 24$, and that $\wid_{\hid\lp \Gamma\rp}\lp \Gamma\rp = 24$. This proves Item (vii)\textemdash(viii). This then completes the proof of the Lemma. 
\end{proof}

\begin{corollary}\label{cor_prd}
	Let $\delta,\ve \in \lp 0,\infty \rp $, $q\in \lp 2,\infty \rp$, $A_1,A_2,A_3 \in \R^{1\times 2}$, $\Psi \in \N$ satisfy for all $x\in \R$ that $\delta = \ve \lp 2^{q-1} +1\rp^{-1}$, $A_1 = \lb 1 \quad 1 \rb$, $A_2 = \lb 1 \quad 0 \rb$, $A_3 = \lb 0 \quad 1 \rb$, $\real_{\rect} \in C\lp \R, \R \rp$, $\lp \real_{\rect} \lp \Psi \rp \rp \lp 0\rp = 0$, $0\les \lp \real_{\rect} \lp \Psi \rp \rp \lp x \rp \les \delta+|x|^2$, $|x^2-\lp \real_{\rect}\lp \Psi \rp \rp \lp x \rp |\les \delta \max \{1,|x|^q\}$, $\dep\lp \Psi \rp \les \max\{ 1+\frac{1}{q-2}+\frac{q}{2(q-2)}\log_2 \lp \delta^{-1} \rp ,2\}$, and $\param \lp \Psi \rp \les \max\left\{\lb \frac{40q}{q-2} \rb \log_2\lp \delta^{-1} \rp +\frac{80}{q-2}-28,52\right\}$, and finally let $\Gamma$ be defined as in Lemma \ref{prd_network}, i.e.:
	\begin{align}
		\Gamma = \lp \frac{1}{2}\circledast \lp \Psi \bullet \aff_{A_1,0} \rp \rp \bigoplus \lp \lp -\frac{1}{2}\rp \circledast\lp \Psi \bullet \aff_{A_2,0} \rp \rp \bigoplus\lp \lp -\frac{1}{2}\rp \circledast \lp \Psi \bullet \aff_{A_3,0} \rp \rp 
	\end{align}

	It is then the case for all $x,y \in \R$ that:
	\begin{align}
		\real_{\rect} \lp \Gamma \rp \lp x,y \rp \les \frac{3}{2} \lp \frac{\ve}{3} +x^2+y^2\rp \les \ve + 2x^2+2y^2
	\end{align}
\end{corollary}
\begin{proof}
	Note that the triangle inequality, the fact that $\delta = \ve \lp 2^{q-1} +1\rp^{-1}$, the fact that for all $x,y\in \R$ it is the case that $|x+y|^2 \les 2\lp |x|^2+|y|^2\rp $ and (\ref{6.0.33}) tell us that:
	\begin{align}
		\left| \real_{\rect} \lp \Gamma \rp\lp x,y\rp  \right| &\les \frac{1}{2}\left| \real_{\rect} \lp \Psi \rp\lp x+y \rp \right| + \frac{1}{2}\left| \real_{\rect} \lp \Psi \rp\lp x \rp \right| + \frac{1}{2}\left| \real_{\rect} \lp \Psi \rp\lp y \rp \right| \nonumber \\
		&\les \frac{1}{2} \lp \delta + |x+y|^2 \rp + \frac{1}{2}\lp \delta + |x|^2\rp + \frac{1}{2}\lp \delta + |y|^2\rp\nonumber \\
		&\les \frac{3\delta}{2} +\frac{3}{2}\lp |x|^2+|y|^2\rp = \lp \frac{3\ve}{2}\rp \lp 2^{q-1}+1\rp^{-1} + \frac{3}{2}\lp |x|^2+|y|^2\rp \nonumber\\
		&= \frac{3}{2}\lp \frac{\ve}{2^{q-1}+1} + |x|^2 + |y|^2 \rp \les \frac{3}{2} \lp \frac{\ve}{3}+|x|^2+|y|^2\rp \nonumber \\
		&\les \ve + 2x^2+2y^2
	\end{align}
\end{proof}
\begin{remark}
	We shall refer to this neural network for a given $q \in \lp 2,\infty \rp$ and given $\ve \in \lp 0,\infty \rp$ from now on as $\prd^{q,\ve}$. 
\end{remark}

\begin{lemma}\label{6.2.4}
	Let $x,y \in \R$, $\ve \in \lp 0,\infty \rp$ and $q \in \lp 2,\infty \rp$. It is then the case for all $x,y \in \R$ that:
	\begin{align}
		\ve \max \left\{ 1,|x|^q,|y|^q\right\} \les \ve + \ve |x|^q+\ve |y|^q.
	\end{align}
\end{lemma}
\begin{proof}
	We will do this in the following cases:
	
	For the case that $|x| \les 1$ and $|y| \les 1$ we then have:
	\begin{align}
		\ve \max \left\{ 1,|x|^q,|y|^q \right\} = \ve \les \ve + \ve |x|^q+\ve |y|^q
	\end{align}
	For the case that $|x| \les 1$ and $|y| \ges 1$, without loss of generality we have then:
	\begin{align}
		\ve \max \left\{1,|x|^q,|y|^q \right\} \les \ve | y|^q \les \ve + \ve |x|^q+\ve |y|^q:
	\end{align}
	For the case that $|x| \ges 1$ and $|y| \ges 1$, and without loss of generality that $|x| \ges |y|$  we have that:
	\begin{align}
		\ve \max\{ 1, |x|^q,|y|^q \} = \ve |x|^q \les \ve + \ve |x|^q+\ve |y|^q
	\end{align}
\end{proof}
\begin{lemma}\label{mathfrak_p}
Let $\mathfrak{p}_i$ for $i \in \{1,2,...\}$ be the set of functions defined for $\ve \in \lp 0,\infty\rp$, and $x \in \R$ as follows:
	\begin{align}
		\mathfrak{p}_1 &= \ve+2+2|x|^2 \nonumber\\
		\mathfrak{p}_i &= \ve +2\lp \mathfrak{p}_{i-1} \rp^2+2|x|^2 \text{ for } i \ges 2
	\end{align}
	 For all $n\in \N$ and $\ve \in (0,\infty)$ and $q\in (2,\infty)$ it holds for all $x\in \R$ that:
	\begin{align}
		\left| \real_{\rect} \lp  \pwr^{q,\ve}_n \rp \lp x \rp\right| \les \mathfrak{p}_n
	\end{align}
	\end{lemma}
	\begin{proof}
	Note that by Corollary \ref{cor_prd}, it is the case that:
	\begin{align}\label{(6.2.31)}
		\left|\real_{\rect} \lp \pwr^{q,\ve}_1 \rp \lp x \rp \right| =\left| \real_{\rect}\lp  \prd^{q,\ve}\rp \lp1,x \rp \right| \les \mathfrak{p}_1
	\end{align}
	and applying (\ref{(6.2.31)}) twice, it is the case that:
	\begin{align}
		\left| \real_{\rect} \lp \pwr_2^{q,\ve}\rp \lp x \rp \right| &= \left| \real_{\rect} \lp \prd^{q,\ve} \rp \lp \real_{\rect} \lp \prd ^{q,\ve}\lp 1,x \rp\rp,x\rp \right| \nonumber \\
		&\les \ve + 2\left| \real_{\rect} \lp \prd^{q,\ve}\rp\lp 1,x\rp \right|^2 + 2|x|^2 \nonumber \\
		&\les \ve + 2\mathfrak{p}_1^2 +2|x|^2 = \mathfrak{p}_2
	\end{align}
	Let's assume this holds for all cases up to and including $n$. For the inductive step, Item (ii) of Proposition 3.5 in \cite{grohs2019spacetime} tells us that:
	\begin{align}
		\left| \real_{\rect} \lp \pwr_{n+1}^{q,\ve}\rp \lp x\rp \right| &\les \left| \real_{\rect} \lp \prd^{q,\ve} \lp \real_{\rect} \lp \prd^{q,\ve} \lp \real_{\rect}\cdots \lp 1,x\rp,x \rp ,x\rp \cdots \rp \rp \right| \nonumber \\
		&\les \real_{\rect} \lb \prd^{q,\ve} \lp \pwr^{q,\ve}_n \lp x\rp,x \rp\rb \nonumber \\
		&\les \ve + 2\mathfrak{p}_n^2 + 2|x|^2 = \mathfrak{p}_{n+1}
	\end{align}
	This completes the proof of the lemma.
	\end{proof}
	\begin{remark}
		Note that since any instance of $\mathfrak{p}_i$ contains an instance of $\mathfrak{p}_{i-1}$ for $i \in \N \cap \lb 2,\infty\rp$, we have that $\mathfrak{p}_n \in \mathcal{O}\lp \ve^{2(n-1)}\rp$
	\end{remark}
	\begin{lemma}\label{param_pwr_geq_param_tun}
		For all $n \in \N$, $q\in \lp 2,\infty\rp$, and $\ve \in \lp 0,\infty\rp$, it is the case that $\param \lp \tun_{\dep\lp\pwr^{q,\ve}_n\rp}\rp \les \param \lp \pwr^{q,\ve}_n\rp$.
	\end{lemma}
	\begin{proof}
		Note that for all $n \in \N$ it is straightforwardly the case that $\param\lp \pwr_n^{q,\ve}\rp \ges \param \lp \tun_{\dep\lp \pwr^{q,\ve}_{n-1}\rp}\rp$ because for all $n\in \N$, a $\pwr^{q,\ve}_n$ network contains a $\tun_{\dep\lp \pwr^{q,\ve}_{n-1}\rp}$ network. Note now that for all $i \in \N$ we have from Lemma \ref{tun_prop} that $5 \les \param\lp \tun_{i+1}\rp - \param\lp \tun_i\rp \les 6$. Recall from Corollary \ref{cor:phi_network} that every  instance of the $\Phi$ network contains atleast one $\mathfrak{i}_4$ network, which by Lemma \ref{lem:mathfrak_i} has $40$ parameters, whence the $\prd^{q,\ve}$ network has atleast $40$ parameters for all $\ve \in \lp 0,\infty \rp$ and $q \in \lp 2,\infty\rp$. Note now that for all $i\in \N$, $\pwr^{q,\ve}_{i}$ and $\pwr^{q,\ve}_{i+1}$ differ by atleast as many parameters as there are in $\prd^{q,\ve}$, since, indeed, they differ by atleast one more $\prd^{q,\ve}$. Thus for every increment in $i$, $\pwr_i^{q,\ve}$ outstrips $\tun_i$ by at-least $40-6 = 34$ parameters. This is true for all $i\in \N$. Whence it is the case that for all $i \in \N$, it is the case that $\param\lp \tun_i\rp \les \param \lp \pwr^{q,\ve}_i\rp$.
	\end{proof}
\begin{lemma}\label{power_prop}
	Let $\delta,\ve \in \lp 0,\infty \rp $, $q\in \lp 2,\infty \rp$, and $\delta = \ve \lp 2^{q-1} +1\rp^{-1}$. Let $n \in \N_0$, and $\pwr_n \in \neu$. It is then the case for all $n \in \N_0$, and $x \in \R$ that:
	\begin{enumerate}
		\item $\lp \real_{\rect} \lp \pwr_n^{q,\ve} \rp \rp \lp x \rp \in C \lp \R, \R \rp $
		\item $\dep(\pwr_n^{q,\ve}) \les \begin{cases}
			1 & :n=0\\
			n\lb \frac{q}{q-2} \lb \log_2 \lp \ve^{-1} \rp +q\rb -1 \rb +1 & :n \in \N
		\end{cases}$
		\item $\wid_1 \lp \pwr^{q,\ve}_{n}\rp = \begin{cases}
			1 & :n=0 \\
			24+2\lp n-1 \rp & :n \in \N 
		\end{cases}$
		\item $\param(\pwr_n^{q,\ve}) \les \begin{cases}
			2 & :n=0 \\
			4^{n+\frac{3}{2}} + \lp \frac{4^{n+1}-1}{3}\rp \lp \frac{360q}{q-2} \lb \log_2 \lp \ve^{-1} \rp +q+1 \rb +372\rp &: n\in \N		
			\end{cases}$\\~\\
		\item $\left|x^n -\lp \real_{\rect} \lp \pwr^{q,\ve}_n \rp \rp \lp x \rp \right| \les \begin{cases}
			0 & :n=0 \\
			\left| x \lp x^{n-1} - \real_{\rect}\lp \pwr^{q,\ve}_{n-1}\rp\lp x\rp\rp\right| + \ve + |x|^q + \mathfrak{p}_{n-1}^q & :n\in \N
		\end{cases}$ \\~\\
		Where we let $\mathfrak{p}_i$ for $i \in \{1,2,...\}$ be the set of functions defined as follows:
	\begin{align}
		\mathfrak{p}_1 &= \ve+2+2|x|^2 \nonumber\\
		\mathfrak{p}_i &= \ve +2\lp \mathfrak{p}_{i-1} \rp^2+2|x|^2
	\end{align}
	And whence we get that:
	\begin{align}
		\left| x^{n} - \real_{\rect} \lp \pwr^{q,\ve}_n\rp \lp x\rp\right| \in \mathcal{O} \lp \ve^{2q\lp n-1\rp} \rp &\text{ for } n \ges 2
	\end{align}
	\item $\wid_{\hid \lp \pwr_n^{q,\ve}\rp}\lp \pwr^{q,\ve}_n\rp = \begin{cases}
		1 & n=0 \\
		24 & n \in \N
	\end{cases}$
	\end{enumerate}
\end{lemma}
\begin{proof} 
	Note that Item (ii) of Lemma 2.3.2 in \cite{bigbook} ensures that $\real_{\rect} \lp \pwr_0 \rp = \aff_{1,0} \in C \lp \R, \R \rp$. Note next that by Item (v) of Proposition 2.6 in \cite{grohs2019spacetime}, with $\Phi_1 \curvearrowleft \nu_1, \Phi_2 \curvearrowleft \nu_2, a \curvearrowleft \rect$, we have that:
	\begin{align}
		\lp \real_{\rect} \lp \nu_1 \bullet \nu_2 \rp\rp \lp x \rp = \lp\lp \real_{\rect}\lp \nu_1 \rp \rp \circ \lp \real_{\rect}\lp \nu_2 \rp \rp \rp  \lp x \rp
	\end{align}
	This, with the fact that the composition of continuous functions is continuous, the fact the stacking of continuous instantiated neural networks is continuous tells us that $\lp \real_{\rect} \pwr_n \rp \in C \lp \R, \R \rp$ for $n \in \N \cap \lb 2,\infty \rp$. This establishes Item (i).
	
	Note next that by observation $\dep \lp \pwr_0^{q,\ve} \rp=1$ and by Lemma 2.2.7 in \cite{bigbook}, it is the case that $\dep\lp \id_1 \rp  = 2$. By Lemmas 2.4.2 in \cite{bigbook} and Proposition 2.6 in \cite{grohs2019spacetime} it is also the case that: $\dep\lp \prd^{q,\ve} \bullet  \lb  \pwr^{q,\ve}_{n-1} \boxminus \tun_{\dep(\pwr^{q,\ve}_{n-1})} \rb \bullet \cpy \rp = \dep \lp \prd^{q,\ve} \bullet \lb  \pwr^{q,\ve}_{n-1} \boxminus \tun_{\dep(\pwr^{q,\ve}_{n-1})}\rb\rp $. Note also that by Lemma 2.2.2 in \cite{bigbook}, and by Definition \ref{def:stk}, we have that $\dep \lp \pwr^{q,\ve}_{n-1} \boxminus \tun_{\dep \lp \pwr^{q,\ve}_{n-1}\rp}\rp = \dep \lp \pwr^{q,\ve}_{n-1} \rp$.
	This with Proposition 2.6 in \cite{grohs2019spacetime}, and Lemma \ref{prd_network}, then yields for $n \in \N$ that:
	\begin{align}
		\dep \lp \pwr^{q,\ve}_n  \rp &= \dep \lp \prd \bullet \lb \tun_{\mathcal{D} \lp \pwr^{q,\ve}_{n-1} \rp } \boxminus \pwr^{q,\ve}_{n-1} \rb \bullet \cpy_{2,1} \rp \nonumber \\
		&= \dep \lp \prd^{q,\ve} \bullet \lb \tun_{\dep \lp \pwr^{q,\ve}_{n-1} \rp } \boxminus \pwr^{q,\ve}_{n-1} \rb \rp \nonumber \\
		&= \dep \lp \prd^{q,\ve} \rp + \dep \lp \tun_{\dep \lp \pwr^{q,\ve}_{n-1} \rp} \rp -1 \nonumber \\
		&\les \frac{q}{q-2} \lb \log_2 \lp \ve^{-1}\rp +q \rb + \dep \lp \tun_{\dep\lp \pwr^{q,\ve}_{n-1} \rp} \rp - 1 \nonumber \\ 
		&= \frac{q}{q-2}\lb \log_2 \lp\ve^{-1} \rp + q\rb + \dep \lp \pwr^{q,\ve}_{n-1}\rp - 1 
	\end{align}
	And hence for all $n \in \N$ it is the case that:
	\begin{align}
		\dep\lp \pwr^{q,\ve}_n\rp - \dep \lp \pwr^{q,\ve}_{n-1}\rp \les \frac{q}{q-2} \lb \log_2 \lp \ve^{-1} \rp +q\rb -1
	\end{align}
	This, in turn, indicates that:
	\begin{align}
		\dep \lp \pwr^{q,\ve}_n\rp &\les n\lb \frac{q}{q-2} \lb \log_2 \lp \ve^{-1} \rp +q\rb -1 \rb +1 \nonumber \\
		&\les n\lb \frac{q}{q-2} \lb \log_2 \lp \ve^{-1} \rp +q\rb -1 \rb +1
	\end{align}
	This proves Item (ii).
	
	Note now that $\wid_1 \lp \pwr^{q,\ve}_0\rp = \wid_1 \lp \aff_{0,1}\rp = 1$. Further Proposition 2.6 in \cite{grohs2019spacetime}, Lemma \ref{tun_prop}, tells us that for all $i,k \in \N$ it is the case that $\wid_i \lp \tun_k\rp \les 2$. Observe that since $\cpy_{2,1}, \pwr_0^{q,\ve}$, and $\tun_{\dep \lp \pwr_0^{q,\ve}\rp}$ are all affine neural networks, Lemma 2.3.3 in \cite{bigbook}, Corollary 2.9 in \cite{grohs2019spacetime}, and Lemma \ref{prd_network} tells us that: 	
	\begin{align}
		\wid_1 \lp \pwr_1^{q,\ve} \rp &= \wid_1 \lp \prd^{q,\ve} \bullet \lb \tun_{\dep(\pwr_{0}^{q,\ve})} \boxminus \pwr_{0}^{q,\ve} \rb \bullet \cpy_{2,1} \rp \nonumber \\
		&= \wid_1 \lp \prd^{q,\ve}\rp = 24
	\end{align}
	And that:
	\begin{align}
		\wid_1 \lp \pwr_2^{q,\ve} \rp &= \wid_1 \lp \prd^{q,\ve} \bullet \lb \tun_{\dep(\pwr_{1}^{q,\ve})} \boxminus \pwr_{1}^{q,\ve} \rb \bullet \cpy_{2,1} \rp \nonumber \\
		&= \wid_1 \lp \lb \tun_{\dep \lp \pwr^{q,\ve}_1 \rp} \boxminus \pwr_{1}^{q,\ve} \rb \rp \nonumber\\
		&= 24+2 = 26 \nonumber
	\end{align} 
	This completes the base case. For the inductive case, assume that for all $i$ up to and including $k\in \N$ it is the case that $\wid_1 \lp \pwr_i^{q,\ve}\rp \les  \begin{cases}
		1 & :i=0 \\
		24+2(i-1) & :i \in \N
	\end{cases}$. For the case of $k+1$, we get that:
	\begin{align}
		\wid_1 \lp \pwr_{k+1}^{q,\ve} \rp &= \wid_1 \lp \prd^{q,\ve} \bullet \lb \tun_{\dep(\pwr_{k}^{q,\ve})} \boxminus \pwr_{k}^{q,\ve} \rb \bullet \cpy_{2,1} \rp \nonumber \\
		&=\wid_1 \lp \lb \tun_{\dep(\pwr_{k}^{q,\ve})} \boxminus \pwr_{k}^{q,\ve} \rb  \rp \nonumber \\
		&=\wid_1 \lp \tun_{\dep \lp \pwr^{q,\ve}_{k}\rp}\rp + \wid_1 \lp \pwr^{q,\ve}_k\rp \nonumber \\
		&\les \begin{cases}
			2 & :k=0 \\
			24 +2 k & :k\in \N 
		\end{cases}
	\end{align}
	This establishes Item (iii).
	
	For Item (iv), we will prove this in cases.
	
	\textbf{Case 1: $\pwr_n^{q,\ve}$ where $n=0$:}
	
	Note that by Lemma 2.3.2 in \cite{bigbook} and  Definition \ref{def:pwr} we have that:
	\begin{align}
		\param\lp \pwr_0^{q,\ve} \rp = \param \lp \aff_{0,1} \rp =2
	\end{align}
	This completes Case 1.
	
	\textbf{Case 2: $\pwr_n^{q,\ve}$ where $n\in \N$:}
	
	Note that Proposition 2.20 in \cite{grohs2019spacetime}, Lemma \ref{param_pwr_geq_param_tun}, Corollary 2.21 in \cite{grohs2019spacetime}, and Definition \ref{def:stk} tells us it is the case that: 
	\begin{align}
		\param \lp \pwr_{n-1}^{q,\ve} \boxminus \tun_{\dep \lp \pwr_{n-1}^{q,\ve}\rp }\rp &\les \param \lp \pwr^{q,\ve}_{n-1} \boxminus \pwr^{q,\ve}_{n-1}\rp \nonumber\\
		&\les 4\param\lp \pwr^{q,\ve}_{n-1}\rp	
	\end{align}
	Then Proposition 2.6 in \cite{grohs2019spacetime}, Lemma \ref{param_pwr_geq_param_tun}, Corollary 2.21 in \cite{grohs2019spacetime}, and Corollary 2.9 in \cite{grohs2019spacetime} tells us that: 
	\begin{align}\label{(6.2.34)}
		&\param \lp  \lb \pwr^{q,\ve}_{n-1} \boxminus\tun_{\dep \lp \pwr_{n-1}^{q,\ve} \rp}\rb \bullet \cpy_{2,1}\rp \nonumber\\&= \param \lp  \lb \pwr^{q,\ve}_{n-1} \boxminus\tun_{\dep \lp \pwr_{n-1}^{q,\ve} \rp}\rb \rp \nonumber\\
		&\les 4\param \lp \pwr^{q,\ve}_{n-1}\rp
	\end{align}
	Note next that by definition for all $q\in \lp 2,\infty\rp$, and $\ve \in \lp 0,\infty\rp$ it is case that $\wid_{\hid\lp \pwr_0^{q,\ve}\rp}\pwr_0^{q,\ve} = \wid_{\hid \lp \aff_{0,1}\rp} = 1$. Now, by Lemma \ref{prd_network}, and by construction of $\pwr_i^{q,\ve}$ we may say that for $i\in \N$ it is the case that:
	\begin{align}
		\wid_{\hid \lp \pwr^{q,\ve}_i\rp} = \wid _{\hid \lp \prd^{q,\ve}\rp} = 24
	\end{align}

	Note also that by Lemma \ref{6.2.2} it is the case that:
	\begin{align}
		\wid_{\hid \lp \tun_{\dep \lp \pwr_{i-1}^{q,\ve}\rp}\rp} \lp \tun_{\dep \lp \pwr^{q,\ve}_{i-1}\rp} \rp = 2	\end{align}
	Furthermore, note that for $n\in \lb 2, \infty \rp \cap \N$, Lemma \ref{prd_network}, and Lemma \ref{tun_prop} tells us that:
	\begin{align}
		 \wid_{\hid \lp \lb \pwr^{q,\ve}_{n-1} \boxminus\tun_{\dep \lp \pwr_{n-1}^{q,\ve} \rp}\rb\rp} \lp \lb \pwr^{q,\ve}_{n-1} \boxminus\tun_{\dep \lp \pwr_{n-1}^{q,\ve} \rp}\rb\rp = 24+2=26
	\end{align}
	
	Finally Proposition 2.6 in \cite{grohs2019spacetime}, (\ref{(6.2.34)}), and Corollary 2.21 in \cite{grohs2019spacetime}, also tells us that:
	\begin{align}
		&\param \lp \pwr_{n}^{q,\ve}\rp\\ &= \param \lp  \prd^{q,\ve} \bullet\lb \pwr^{q,\ve}_{n-1} \boxminus\tun_{\dep \lp \pwr_{n-1}^{q,\ve} \rp}\rb \bullet \cpy_{2,1}\rp \nonumber \\
		&= \param \lp \prd^{q,\ve} \bullet \lb \pwr^{q,\ve}_{n-1} \boxminus\tun_{\dep \lp \pwr_{n-1}^{q,\ve} \rp}\rb\rp \nonumber \\
		&\les \param \lp \prd^{q,\ve} \rp +  4\param \lp \pwr_{n-1}^{q,\ve}\rp+\nonumber\\
		&+ \wid_1 \lp \prd^{q,\ve} \rp\ \cdot \wid_{\hid \lp \lb \pwr^{q,\ve}_{n-1} \boxminus\tun_{\dep \lp \pwr_{n-1}^{q,\ve} \rp}\rb\rp} \lp \lb \pwr^{q,\ve}_{n-1} \boxminus\tun_{\dep \lp \pwr_{n-1}^{q,\ve} \rp}\rb\rp \nonumber \\
		&= \param\lp \prd^{q,\ve}\rp + 4\param\lp \pwr^{q,\ve}_{n-1}\rp + 624 \nonumber\\
		&= 4^{n+1}\param\lp \pwr^{q,\ve}_0\rp + \lp \frac{4^{n+1}-1}{3}\rp \lp \param\lp \prd^{q,\ve}\rp + 624\rp \nonumber\\
		&= 4^{n+\frac{3}{2}} + \lp \frac{4^{n+1}-1}{3}\rp \lp \frac{360q}{q-2} \lb \log_2 \lp \ve^{-1} \rp +q+1 \rb +372\rp
	\end{align}

	Next note that $\lp \real_{\rect} \lp \pwr_{0,1} \rp\rp \lp x \rp$  is exactly $1$, which implies that for all $x\in \R$ we have that $|x^0-\lp \real_{\rect} \lp \pwr_{0.1}\rp\lp x \rp\rp |=0$. Note also that the instantiation with $\rect$ of $\tun_n$ and $\cpy_{2,1}$ are exact. Note next that since $\tun_n$ and $\cpy_{2,1}$ are exact, the only sources of error for $\pwr^{q,\ve}_n$ a are $n$ compounding applications of $\prd^{q,\ve}$.
	
	Note also that by definition, it is the case that:
	\begin{align}
		\real_{\rect}\lp \pwr_n^{q,\ve} \rp = \real_{\rect} \lb \underbrace{\prd^{q,\ve}  \lp \inst_{\rect} \lb \prd^{q,\ve}\lp\cdots \inst_{\rect}\lb \prd^{q,\ve} \lp  1,x\rp \rb, \cdots x\rp \rb, x \rp}_{n-copies } 	\rb 
	\end{align}
	Lemma \ref{prd_network}, tells us that:
	\begin{align}
		\left| x-\inst_{\rect}\lp \pwr_1\lp x\rp\rp\right|=\left|x-\real_{\rect}\lp \prd^{q,\ve} \lp 1,x \rp \rp \right| \les \ve \max\{ 1,|x|^q\} \les \ve + \left| x\right|^q
	\end{align}
	The triangle inequality, Lemma \ref{6.2.4}, Lemma \ref{prd_network}, and Corollary \ref{cor_prd}, then tells us that:
	\begin{align}
		&\left| x^2 - \real_{\rect} \lp \pwr^{q,\ve}_2 \rp \lp x \rp \right| \nonumber\\
		&=\left| x\cdot x-\real_{\rect}\lp \prd^{q,\ve}\lp \inst_{\rect}\lp \prd^{q,\ve} \lp 1,x \rp \rp,x\rp \rp\right| \nonumber\\
		&\les \left| x\cdot x - x \cdot \inst_{\rect} \lp \prd^{q,\ve}\lp 1,x\rp \rp \right| + \left| x\cdot \inst_{\rect}\lp \prd^{q,\ve} \lp 1,x \rp\rp -\inst_{\rect}\lp \prd^{q,\ve} \lp \inst_{\rect}\lp \prd^{q,\ve}\lp 1,x\rp \rp,x \rp \rp \right| \nonumber\\
		&=\left| x\lp x-\inst_{\rect}\lp \prd^{q,\ve}\lp 1,x\rp\rp\rp\right|+ \ve + \ve\left| x\right|^q+\ve \left| \inst_{\rect}\lp \prd^{q,\ve}\lp 1,x\rp\rp\right|^q \nonumber\\
		&\les \left|x\ve + x\ve\left|x\right|^q \right| + \ve + \ve\left|x\right|^q+\ve \left|\ve + 1+x^2 \right|^q \nonumber\\
		&= \left| x\ve + x\ve \left| x\right|^q\right| + \ve + \ve\left| x\right|^q + \ve \mathfrak{p}_{1}^q
	\end{align}
	
	Note that this takes care of our base case. Assume now that for all integers up to and including $n$, it is the case that:
	\begin{align}\label{(6.2.39)}
		\left| x^n - \real_{\rect}\lp \pwr_n^{q,\ve}\rp \lp x \rp \right| &\les \left| x\cdot x^{n-1}-x \cdot \real_{\rect}\lp \pwr_{n-1}^{q,\ve}\rp \lp x\rp\right| + \left| x \cdot \real_{\rect}\lp \pwr_{n-1}^{q,\ve}\rp \lp x\rp -\real_{\rect} \lp \pwr_n^{q,\ve} \rp  \lp x \rp \right| \nonumber \\
		&\les \left| x\lp x^{n-1}-\real_{\rect} \lp \pwr^{q,\ve}_{n-1}\rp \lp x\rp\rp\right| + \ve + \ve|x|^q + \ve\left| \real_{\rect}\lp \pwr^{q,\ve}_{n-1}\rp \lp x \rp \right| ^q\nonumber \\
		&\les \left| x \lp x^{n-1} - \real_{\rect}\lp \pwr^{q,\ve}_{n-1}\rp\lp x\rp\rp\right| + \ve + \ve|x|^q + \ve\mathfrak{p}_{n-1}^q
	\end{align}
	For the inductive case, Lemma \ref{mathfrak_p}, and Corollary \ref{cor_prd} allows us to see that:
	\begin{align}
		\left|x^{n+1}-\real_{\rect}\lp \pwr_{n+1}^{q,\ve}\rp\lp x\rp \right| &\les \left| x^{n+1}-x\cdot \real_{\rect}\lp \pwr_{n}^{q,\ve}\rp \lp x \rp\right| + \left| x\cdot \real_{\rect}\lp \pwr^{q,\ve}_n\rp \lp x \rp  - \real_{\rect} \lp \pwr^{q,\ve}_{n+1}\rp\right| \nonumber \\
		&\les \left|x\lp x^n-\real_{\rect} \lp \pwr^{q,\ve}_n\rp \lp x\rp\rp \right| + \ve + \ve|x|^q+\ve\left| \real_{\rect} \lp \pwr^{q,\ve}_{n}\rp \lp x \rp\right|^q \nonumber \\
		&\les \left|x\lp x^n-\real_{\rect} \lp \pwr^{q,\ve}_n\rp \lp x\rp\rp \right| + \ve + \ve|x|^q + \ve\mathfrak{p}^q_n
	\end{align}
	Note that since $\mathfrak{p}_n \in \mathcal{O} \lp \ve^{2(n-1)}\rp$ for $n\in \N \cap \lb 2,\infty \rp$, it is the case for all $x\in \R$ then that $\left| x^{n} - \real_{\rect} \lp \pwr^{q,\ve}_n\rp \lp x\rp\right| \in \mathcal{O} \lp \ve^{2q(n-1)} \rp$ for $n \ges 2$.
	
	Finally note that $\wid_{\hid \lp \pwr^{q,\ve}_0\rp}\lp \pwr^{q,\ve}_0\rp = 1$ from observation. For $n\in \N$, note that the second to last layer is the second to last layer of the $\prd^{q,\ve}$ network. Thus Lemma \ref{prd_network} tells us that:
	\begin{align}
		\wid_{\hid\lp \pwr^{q,\ve}_m\rp} \lp \pwr^{q,\ve}_n\rp = \begin{cases}
			1 & n=0 \\
			24 & n\in \N 
		\end{cases}
	\end{align}
	This completes the proof of the lemma.
\end{proof}
\begin{remark}\label{rem:pwr_gets_deeper}
	Note each power network $\pwr_n^{q,\ve}$ is at least as big as the previous power network $\pwr_{n-1}^{q,\ve}$, one differs from the next by one $\prd^{q, ve}$ network.
\end{remark}

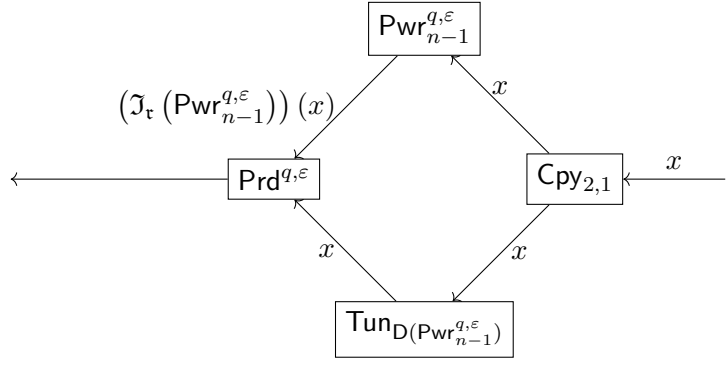
\begin{figure}[h]
\begin{center}
	\begin{tikzpicture}
  \node[draw, rectangle] (top) at (0, 2) {$\pwr_{n-1}^{q,\ve}$};
  \node[draw, rectangle] (right) at (2, 0) {$\cpy_{2,1}$};
  \node[draw, rectangle] (bottom) at (0, -2) {$\tun_{\dep(\pwr_{n-1}^{q,\ve})}$};
  \node[draw, rectangle] (left) at (-2, 0) {$\prd^{q,\ve} $};

  \draw[->] (right) -- node[midway, above] {$x$} (top);
  \draw[<-] (right) -- node[midway, above] {$x$} (4,0)(right);
  \draw[->] (right) -- node[midway, right] {$x$} (bottom);
  \draw[->] (top) -- node[midway, left] {$\lp \real_{\rect}\lp \pwr^{q,\ve}_{n-1}\rp \rp \lp x \rp $} (left);
  \draw[->] (bottom) -- node[midway, left] {$x$} (left);
  \draw[->] (left) -- node[midway, above] {} (-5.5,0);
\end{tikzpicture}
\end{center}
\caption{A representation of a typical $\pwr^{q,\ve}_n$ network.}
\end{figure}
\begin{remark}
	
\end{remark}

\subsection{Neural Network Polynomials}

\begin{lemma}\label{6.2.9}\label{nn_poly}\label{pnm_prop}
	Let $\delta,\ve \in \lp 0,\infty \rp $, $q\in \lp 2,\infty \rp$ and $\delta = \ve \lp 2^{q-1} +1\rp^{-1}$. It is then the case for all $n\in\N_0$ and $x\in \R$ that:
	\begin{enumerate}
		\item $\real_{\rect} \lp \pnm_{n,C}^{q,\ve}\rp \in C \lp \R, \R \rp $
		\item $\dep \lp \pnm_{n,C}^{q,\ve} \rp \les \begin{cases}
			1 & :n=0\\
			n\lb \frac{q}{q-2} \lb \log_2 \lp \ve^{-1} \rp +q\rb -1 \rb +1 &:n\in \N
		\end{cases}$
		\item $\param \lp \pnm_{n,C}^{q,\ve} \rp \les \begin{cases}
			2 & :n =0 \\
			\lp n+1\rp\lb 4^{n+\frac{3}{2}} + \lp \frac{4^{n+1}-1}{3}\rp \lp \frac{360q}{q-2} \lb \log_2 \lp \ve^{-1} \rp +q+1 \rb +372\rp\rb &:n\in \N
		\end{cases}$ \\~\\
		\item $\left|\sum^n_{i=0} c_ix^i - \real_{\rect} \lp \pnm_{n,C}^{q,\ve} \rp \lp x \rp \right| \les \sum^n_{i=1} c_i\lp \left| x \lp x^{i-1} - \real_{\rect}\lp \pwr^{q,\ve}_{i-1}\rp\lp x\rp\rp\right| + \ve + |x|^q + \mathfrak{p}_{i-1}^q	 \rp  $\\~\\
		Where $\mathfrak{p}_i$ are the set of functions defined for $i \in \N$ as such:
		\begin{align}
		\mathfrak{p}_1 &= \ve+2+2|x|^2 \nonumber\\
		\mathfrak{p}_i &= \ve +2\lp \mathfrak{p}_{i-1} \rp^2+2|x|^2
		\end{align}
		Whence it is the case that:
		\begin{align}
		\left|\sum^n_{i=0} c_ix^i - \real_{\rect} \lp \pnm_{n,C}^{q,\ve} \rp \lp x \rp \right| \in \mathcal{O} \lp \ve^{2q(n-1)}\rp
		\end{align}
		\item $\wid_1 \lp \pnm_{n,C}^{q,\ve} \rp =  2+23n+n^2 $
		\item $\wid_{\hid \lp \pnm_{n,C}^{q,\ve}\rp} \lp \pnm_{n,C}^{q,\ve}\rp \les\begin{cases}
			1 &:n=0 \\
			24 + 2n &:n\in \N \end{cases}$
	\end{enumerate}
\end{lemma}
\begin{proof}
	Note that by Lemma 2.4.11 in \cite{bigbook}, Lemma \ref{power_prop}, and Proposition 2.6 in \cite{grohs2019spacetime} for all $n\in \N_0$ it is the case that:
	\begin{align}
		\real_{\rect}\lp \pnm_{n,C}^{q,\ve} \rp &= \real_{\rect} \lp  \bigoplus^n_{i=0} \lb c_i \triangleright\lb \tun_{\max_i \left\{\dep \lp \pwr_i^{q,\ve} \rp\right\} +1 - \dep \lp \pwr^{q,\ve}_i\rp} \bullet \pwr_i^{q,\ve}\rb \rb \rp \nonumber\\
		&= \sum^n_{i=1}c_i \real_{\rect}\lp \tun_{\max_i \left\{\dep \lp \pwr_i^{q,\ve} \rp\right\} +1 - \dep \lp \pwr^{q,\ve}_i\rp} \bullet \pwr_i^{q,\ve} \rp \nonumber\\
		&= \sum^n_{i=1}c_i\real_{\rect}\lp \pwr^{q,\ve}_i \rp\nonumber
	\end{align}
	Since Lemma \ref{power_prop} tells us that $\lp \real_{\rect} \lp \pwr_n^{q,\ve} \rp \rp \lp x \rp \in C \lp \R, \R \rp$, for all $n\in \N_0$ and since the finite sum of continuous functions is continuous, this proves Item (i).
	
	Note that $\pnm_n^{q,\ve}$ is only as deep as the deepest of the $\pwr^{q,\ve}_i$ networks, which from the definition is $\pwr_n^{q,\ve}$, which in turn also has the largest bound. Therefore,  by Proposition 2.6 in \cite{grohs2019spacetime}, Definition \ref{def:stk}, and Lemma \ref{power_prop}, we have that:
	\begin{align}
		\dep \lp \pnm_{n,C}^{q,\ve} \rp &\les \dep \lp \pwr_n^{q,\ve}\rp \nonumber\\
		&\les \begin{cases}
			1 & :n=0\\
			n\lb \frac{q}{q-2} \lb \log_2 \lp \ve^{-1} \rp +q\rb -1 \rb +1 &:n\in \N
		\end{cases} \nonumber
	\end{align}
	This proves Item (ii).
	
	Note next that for the case of $n=0$, we have that:
	\begin{align}
		\pnm_n^{q,\ve} = c_i \triangleright\pwr_0^{q,\ve}
	\end{align}
	This then yields us $2$ parameters.
	
	Note that each neural network summand in $\pnm_n^{q,\ve}$ consists of a combination of $\tun_k$ and $\pwr_k$ for some $k\in \N$. Each $\pwr_k$ has at least as many parameters as a tunneling neural network of that depth, as Lemma \ref{param_pwr_geq_param_tun} tells us. This, finally, with Lemma 2.3.3 in \cite{bigbook}, Corollary 2.9 in \cite{grohs2019spacetime}, and Lemma \ref{power_prop} then implies that: 
	\begin{align}
		\param\lp \pnm^{q,\ve}_{n,C} \rp &= \param  \lp  \bigoplus^n_{i=0} \lb c_i \triangleright\lb \tun_{\max_i \left\{\dep \lp \pwr_i^{q,\ve} \rp\right\} +1 - \dep \lp \pwr^{q,\ve}_i\rp} \bullet \pwr_i^{q,\ve}\rb \rb \rp\nonumber \\
		&\les \lp n+1 \rp \cdot \param \lp c_i \triangleright \lb \tun_1 \bullet \pwr_n^{q,\ve} \rb\rp \nonumber\\
		&\les \lp n+1 \rp \cdot \param \lp \pwr_n^{q,\ve} \rp \nonumber \\
		&\les \begin{cases}
			2 & :n =0 \\
			\lp n+1\rp\lb 4^{n+\frac{3}{2}} + \lp \frac{4^{n+1}-1}{3}\rp \lp \frac{360q}{q-2} \lb \log_2 \lp \ve^{-1} \rp +q+1 \rb +372\rp\rb &:n\in \N
		\end{cases} \nonumber
	\end{align}
	This proves Item (iii).
	
	Finally, note that for all $i\in \N$, Lemma \ref{power_prop}, and the triangle inequality then tells us that it is the case for all $i \in \N$ that:
	\begin{align}
		\left| x^i - \real_{\rect}\lp \pwr_i^{q,\ve}\rp \lp x \rp \right| &\les \left| x^i-x \cdot \real_{\rect}\lp \pwr_{i-1}^{q,\ve}\rp \lp x\rp\right| + \left| x \cdot \real_{\rect}\lp \pwr_{i-1}^{q,\ve}\rp \lp x\rp -\real_{\rect} \lp \pwr_i^{q,\ve} \rp  \lp x \rp \right| \nonumber \\
	\end{align}	
	This, Lemma \ref{6.2.9}, and the fact that instantiation of the tunneling neural network leads to the identity function (Lemma 2.3.5 in \cite{bigbook} and Proposition 2.6 in \cite{grohs2019spacetime}), together with Lemma \ref{scalar_left_mult_distribution}, and the absolute homogeneity condition of norms, then tells us that for all $x\in \R$, and $c_0,c_1,\hdots, c_n \in \R$ it is the case that:
	\begin{align}
		&\left|\sum^n_{i=0} c_ix^i - \real_{\rect} \lp \pnm^{q,\ve}_{n,C} \lp x\rp \rp \right| \nonumber\\
		&= \left| \sum^n_{i=0}  c_ix^i - \real_{\rect} \lb \bigoplus^n_{i=0} \lb c_i \triangleright \tun_{\max_i \left\{\dep \lp \pwr_i^{q,\ve} \rp\right\} +1 - \dep \lp \pwr^{q,\ve}_i\rp} \bullet \pwr_i^{q,\ve} \rb \rb\lp x \rp\right| \nonumber \\
		&=\left| \sum^n_{i=1} c_ix^i-\sum_{i=0}^n c_i \lp \inst_{\rect}\lb \tun_{\max_i \left\{\dep \lp \pwr_i^{q,\ve} \rp\right\} +1 - \dep \lp \pwr^{q,\ve}_i\rp} \bullet \pwr_i^{q,\ve}\rb\lp x\rp\rp\right| \nonumber\\
		&\les \sum_{i=1}^n \left|c_i\right| \cdot\left| x^i - \inst_{\rect}\lb \tun_{\max_i \left\{\dep \lp \pwr_i^{q,\ve} \rp\right\} +1 - \dep \lp \pwr^{q,\ve}_i\rp} \bullet \pwr_i^{q,\ve}\rb\lp x\rp\right| \nonumber\\
		&\les \sum^n_{i=1} \left|c_i\right|\cdot\lp \left| x \lp x^{i-1} - \real_{\rect}\lp \pwr^{q,\ve}_{i-1}\rp\lp x\rp\rp\right| + \ve + 2|x|^q + 2\mathfrak{p}_{i-1}^q	 \rp \nonumber
	\end{align}
	Note however that since for all $x\in \R$ and $i \in \N \cap \lb 2, \infty\rp$, Lemma \ref{prd_network} tells us that $\left| x^{i} - \real_{\rect} \lp \pwr^{q,\ve}_i\rp \lp x\rp\right| \in \mathcal{O} \lp \ve^{2q\lp i-1\rp} \rp$, this, and the fact that $f+g \in \mathcal{O}\lp x^a \rp$ if $f \in \mathcal{O}\lp x^a\rp$, $g \in \mathcal{O}\lp x^b\rp$, and $a \ges b$, then implies that: 
	\begin{align}
		\sum^n_{i=1} \left| c_i\right|\cdot\lp \left| x \lp x^{i-1} - \real_{\rect}\lp \pwr^{q,\ve}_{i-1}\rp\lp x\rp\rp\right| + \ve + 2|x|^q + 2\mathfrak{p}_{i-1}^q	 \rp \in \mathcal{O} \lp \ve^{2q(n-1)}\rp
	\end{align}
	This proves Item (iv).
	
	Note next in our construction $\aff_{0,1}$ will require tunneling whenever $i\in \N$ in $\pwr_{i}^{q,\ve}$. Lemma \ref{aff_effect_on_layer_architecture} and Corollary \ref{affcor} then tell us that:
	\begin{align}
		\wid_1 \lp \pnm_n^{q,\ve} \rp &= \wid_1 \lp \bigoplus^n_{i=0} \lb c_i \triangleright\lb \tun_{\max_i \left\{\dep \lp \pwr_i^{q,\ve} \rp\right\} +1 - \dep \lp \pwr^{q,\ve}_i\rp} \bullet \pwr_i^{q,\ve}\rb \rb\rp \nonumber\\
		&= \wid_1 \lp \bigoplus^n_{i=0}\pwr^{q,\ve}_i\rp \nonumber \\
		&\les \sum^n_{i=0}\wid_1 \lp \pwr^{q,\ve}_i\rp =2 + \frac{n}{2}\lp 24+24+2\lp n-1\rp\rp = 2+23n+n^2 \nonumber \\
	\end{align}
	This proves Item (v).
	
	Finally note that from the definition of the $\pnm_{n,C}^{q,\ve}$, it is evident that $\wid_{\hid\lp \pwr_{0,C}^{q,\ve}\rp}\lp \pwr_{0,C}^{q,\ve}\rp = 1$ since $\pwr_{0,C}^{q,\ve} = \aff_{0,1}$. Other than this network, for all $i \in \N$, $\pwr_{i,C}^{q,\ve}$ end in the $\prd^{q,\ve}$ network, and the deepest of the $\pwr_i^{q,\ve}$ networks is $\pwr^{q,\ve}_n$ inside $\pnm_{n,C}^{q,\ve}$. All other $\pwr_i^{q,\ve}$ must end in tunnels. Whence in the second to last layer, Lemma \ref{prd_network} tells us that:	
		\begin{align}
			\wid_{\hid\lp \pnm_{n,C}^{q,\ve}\rp} \les \begin{cases}
				1 &: n =0 \\
				24+2n &:n \in \N
			\end{cases}
		\end{align}
		This completes the proof of the Lemma.
\end{proof}
\begin{remark}
	Diagrammatically, these can be represented as
\end{remark}
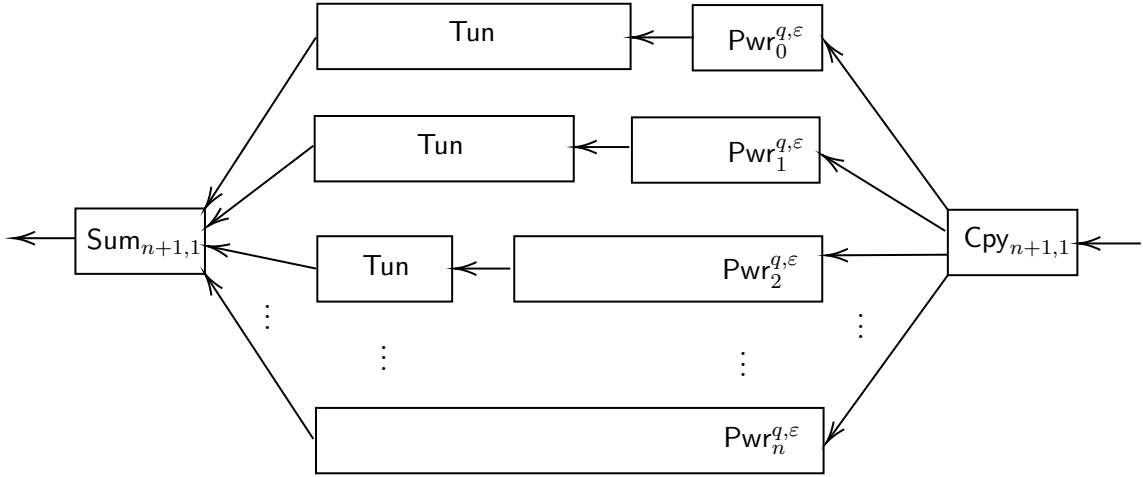
\begin{figure}[h]
\begin{center}

\tikzset{every picture/.style={line width=0.75pt}} 

\begin{tikzpicture}[x=0.75pt,y=0.75pt,yscale=-1,xscale=1]

\draw   (390,52) -- (455.33,52) -- (455.33,85) -- (390,85) -- cycle ;
\draw   (359.33,108.67) -- (454,108.67) -- (454,141.67) -- (359.33,141.67) -- cycle ;
\draw   (300,168.67) -- (455.33,168.67) -- (455.33,201.67) -- (300,201.67) -- cycle ;
\draw   (200,255.33) -- (456,255.33) -- (456,288.33) -- (200,288.33) -- cycle ;
\draw   (200.67,51.33) -- (358.67,51.33) -- (358.67,84.33) -- (200.67,84.33) -- cycle ;
\draw   (199.33,108) -- (330,108) -- (330,141) -- (199.33,141) -- cycle ;
\draw   (200.67,168.67) -- (268.67,168.67) -- (268.67,201.67) -- (200.67,201.67) -- cycle ;
\draw    (390.67,68.33) -- (361.33,68.33) ;
\draw [shift={(359.33,68.33)}, rotate = 360] [color={rgb, 255:red, 0; green, 0; blue, 0 }  ][line width=0.75]    (10.93,-3.29) .. controls (6.95,-1.4) and (3.31,-0.3) .. (0,0) .. controls (3.31,0.3) and (6.95,1.4) .. (10.93,3.29)   ;
\draw    (358,123.67) -- (332.67,123.67) ;
\draw [shift={(330.67,123.67)}, rotate = 360] [color={rgb, 255:red, 0; green, 0; blue, 0 }  ][line width=0.75]    (10.93,-3.29) .. controls (6.95,-1.4) and (3.31,-0.3) .. (0,0) .. controls (3.31,0.3) and (6.95,1.4) .. (10.93,3.29)   ;
\draw    (298,185) -- (272,185) ;
\draw [shift={(270,185)}, rotate = 360] [color={rgb, 255:red, 0; green, 0; blue, 0 }  ][line width=0.75]    (10.93,-3.29) .. controls (6.95,-1.4) and (3.31,-0.3) .. (0,0) .. controls (3.31,0.3) and (6.95,1.4) .. (10.93,3.29)   ;
\draw   (518.67,155.33) -- (584,155.33) -- (584,188.33) -- (518.67,188.33) -- cycle ;
\draw    (518.67,155.33) -- (457.85,71.95) ;
\draw [shift={(456.67,70.33)}, rotate = 53.89] [color={rgb, 255:red, 0; green, 0; blue, 0 }  ][line width=0.75]    (10.93,-3.29) .. controls (6.95,-1.4) and (3.31,-0.3) .. (0,0) .. controls (3.31,0.3) and (6.95,1.4) .. (10.93,3.29)   ;
\draw    (517.33,166) -- (457.03,128.72) ;
\draw [shift={(455.33,127.67)}, rotate = 31.73] [color={rgb, 255:red, 0; green, 0; blue, 0 }  ][line width=0.75]    (10.93,-3.29) .. controls (6.95,-1.4) and (3.31,-0.3) .. (0,0) .. controls (3.31,0.3) and (6.95,1.4) .. (10.93,3.29)   ;
\draw    (518.67,178) -- (459.33,178.32) ;
\draw [shift={(457.33,178.33)}, rotate = 359.69] [color={rgb, 255:red, 0; green, 0; blue, 0 }  ][line width=0.75]    (10.93,-3.29) .. controls (6.95,-1.4) and (3.31,-0.3) .. (0,0) .. controls (3.31,0.3) and (6.95,1.4) .. (10.93,3.29)   ;
\draw    (518.67,188.33) -- (458.51,271.38) ;
\draw [shift={(457.33,273)}, rotate = 305.92] [color={rgb, 255:red, 0; green, 0; blue, 0 }  ][line width=0.75]    (10.93,-3.29) .. controls (6.95,-1.4) and (3.31,-0.3) .. (0,0) .. controls (3.31,0.3) and (6.95,1.4) .. (10.93,3.29)   ;
\draw   (78.67,154.67) -- (144,154.67) -- (144,187.67) -- (78.67,187.67) -- cycle ;
\draw    (200,68.33) -- (145.09,152.99) ;
\draw [shift={(144,154.67)}, rotate = 302.97] [color={rgb, 255:red, 0; green, 0; blue, 0 }  ][line width=0.75]    (10.93,-3.29) .. controls (6.95,-1.4) and (3.31,-0.3) .. (0,0) .. controls (3.31,0.3) and (6.95,1.4) .. (10.93,3.29)   ;
\draw    (198.67,123) -- (146.92,162.45) ;
\draw [shift={(145.33,163.67)}, rotate = 322.67] [color={rgb, 255:red, 0; green, 0; blue, 0 }  ][line width=0.75]    (10.93,-3.29) .. controls (6.95,-1.4) and (3.31,-0.3) .. (0,0) .. controls (3.31,0.3) and (6.95,1.4) .. (10.93,3.29)   ;
\draw    (200,185) -- (147.29,174.07) ;
\draw [shift={(145.33,173.67)}, rotate = 11.71] [color={rgb, 255:red, 0; green, 0; blue, 0 }  ][line width=0.75]    (10.93,-3.29) .. controls (6.95,-1.4) and (3.31,-0.3) .. (0,0) .. controls (3.31,0.3) and (6.95,1.4) .. (10.93,3.29)   ;
\draw    (198.67,271) -- (145.1,189.34) ;
\draw [shift={(144,187.67)}, rotate = 56.74] [color={rgb, 255:red, 0; green, 0; blue, 0 }  ][line width=0.75]    (10.93,-3.29) .. controls (6.95,-1.4) and (3.31,-0.3) .. (0,0) .. controls (3.31,0.3) and (6.95,1.4) .. (10.93,3.29)   ;
\draw    (616,172.33) -- (586.67,172.33) ;
\draw [shift={(584.67,172.33)}, rotate = 360] [color={rgb, 255:red, 0; green, 0; blue, 0 }  ][line width=0.75]    (10.93,-3.29) .. controls (6.95,-1.4) and (3.31,-0.3) .. (0,0) .. controls (3.31,0.3) and (6.95,1.4) .. (10.93,3.29)   ;
\draw    (78.67,169.67) -- (49.33,169.67) ;
\draw [shift={(47.33,169.67)}, rotate = 360] [color={rgb, 255:red, 0; green, 0; blue, 0 }  ][line width=0.75]    (10.93,-3.29) .. controls (6.95,-1.4) and (3.31,-0.3) .. (0,0) .. controls (3.31,0.3) and (6.95,1.4) .. (10.93,3.29)   ;

\draw (412,217.73) node [anchor=north west][inner sep=0.75pt]    {$\vdots $};
\draw (406,61.4) node [anchor=north west][inner sep=0.75pt]    {$\mathsf{Pwr}^{q,\ve}_{0}$};
\draw (406,118.07) node [anchor=north west][inner sep=0.75pt]    {$\mathsf{Pwr}^{q,\ve}_{1}$};
\draw (403.33,177.07) node [anchor=north west][inner sep=0.75pt]    {$\mathsf{Pwr}^{q,\ve}_{2}$};
\draw (265.33,58.4) node [anchor=north west][inner sep=0.75pt]    {$\mathsf{Tun}$};
\draw (404,263.07) node [anchor=north west][inner sep=0.75pt]    {$\mathsf{Pwr}^{q,\ve}_{n}$};
\draw (249.33,115.73) node [anchor=north west][inner sep=0.75pt]    {$\mathsf{Tun}$};
\draw (222,176.4) node [anchor=north west][inner sep=0.75pt]    {$\mathsf{Tun}$};
\draw (525,162.4) node [anchor=north west][inner sep=0.75pt]    {$\mathsf{Cpy}_{n+1,1}$};
\draw (471.33,198.4) node [anchor=north west][inner sep=0.75pt]    {$\vdots $};
\draw (83,163.73) node [anchor=north west][inner sep=0.75pt]    {$\mathsf{Sum}_{n+1,1}$};
\draw (230.67,214.4) node [anchor=north west][inner sep=0.75pt]    {$\vdots $};
\draw (172,193.73) node [anchor=north west][inner sep=0.75pt]    {$\vdots $};

\end{tikzpicture}
\end{center}
\caption{Neural network diagram for an elementary neural network polynomial.}
\end{figure}

\subsection{$\xpn_n^{q,\ve}$, $\csn_n^{q,\ve}$, $\sne_n^{q,\ve}$, and their properties.}

\subsubsection{The $\xpn^{q,\ve}_n$ Network and Their Accuracies}

\begin{lemma}\label{6.2.9}\label{tay_for_exp}\label{xpn_properties}
	Let $\delta,\ve \in \lp 0,\infty \rp $, $q\in \lp 2,\infty \rp$ and $\delta = \ve \lp 2^{q-1} +1\rp^{-1}$. It is then the case for all $n\in\N_0$ and $x\in \R$ that:
	\begin{enumerate}
		\item $\real_{\rect} \lp \xpn_n^{q,\ve}\rp \lp x \rp\in C \lp \R, \R \rp $
		\item $\dep \lp \xpn_n^{q,\ve} \rp \les \begin{cases}
			1 & :n=0\\
			n\lb \frac{q}{q-2} \lb \log_2 \lp \ve^{-1} \rp +q\rb -1 \rb +1 &:n\in \N
		\end{cases}$
		\item $\param \lp \xpn_n^{q,\ve} \rp \les \begin{cases}
			2 & :n =0 \\
			\lp n+1\rp\lb 4^{n+\frac{3}{2}} + \lp \frac{4^{n+1}-1}{3}\rp \lp \frac{360q}{q-2} \lb \log_2 \lp \ve^{-1} \rp +q+1 \rb +372\rp\rb &:n\in \N
		\end{cases}$ \\~\\
		\item \begin{align*}\left|\sum^n_{i=0} \lb \frac{x^i}{i!} \rb- \real_{\rect} \lp \xpn_n^{q,\ve} \rp \lp x \rp \right| \les \sum^n_{i=1} \frac{1}{i!}\lp \left| x \lp x^{i-1} - \real_{\rect}\lp \pwr^{q,\ve}_{i-1}\rp\lp x\rp\rp\right| + \ve + |x|^q + \mathfrak{p}_{i-1}^q	 \rp  \end{align*}\\~\\
		Where $\mathfrak{p}_i$ are the set of functions defined for $i \in \N$ as such:
		\begin{align}
		\mathfrak{p}_1 &= \ve+2+2|x|^2 \nonumber\\
		\mathfrak{p}_i &= \ve +2\lp \mathfrak{p}_{i-1} \rp^2+2|x|^2
		\end{align}
		Whence it is the case that:
		\begin{align}
		\left|\sum^n_{i=0} \lb \frac{x^i}{i!} \rb- \real_{\rect} \lp \xpn_n^{q,\ve} \rp \lp x \rp \right|\in \mathcal{O} \lp \ve^{2q(n-1)}\rp
		\end{align}
		\item $\wid_1 \lp \xpn_n^{q,\ve} \rp =  2+23n+n^2 $
		\item $\wid_{\hid \lp \xpn^n_{q,\ve} \rp}\lp \xpn_n^{q,\ve}\rp \les 24 + 2n$
	\end{enumerate}
\end{lemma}
\begin{proof}
	This follows straightforwardly from Lemma \ref{nn_poly} with $c_i \curvearrowleft \frac{1}{i!}$ for all $n \in \N$ and $i \in \{0,1,\hdots, n\}$. In particular, Item (iv) benefits from the fact that for all $i \in \N_0$, it is the case that $\frac{1}{i!} \ges 0$.
\end{proof}
\begin{lemma}

	Let $\delta,\ve \in \lp 0,\infty \rp $, $q\in \lp 2,\infty \rp$ and $\delta = \ve \lp 2^{q-1} +1\rp^{-1}.$ It is then the case for fixed $n\in\N_0$, fixed $b \in \lb 0,\infty\rp$ and for all $x\in \lb 0,b \rb\subseteq \lb 0,\infty \rp$ that:
	\begin{align}
		\left| e^x - \real_{\rect} \lp \xpn_n^{q,\ve} \rp \lp x \rp  \right| \les \sum^n_{i=0} \frac{1}{i!}\lp \left| x \lp x^{n-1} - \real_{\rect}\lp \pwr^{q,\ve}_{n-1}\rp\lp x\rp\rp\right| + \ve + |x|^q + \mathfrak{p}_{n-1}^q	 \rp  + \left| \frac{e^{b}\cdot b^{n+1}}{(n+1)!}\right|
	\end{align}
\end{lemma}
\begin{proof}
	Note that Taylor's theorem states that for $x \in \lb 0,b\rb \subseteq \lb 0,\infty\rp$ it is the case that:
	\begin{align}
		e^x = \sum^n_{i=0} \lb \frac{x^i}{i!} \rb + \frac{e^{\xi}\cdot x^{n+1}}{(n+1)!}
	\end{align}
	Where $\xi \in \lb 0, x \rb$ in the Lagrange form of the remainder. Note then, for all $n\in \N_0$, $x\in \lb 0,b \rb \subseteq \lb 0, \infty \rp$, and $\xi \in \lb 0,x \rb$ it is the case that the second summand is bounded by:
	\begin{align}
		\frac{e^\xi \cdot x^{n+1}}{(n+1)!} \les \frac{e^b\cdot b^{n+1}}{(n+1)!}
	\end{align}
	This, and the triangle inequality, then indicates that for all $x \in \lb 0,b \rb \subseteq \lb 0,\infty\rp$, and $\xi \in \lb 0, x\rb$ that:
	\begin{align}
		\left| e^x -\real_{\rect} \lp \xpn_n^{q,\ve} \rp \lp x \rp \right| &=\left| \sum^n_{i=0} \lb \frac{x^i}{i!} \rb + \frac{e^{\xi}\cdot x^{n+1}}{(n+1)!}-\real_{\rect} \lp \xpn_n^{q,\ve} \rp \lp x \rp\right| \nonumber\\
		&\les  \left| \sum^n_{i=0} \lb \frac{x^i}{i!} \rb - \real_{\rect} \lp \xpn_n^{q,\ve} \rp \lp x \rp \right| + \frac{e^{b}\cdot b^{n+1}}{(n+1)!} \nonumber \\
		&\les \sum^n_{i=1} \frac{1}{i!}\lp \left| x \lp x^{n-1} - \real_{\rect}\lp \pwr^{q,\ve}_{n-1}\rp\lp x\rp\rp\right| + \ve + |x|^q + \mathfrak{p}_{n-1}^q	 \rp + \frac{e^{b}\cdot b^{n+1}}{(n+1)!} \nonumber
	\end{align}
	Whence we have that for fixed $n\in \N_0$ and $b \in \lb 0, \infty\rp$, the last summand is constant, whence it is the case for fixed $n\in \N_0$ and $b \in \lb 0, \infty\rp$, that:
	\begin{align}
		\left| e^x -\real_{\rect} \lp \xpn_n^{q,\ve} \rp \lp x \rp \right|  \in \mathcal{O} \lp \ve^{2q(n-1)}\rp
	\end{align}
\end{proof}

\subsubsection{$\csn_n^{q,\ve}$ Networks and their accuracies}

\begin{lemma}\label{6.2.9}\label{sne_properties}
	Let $\delta,\ve \in \lp 0,\infty \rp $, $q\in \lp 2,\infty \rp$ and $\delta = \ve \lp 2^{q-1} +1\rp^{-1}$. It is then the case for all $n\in\N_0$ and $x\in \R$ that:
	\begin{enumerate}
		\item $\real_{\rect} \lp \csn_n^{q,\ve}\rp \in C \lp \R, \R \rp $
		\item $\dep \lp \csn_n^{q,\ve}\rp \les \begin{cases}
			1 & :n=0\\
			2n\lb \frac{q}{q-2} \lb \log_2 \lp \ve^{-1} \rp +q\rb -1 \rb +1 &:n\in \N
		\end{cases}$
		\item $\param \lp \csn_n^{q,\ve} \rp \les \begin{cases}
			2 & :n =0 \\
			\lp 2n+1\rp\lb 4^{2n+\frac{3}{2}} + \lp \frac{4^{2n+1}-1}{3}\rp \lp \frac{360q}{q-2} \lb \log_2 \lp \ve^{-1} \rp +q+1 \rb +372\rp\rb &:n\in \N
		\end{cases}$ \\~\\
		\item $\left|\sum^n_{i=0} \frac{(-1)^i}{2i!}x^{2i} - \real_{\rect} \lp \csn_n^{q,\ve} \rp \lp x \rp \right| \les \sum^n_{i=1} \left| \frac{\lp -1\rp^i}{2i!}\right|\lp \left| x \lp x^{2i-1} - \real_{\rect}\lp \pwr^{q,\ve}_{2i-1}\rp\lp x\rp\rp\right| + \ve + |x|^q + \mathfrak{p}_{2i-1}^q	 \rp  $\\~\\
		Where $\mathfrak{p}_i$ are the set of functions defined for $i \in \N$ as such:
		\begin{align}
		\mathfrak{p}_1 &= \ve+2+2|x|^2 \nonumber\\
		\mathfrak{p}_i &= \ve +2\lp \mathfrak{p}_{i-1} \rp^2+2|x|^2
		\end{align}
		Whence it is the case that:
		\begin{align}
		\left|\sum^n_{i=0} \frac{\lp -1\rp^i}{2i!}x^{2i} - \real_{\rect} \lp \csn_n^{q,\ve} \rp \lp x \rp \right| \in \mathcal{O} \lp \ve^{2q(2n-1)}\rp
		\end{align}
	\end{enumerate}
	
\end{lemma}

\begin{proof}
	Item (i) derives straightforwardly from Lemma \ref{nn_poly}. This proves Item (i).
	
	Next, observe that since $\csn_n^{q,\ve}$ will contain, as the deepest network in the summand, $\pwr_{2n}^{q,\ve}$, we may then conclude that
	\begin{align}
		\dep \lp \csn_n^{q,\ve} \rp &\les \dep \lp \pwr_{2n}^{q,\ve}\rp \nonumber\\
		&\les \begin{cases}
			1 & :n=0\\
			2n\lb \frac{q}{q-2} \lb \log_2 \lp \ve^{-1} \rp +q\rb -1 \rb +1 &:n\in \N
		\end{cases} \nonumber
	\end{align}
	This proves Item (ii).
	
	A similar argument to the above, Lemma \ref{aff_effect_on_layer_architecture}, and Corollary \ref{affcor} reveals that:
	\begin{align}
		\param\lp \csn_n^{q,\ve} \rp &= \param  \lp  \bigoplus^n_{i=0} \lb \frac{\lp -1\rp^i}{2i!} \triangleright\lb \tun_{\max_i \left\{\dep \lp \pwr_i^{q,\ve} \rp\right\} +1 - \dep \lp \pwr^{q,\ve}_i\rp} \bullet \pwr_i^{q,\ve}\rb \rb \rp\nonumber \\
		&\les \lp n+1 \rp \cdot \param \lp c_i \triangleright \lb \tun_1 \bullet \pwr_{2n}^{q,\ve} \rb\rp \nonumber\\
		&\les \lp n+1 \rp \cdot \param \lp \pwr_{2n}^{q,\ve} \rp \nonumber \\
		&\les \begin{cases}
			2 & :n =0 \\
			\lp n+1\rp\lb 4^{2n+\frac{3}{2}} + \lp \frac{4^{2n+1}-1}{3}\rp \lp \frac{360q}{q-2} \lb \log_2 \lp \ve^{-1} \rp +q+1 \rb +372\rp\rb &:n\in \N
		\end{cases} \nonumber
	\end{align}
	This proves Item (iii).
	
	In a similar vein, we may argue from Lemma \ref{nn_poly} and from the absolute homogeneity property of norms that:
	\begin{align}
		&\left|\sum^n_{i=0} \frac{\lp -1\rp^i}{2i!}x^{2i} - \real_{\rect} \lp \csn_n^{q,\ve} \lp x\rp \rp \right| \nonumber\\
		&= \left| \sum^n_{i=0}  \frac{\lp -1\rp^i}{2i!}x^{2i} - \real_{\rect} \lb \bigoplus^n_{i=0} \lb \frac{\lp -1\rp^i}{2i!} \triangleright \tun_{\max_{2i} \left\{\dep \lp \pwr_{2i}^{q,\ve} \rp\right\} +1 - \dep \lp \pwr^{q,\ve}_{2i}\rp} \bullet \pwr_{2i}^{q,\ve} \rb \rb\lp x \rp\right| \nonumber \\
		&=\left| \sum^n_{i=1} \frac{\lp -1\rp^i}{2i!}x^{2i}-\sum_{i=0}^n \frac{\lp -1 \rp^i}{2i!} \lp \inst_{\rect}\lb \tun_{\max_{2i} \left\{\dep \lp \pwr_{2i}^{q,\ve} \rp\right\} +1 - \dep \lp \pwr^{q,\ve}_{2i}\rp} \bullet \pwr_{2i}^{q,\ve}\rb\lp x\rp\rp\right| \nonumber\\
		&\les \sum_{i=1}^n \left|\frac{\lp -1\rp^i}{2i!} \right|\cdot\left| x^{2i} - \inst_{\rect}\lb \tun_{\max_{2i} \left\{\dep \lp \pwr_{2i}^{q,\ve} \rp\right\} +1 - \dep \lp \pwr^{q,\ve}_{2i}\rp} \bullet \pwr_{2i}^{q,\ve}\rb\lp x\rp\right| \nonumber\\
		&\les \sum^n_{i=1} \left|\frac{\lp -1\rp^i}{2i!}\right|\cdot \left|\lp \left| x \lp x^{2i-1} - \real_{\rect}\lp \pwr^{q,\ve}_{2i-1}\rp\lp x\rp\rp\right| + \ve + 2|x|^q + 2\mathfrak{p}_{2i-1}^q	 \rp\right| \nonumber
	\end{align}
	Whence we have that:
	\begin{align}
		\left|\sum^n_{i=0} \lb \frac{\lp -1\rp^i x^{2i}}{2i!} \rb- \real_{\rect} \lp \csn_n^{q,\ve} \rp \lp x \rp \right|\in \mathcal{O} \lp \ve^{2q(2n-1)}\rp
		\end{align}
	This proves Item (iv). This then completes the Lemma.
\end{proof}
\begin{lemma}
	Let $\delta,\ve \in \lp 0,\infty \rp $, $q\in \lp 2,\infty \rp$ and $\delta = \ve \lp 2^{q-1} +1\rp^{-1}.$ It is then the case for fixed $n\in\N_0$, fixed $b \in \lb 0,\infty \rp$ and for all $x\in [a,b]\subseteq \lb 0,\infty \rp$ that:
	\begin{align}
		\left| \cos\lp x\rp  - \real_{\rect} \lp \csn_n^{q,\ve} \rp \lp x \rp  \right| \les \sum^n_{i=0} \frac{\lp -1\rp^i}{2i!}\lp \left| x \lp x^{n-1} - \real_{\rect}\lp \pwr^{q,\ve}_{n-1}\rp\lp x\rp\rp\right| + \ve + |x|^q + \mathfrak{p}_{n-1}^q	 \rp  + \frac{b^{n+1}}{(n+1)!}\nonumber
	\end{align}
\end{lemma}

\begin{proof}
	Note that Taylor's theorem states that for $x \in \lb 0,b\rb \subseteq \lb 0,\infty\rp$ it is the case that:
	\begin{align}
		\cos\lp x \rp= \sum^n_{i=0} \frac{\lp -1\rp^i}{2i!}x^i + \frac{\cos^{\lp n+1\rp}\lp \xi \rp \cdot x^{n+1}}{(n+1)!}
	\end{align}
	Note further that for all $n \in \N_0$, and $x \in \R$, it is the case that $\cos^{\lp n \rp} \lp x\rp \les 1$. Whence we may conclude that for all $n\in \N_0$, $x\in \lb 0,b \rb \subseteq \lb 0, \infty \rp$, and $\xi \in \lb 0,x \rb$, we may bound the second summand by:
	\begin{align}
		\frac{\cos^{\lp n+1\rp}\lp \xi \rp \cdot x^{n+1}}{(n+1)!} \les \frac{b^{n+1}}{\lp n+1\rp!}
	\end{align}
	This, and the triangle inequality, then indicates that for all $x \in \lb 0,b \rb \subseteq \lb 0,\infty\rp$ and $\xi \in \lb 0,x\rb$:
	\begin{align}
		\left| \cos \lp x \rp -\real_{\rect} \lp \csn_n^{q,\ve} \rp \lp x \rp \right| &=\left| \sum^n_{i=0} \frac{\lp -1\rp^i}{2i!}x^i + \frac{\cos^{(n+1)}\lp \xi \rp \cdot x^{n+1}}{(n+1)!}-\real_{\rect} \lp \csn_n^{q,\ve} \rp \lp x \rp\right| \nonumber\\
		&\les  \left| \sum^n_{i=0} \frac{\lp -1\rp^i}{2i!}x^i - \real_{\rect} \lp \csn_n^{q,\ve} \rp \lp x \rp \right| + \frac{b^{n+1}}{(n+1)!} \nonumber \\
		&\les \sum^n_{i=1} \left|\frac{\lp -1\rp^i}{2i!}\right|\cdot \left|\lp \left| x \lp x^{2i-1} - \real_{\rect}\lp \pwr^{q,\ve}_{2i-1}\rp\lp x\rp\rp\right| + \ve + |x|^q + \mathfrak{p}_{2i-1}^q	 \rp\right| \nonumber\\&+ \frac{b^{n+1}}{(n+1)!} \nonumber
	\end{align}
	This completes the proof of the Lemma.
\end{proof}

\subsubsection{$\sne_n^{q,\ve}$ networks and their accuracies}

\begin{lemma}\label{6.2.9}\label{csn_properties}
	Let $\delta,\ve \in \lp 0,\infty \rp $, $q\in \lp 2,\infty \rp$ and $\delta = \ve \lp 2^{q-1} +1\rp^{-1}$. It is then the case for all $n\in\N_0$ and $x\in \R$ that:
	\begin{enumerate}
		\item $\real_{\rect} \lp \sne_n^{q,\ve}\rp \in C \lp \R, \R \rp $
		\item $\dep \lp \sne_n^{q,\ve}\rp \les \begin{cases}
			1 & :n=0\\
			2n\lb \frac{q}{q-2} \lb \log_2 \lp \ve^{-1} \rp +q\rb -1 \rb +1 &:n\in \N
		\end{cases}$
		\item $\param \lp \sne_n^{q,\ve} \rp \les \begin{cases}
			2 & :n =0 \\
			\lp 2n+1\rp\lb 4^{2n+\frac{3}{2}} + \lp \frac{4^{2n+1}-1}{3}\rp \lp \frac{360q}{q-2} \lb \log_2 \lp \ve^{-1} \rp +q+1 \rb +372\rp\rb &:n\in \N
		\end{cases}$ \\~\\
		\item \begin{align}&\left|\sum^n_{i=0} \frac{(-1)^i}{2i!}{\lp x-\frac{\pi}{2}\rp}^{2i} - \real_{\rect} \lp \sne_n^{q,\ve} \rp \lp x \rp \right| \nonumber\\
		&= \left|\sum^n_{i=0} \frac{(-1)^i}{2i!}{\lp x-\frac{\pi}{2}\rp}^{2i} - \real_{\rect} \lp \csn_n^{q,\ve} \bullet \aff_{1,-\frac{\pi}{2}}\rp \lp x \rp \right|\nonumber\\
		&\les \sum^n_{i=1} \left| \frac{\lp -1\rp^i}{2i!}\right|\lp \left| \lp x -\frac{\pi}{2}\rp\lp \lp x -\frac{\pi}{2}\rp^{2i-1} - \real_{\rect}\lp \pwr^{q,\ve}_{i-1}\rp\lp x-\frac{\pi}{2}\rp\rp\right| + \ve + |x|^q + \mathfrak{p}_{i-1}^q	 \rp \nonumber \end{align}\\~\\
		Where $\mathfrak{p}_i$ are the set of functions defined for $i \in \N$ as such:
		\begin{align}
		\mathfrak{p}_1 &= \ve+2+2|x|^2 \nonumber\\
		\mathfrak{p}_i &= \ve +2\lp \mathfrak{p}_{i-1} \rp^2+2|x|^2
		\end{align}
		Whence it is the case that:
		\begin{align}
		\left|\sum^n_{i=0} \frac{\lp -1\rp^i}{2i!}\lp x-\frac{\pi}{2}\rp^{2i} - \real_{\rect} \lp \sne_n^{q,\ve} \rp \lp x \rp \right| \in \mathcal{O} \lp \ve^{2q(2n-1)}\rp
		\end{align}
	\end{enumerate}
\end{lemma}
\begin{proof}
	This follows straightforwardly from Lemma \ref{csn_properties}, and the fact that by Corollary 2.9 in \cite{grohs2019spacetime}, there is not a change to the parameter count, by Proposition 2.6 in \cite{grohs2019spacetime}, there is no change in depth, by Proposition 2.6 in \cite{grohs2019spacetime}, Lemma 2.3.2 in \cite{bigbook}, and Lemma \ref{csn_properties}, continuity is preserved, and the fact that $\aff_{1,-\frac{\pi}{2}}$ is exact and hence contributes nothing to the error, and finally by the fact that $\aff_{1,-\frac{\pi}{2}} \rightarrow \lp \cdot\rp -\frac{\pi}{2}$  under instantiation, assures us that the $\sne^{q,\ve}_n$ has the same error bounds as $\csn_n^{q,\ve}$. 
\end{proof}

\begin{lemma}
	Let $\delta,\ve \in \lp 0,\infty \rp $, $q\in \lp 2,\infty \rp$ and $\delta = \ve \lp 2^{q-1} +1\rp^{-1}.$ It is then the case for fixed $n\in\N_0$, fixed $b\in \lb 0,\infty \rp$ and for all $x\in [a,b]\subseteq \lb 0,\infty \rp$ that:
	\begin{align}
		&\left| \sin\lp x\rp  - \real_{\rect} \lp \sne_n^{q,\ve} \rp \lp x \rp  \right|\nonumber \\ 
		&\les \sum^n_{i=1} \left| \frac{\lp -1\rp^i}{2i!}\right|\lp \left| \lp x -\frac{\pi}{2}\rp\lp \lp x -\frac{\pi}{2}\rp^{2i-1} - \real_{\rect}\lp \pwr^{q,\ve}_{i-1}\rp\lp x-\frac{\pi}{2}\rp\rp\right| + \ve + |x|^q + \mathfrak{p}_{i-1}^q	 \rp  \nonumber\\ 
		&+ \frac{b^{n+1}}{(n+1)!}\label{sin_diff}
	\end{align}
\end{lemma}
\begin{proof}
	Note that the fact that $\sin\lp x\rp = \cos\lp x-\frac{\pi}{2}\rp$, Proposition 2.6 in \cite{grohs2019spacetime}, and Lemma \ref{aff_prop} then renders (\ref{sin_diff}) as:
	\begin{align}
		&\left| \sin\lp x\rp - \inst_{\rect}\lp \sne_n^{q,\ve}\rp\right| \nonumber\\
		&= \left| \cos \lp x - \frac{\pi}{2}\rp - \inst_{\rect}\lp \csn_n^{q,\ve}\bullet \aff_{1,-\frac{\pi}{2}}\rp\lp x\rp\right| \nonumber\\
		&=\left| \cos \lp x-\frac{x}{2}\rp - \inst_{\rect}\csn_n^{q,\ve}\lp x-\frac{\pi}{2} \rp\right| \nonumber \\
		&\les \sum^n_{i=1} \left| \frac{\lp -1\rp^i}{2i!}\right|\lp \left| \lp x -\frac{\pi}{2}\rp\lp \lp x -\frac{\pi}{2}\rp^{2i-1} - \real_{\rect}\lp \pwr^{q,\ve}_{i-1}\rp\lp x-\frac{\pi}{2}\rp\rp\right| + \ve + |x|^q + \mathfrak{p}_{i-1}^q	 \rp+ \frac{b^{n+1}}{(n+1)!}\nonumber
	\end{align}
\end{proof}

\subsection{The $\mathsf{E}^{N,h,q,\ve}_n$ Network}

\begin{lemma}\label{mathsfE}
	Let $n, N\in \N$ and $h \in \lp 0,\infty\rp$. Let $\delta,\ve \in \lp 0,\infty \rp $, $q\in \lp 2,\infty \rp$, satisfy that $\delta = \ve \lp 2^{q-1} +1\rp^{-1}$. Let $a\in \lp -\infty,\infty \rp$, $b \in \lb a, \infty \rp$. Let $f:[a,b] \rightarrow \R$ be continuous and have second derivatives almost everywhere in $\lb a,b \rb$. Let $a=x_0 \les x_1\les \cdots \les x_{N-1} \les x_N=b$ such that for all $i \in \{0,1,...,N\}$ it is the case that $h = \frac{b-a}{N}$, and $x_i = x_0+i\cdot h$ . Let $x = \lb x_0 \: x_1\: \cdots x_N \rb$ and as such let $f\lp\lb x \rb_{*,*} \rp = \lb f(x_0) \: f(x_1)\: \cdots \: f(x_N) \rb$. Let $\mathsf{E}^{N,h,q,\ve}_{n} \in \neu$ be the neural network given by:
	\begin{align}
		\mathsf{E}^{N,h,q,\ve}_n = \xpn_n^{q,\ve} \bullet \etr^{N,h}
	\end{align}
	It is then the case that:
	\begin{enumerate}
		\item $\lp \real_{\rect}\lp \mathsf{E}^{N,h,q,\ve}_n \rp\rp\lp x\rp \in C \lp \R^N,\R\rp$
		\item $\dep\lp \mathsf{E}^{N,h,q,\ve}_n \rp \les \begin{cases}
			1 & n=0\\
			n\lb \frac{q}{q-2} \lb \log_2 \lp \ve^{-1} \rp +q\rb -1 \rb +1 &n\ges 1
		\end{cases} $
		\item \begin{align*}&\param \lp \mathsf{E}^{N,h,q,\ve}_{n}\rp \\&\les \begin{cases}
			N+2 & :n =0 \\
			\lp \frac{1}{2}N+1 \rp\lp n+1\rp\lb 4^{n+\frac{3}{2}} + \lp \frac{4^{n+1}-1}{3}\rp \lp \frac{360q}{q-2} \lb \log_2 \lp \ve^{-1} \rp +q+1 \rb +372\rp\rb &:n\in \N
		\end{cases} \end{align*}
		\item for all $x = \{x_0,x_1,\hdots, x_N \}\in \R^{N+1}$, where $0 \les a=x_0 \les x_1\les \cdots \les x_{N-1} \les x_N=b \les \infty$, and where $\int^b_a f dx \in \lb 0,\infty\rp$, we have that:
		\begin{align}
		&\left| \exp \lb \int^b_afdx\rb  - \real_{\rect} \lp \mathsf{E}^{N,h,q,\ve}_{n}\rp\lp f \lp \lb x \rb _{*,*}\rp\rp\right| \nonumber\\
		&\les \frac{\lp b-a\rp^3}{12N^2}f''\lp \xi \rp \cdot n^2 \cdot \lb \Xi + \frac{\lp b-a\rp^3}{12N^2} f''\lp \xi\rp\rb^{n-1} + \nonumber \\
		&\sum^n_{i=1} \frac{1}{i!}\lp \left| \Xi \lp \Xi^{i-1} - \real_{\rect}\lp \pwr^{q,\ve}_{i-1}\rp\lp \Xi\rp\rp\right| + \ve + |\Xi|^q + \mathfrak{p}_{i-1}^q	 \rp
		\end{align}
		\item it is the case that $\wid_{\hid \lp \mathsf{E}^{\exp,f}_{N,n,h,q,\ve}\rp} \lp \mathsf{E}^{\exp,f}_{N,n,h,q,\ve}\rp =1+4n $
	\end{enumerate}
\end{lemma}
\begin{proof}
	Note that Lemma \ref{etr_prop}, tells us that $ \inst_{\rect}\lp \etr^{N,h}\rp \in C\lp \R^{N+1},\R\rp$, and Lemma \ref{xpn_properties} tells us that $\inst_{\rect}\lp \xpn^{q,\ve}_n\rp \lp x\rp \in C \lp \R, \R \rp$. Next, note that Proposition 2.6 in \cite{grohs2019spacetime}, and the fact that the composition of continuous functions is continuous yields that:
	\begin{align}
		\real_{\rect} \lp \mathsf{E}^{N,h,q,\ve}_n\rp &= \real_{\rect}  \lp \xpn_n^{q,\ve} \bullet \aff_{\lb \frac{h}{2} \: h \:\hdots \:h \: \frac{h}{2}\rb,0}\rp \nonumber \\
		&= \real_{\rect} \lp \xpn_n^{q,\ve} \rp \circ \real_{\rect} \lp  \aff_{\lb \frac{h}{2} \: h \:\hdots \:h \: \frac{h}{2}\rb,0} \rp \in C \lp \R^{N+1},\R \rp\nonumber
	\end{align}
	Since both component neural networks are continuous, and the composition of continuous functions is continuous, so is $\mathsf{E}$. This proves Item (i).
	
	Next note that $\dep \lp \aff_{\lb \frac{h}{2} \: h \:\hdots \:h \: \frac{h}{2}\rb} \rp = 1$, and thus Proposition 2.6 in \cite{grohs2019spacetime} and Lemma \ref{xpn_properties} tells us that:
	\begin{align}
		\dep \lp \mathsf{E}^{N,h,q,\ve}_n\rp &= \dep \lp \xpn^{q,\ve}_{n} \bullet \aff_{\lb \frac{h}{2} \: h \:\hdots \:h \: \frac{h}{2}\rb,0}\rp \nonumber \\
		&= \nonumber \dep \lp \xpn^{q,\ve}_{n} \rp + \dep \lp \aff_{\lb \frac{h}{2} \: h \:\hdots \:h \: \frac{h}{2}\rb,0} \rp -1 \nonumber \\
		&=\dep \lp \xpn^{q,\ve}_{n}\rp \nonumber \\
		&\les \begin{cases}
			1 & :n=0\\
			n\lb \frac{q}{q-2} \lb \log_2 \lp \ve^{-1} \rp +q\rb -1 \rb +1 &:n\in \N
		\end{cases} \nonumber
	\end{align}
	This proves Item (ii).
	
	Next note that by Corollary 2.9 in \cite{grohs2019spacetime}, Lemma \ref{xpn_properties}, Lemma \ref{etr_prop}, and the fact that $\inn\lp \etr^{N,h}\rp = N$, and $\inn \lp \xpn_n^{q,\ve}\rp = 1$, tells us that, for all $N \in \N$ it is the case that:
	\begin{align}
		&\param \lp \mathsf{E}^{N,h,q,\ve}_n\rp \nonumber\\
		&\les \lb \max \left\{1, \frac{\inn\lp \etr^{N,h}\rp+1}{\inn\lp \xpn_n^{q,\ve}\rp+1}\right\}\rb \cdot \param\lp \xpn_n^{q,\ve}\rp \nonumber\\
		&=\lp \frac{1}{2}N+1 \rp \cdot \param \lp \xpn_n^{q,\ve}\rp \nonumber \\
		&\les \begin{cases}
			N+2 & :n =0 \\
			\lp \frac{1}{2}N+1 \rp\lp n+1\rp\lb 4^{n+\frac{3}{2}} + \lp \frac{4^{n+1}-1}{3}\rp \lp \frac{360q}{q-2} \lb \log_2 \lp \ve^{-1} \rp +q+1 \rb +372\rp\rb &:n\in \N
		\end{cases}\nonumber
	\end{align}	
	This proves Item (iii). 
	
	Note next that:
	\begin{align}
		\aff_{\lb \frac{h}{2} \: h \:\hdots \:h \: \frac{h}{2}\rb,0} = \etr^{N,h}
	\end{align} 
	Thus the well-known error term of the trapezoidal rule tells us that for $\lb a,b \rb \subseteq \lb 0, \infty \rp$, and for $\xi \in \lb a,b \rb$ it is the case that:
	\begin{align}
	\left| \int^b_a f\lp x \rp dx - \lp \real_{\rect}\lp \etr^{N,h} \rp\rp \lp f\lp \lb x \rb_{*,*}\rp\rp  \right| \les \frac{ \lp b-a\rp^3}{12N^2} f''\lp \xi \rp 
	\end{align}
	and note also that for $n\in \N_0$, $\delta,\ve \in \lp 0,\infty \rp $, $q\in \lp 2,\infty \rp$ and $\delta = \ve \lp 2^{q-1} +1\rp^{-1}$, and for $x \in \lb 0,b\rb \subseteq \lb 0,\infty\rp$ it is the case that:
	\begin{align}
	\left| e^x - \real_{\rect} \lp \tay^{\exp}_{n,q,\ve} \rp \lp x \rp  \right| &\les \sum^n_{i=1} \frac{\ve}{i!}\lp \left| x \lp x^{n-1} - \real_{\rect}\lp \pwr^{q,\ve}_{n-1}\rp\lp x\rp\rp\right| + \ve + |x|^q + \mathfrak{p}_{n-1}^q	 \rp \nonumber\\&+ \left| \frac{e^{b}\cdot b^{n+1}}{(n+1)!}\right| \nonumber
	\end{align}
	Note now that for $f \in C_{ae}\lp \R,\R\rp$, $\int^b_a f dx \in \lb a,b\rb \subseteq \lb 0,\infty \rp $, and $\xi \in \lb 0, \int^b_a f dx\rb$ it is the case that:
	\begin{align}
		\exp \lb \int_a^b f dx\rb = \sum^n_{i=1}\lb \frac{1}{i!} \lp \int^b_afdx\rp^i\rb +  \frac{e^{\xi}\cdot \lp \int^b_a f dx\rp^{n+1}}{(n+1)!}
	\end{align}
	And thus the triangle inequality, Proposition 2.6 in \cite{grohs2019spacetime}, and Lemma \ref{xpn_properties}, tells us that for $x = x_0 \les x_1 \les \cdots \les x_N=b$, and $\lb a,b\rb \subseteq \lb 0,\infty\rp$ that:
	\begin{align}\label{trian_ineq_exp_real_rect}
		&\left| \exp \lb \int^b_afdx\rb  - \real_{\rect} \lp \mathsf{E}^{N,h,q,\ve}_{n}\rp\lp f\lp \lb x\rb_{*,*} \rp\rp\right| \nonumber \\
		&= \left|\sum^n_{i=1}\lb \frac{1}{i!} \lp \int^b_afdx\rp^i\rb +  \frac{e^{\xi}\cdot \lp \int^b_a f dx\rp^{n+1}}{(n+1)!} - \real_{\rect}\lp \xpn^{q,\ve}_{n} \bullet \etr^{N,h} \rp \lp f\lp \lb x \rb_{*,*}\rp\rp\right| \nonumber \\
		&\les \left|\sum^n_{i=1}\lb \frac{1}{i!} \lp \int^b_afdx\rp^i\rb -  \real_{\rect}\lp \xpn^{q,\ve}_{n}\rp\lp x\rp \circ \real_{\rect}\lp\etr^{N,h} \rp\lp f\lp \lb x\rb_{*,*}\rp\rp\right| + \left| \frac{e^{\xi}\cdot \lp \int^b_a f dx\rp^{n+1}}{(n+1)!} \right| 
	\end{align}
	Note that the instantiation of $\etr^{N,h}$ is exact as it is the instantiation of an affine neural network. For notational simplicity let $\Xi = \real_{\rect} \lp \etr^{N,h}\rp \lp f\lp \lb x\rb_{*,*}\rp\rp$. Then Lemma \ref{xpn_properties} tells us that:
	\begin{align}\label{10.0.17}
		\left|\sum^n_{i=0}\lb \frac{\Xi^i}{i!}\rb - \real_{\rect} \lp \xpn^{q,\ve}_n \rp \lp \Xi \rp \right| & \les \sum^n_{i=1} \frac{1}{i!}\lp \left| \Xi \lp \Xi^{i-1} - \real_{\rect}\lp \pwr^{q,\ve}_{i-1}\rp\lp \Xi\rp\rp\right| + \ve + |\Xi|^q + \lp \mathfrak{p}_{i-1}^{\Xi}\rp^q	 \rp 
	\end{align} 
	Where for $i\in \N$, $\mathfrak{p}^{\Xi}_{i-1}$ are the family of functions defined as such:
	\begin{align}
		\mathfrak{p}^{\Xi}_1 &= \ve+1+|\Xi|^2 \nonumber\\
		\mathfrak{p}^{\Xi}_i &= \ve +\lp \mathfrak{p}_{i-1} \rp^2+|\Xi|^2
	\end{align}
	
	This then leaves us with:
	\begin{align}\label{(10.0.19)}
		\left|\sum^n_{i=0}\lb \frac{1}{i!} \lp \int^b_afdx\rp^i\rb - \sum^n_{i=0}\lb \frac{\Xi^i}{i!}\rb\right| &\les \sum_{i=0}^n\left| \lb\frac{1}{i!}  \lp \int^b_a fdx\rp^i -\frac{\Xi^i}{i!}\rb \right| \nonumber \\
		&\les \lp n+1\rp  \max_{i \in \{0,1,...,n\}}\left|\lb \frac{1}{i!}  \lp \int^b_a fdx\rp^i -\frac{\Xi^i}{i!} \rb \right|\nonumber \\
		&\les  n \cdot \max_{i \in \{1,...,n\}}\frac{1}{i!} \left|\lb  \lp \int^b_a fdx\rp^i -\Xi^i \rb \right|
	\end{align}
	
	Note that for each $i \in \{1,...,n \}$ it holds that:
	\begin{align}\label{(10.0.18)}
		\lp \int^b_a f dx\rp^i - \Xi^i =\lp \int^b_a f dx - \Xi\rp \lb\lp \int^b_a f dx\rp^{i-1} + \lp \int^b_afdx\rp^{i-2}\cdot \Xi + \cdots +\Xi^{i-1}\rb
	\end{align}
	Note that $\Xi$ and $\int^b_afdx$ differ by at most $\frac{\lp b-a \rp^3}{12N^2} f''\lp \xi \rp$ in absolute terms, and thus:
	\begin{align}
		\max \left\{ \Xi, \int^b_afdx\right\} \les \Xi + \frac{\lp b-a\rp^3}{12N^2}f''\lp \xi \rp
	\end{align}
	This then renders (\ref{(10.0.18)}) as:
	\begin{align}
		\lp \int^b_a fdx\rp^i - \Xi^i \les \frac{\lp b-a\rp^3}{12N^2}f''\lp \xi \rp \cdot i \cdot \lb \Xi + \frac{\lp b-a\rp^3}{12N^2} f''\lp \xi\rp\rb^{i-1}
	\end{align}
	Note that this also renders (\ref{(10.0.19)}) as:
	\begin{align}
		\left|\sum^n_{i=0}\lb \frac{1}{i!} \lp \int^b_afdx\rp^i\rb - \sum^n_{i=0}\lb \frac{\Xi^i}{i!}\rb\right| &\les \frac{\lp b-a\rp^3}{12N^2}f''\lp \xi \rp \cdot n^2 \cdot \lb \Xi + \frac{\lp b-a\rp^3}{12N^2} f''\lp \xi\rp\rb^{n-1}
	\end{align}
	This, the triangle inequality and (\ref{10.0.17}), then tell us for all $x \in \lb a,b\rb \subseteq \lb 0,\infty\rp$ that:
	\begin{align}
		&\left| \sum^n_{i=0} \lb \frac{1}{i!} \lp \int^b_a fdx\rp^i\rb - \real_{\rect} \lp \xpn^{q,\ve}_n \rp\lp x\rp \circ \Xi   \right| \nonumber\\
		&\les \left|\sum^n_{i=0}\lb \frac{1}{i!} \lp \int^b_afdx\rp^i\rb - \sum^n_{i=0}\lb \frac{\Xi^i}{i!}\rb\right| \nonumber+ \left|\sum^n_{i=0}\lb \frac{\Xi^i}{i!}\rb - \real_{\rect} \lp \xpn_n^{q,\ve} \rp\lp x \rp \circ  \Xi \right| \nonumber \\
		&\les \frac{\lp b-a\rp^3}{12N^2}f''\lp \xi \rp \cdot n^2 \cdot \lb \Xi + \frac{\lp b-a\rp^3}{12N^2} f''\lp \xi\rp\rb^{n-1} + \nonumber \\
		&\sum^n_{i=1} \frac{1}{i!}\lp \left| \Xi \lp \Xi^{i-1} - \real_{\rect}\lp \pwr^{q,\ve}_{i-1}\rp\lp \Xi\rp\rp\right| + \ve + |\Xi|^q + \mathfrak{p}_{i-1}^{q}	 \rp
	\end{align}
	This, applied to (\ref{trian_ineq_exp_real_rect}) then gives us that: 
	\begin{align}
		&\left| \exp \lb \int^b_afdx\rb  - \real_{\rect} \lp \mathsf{E}^{N,h,q,\ve}_n\rp\lp f\lp \lb x \rb_{*,*}\rp\rp\right| \nonumber\\
		&\les \left|\sum^n_{i=1}\lb \frac{1}{i!} \lp \int^b_afdx\rp^i\rb -  \real_{\rect}\lp \xpn^{q,\ve}_n\rp\lp x\rp \circ \real_{\rect}\lp\etr^{N,h} \rp\lp f\lp \lb x\rb_{*,*}\rp\rp\right| + \left| \frac{e^{\xi}\cdot \lp \int^b_a f dx\rp^{n+1}}{(n+1)!} \right| \nonumber \\
		&\les \frac{\lp b-a\rp^3}{12N^2}f''\lp \xi \rp \cdot n^2 \cdot \lb \Xi + \frac{\lp b-a\rp^3}{12N^2} f''\lp \xi\rp\rb^{n-1} + \nonumber \\
		&\sum^n_{i=1} \frac{1}{i!}\lp \left| \Xi \lp \Xi^{i-1} - \real_{\rect}\lp \pwr^{q,\ve}_{i-1}\rp\lp \Xi\rp\rp\right| + \ve + |\Xi|^q + \lp\mathfrak{p}_{i-1}^{\Xi}\rp^{q} 	 \rp + \left| \frac{e^{\xi}\cdot \lp \int^b_a f dx\rp^{n+1}}{(n+1)!} \right|
	\end{align}
	This proves Item (iv).
	
	Finally note that Lemma \ref{xpn_properties} tells us that:
	\begin{align}
		\wid_{\hid\lp \mathsf{E}^{N,h,q,\ve}_n\rp} \lp \mathsf{E}^{N,h,q,\ve}_n\rp &= \wid_{\hid \lp \xpn^{q,\ve}_n\rp} \lp \xpn^{q,\ve}_n\rp \nonumber\\
		&\les 24+2n
	\end{align}
\end{proof}
\begin{remark}
	We may represent the $\mathsf{E}^{N,h,q,\ve}_n$ diagrammatically as follows:
\end{remark}
\begin{figure}[h]
	\begin{center}

\tikzset{every picture/.style={line width=0.75pt}} 

\begin{tikzpicture}[x=0.75pt,y=0.75pt,yscale=-1,xscale=1]

\draw   (570,80) -- (640,80) -- (640,380) -- (570,380) -- cycle ;
\draw    (670,90) -- (642,90) ;
\draw [shift={(640,90)}, rotate = 360] [color={rgb, 255:red, 0; green, 0; blue, 0 }  ][line width=0.75]    (10.93,-3.29) .. controls (6.95,-1.4) and (3.31,-0.3) .. (0,0) .. controls (3.31,0.3) and (6.95,1.4) .. (10.93,3.29)   ;
\draw    (670,130) -- (642,130) ;
\draw [shift={(640,130)}, rotate = 360] [color={rgb, 255:red, 0; green, 0; blue, 0 }  ][line width=0.75]    (10.93,-3.29) .. controls (6.95,-1.4) and (3.31,-0.3) .. (0,0) .. controls (3.31,0.3) and (6.95,1.4) .. (10.93,3.29)   ;
\draw    (670,360) -- (642,360) ;
\draw [shift={(640,360)}, rotate = 360] [color={rgb, 255:red, 0; green, 0; blue, 0 }  ][line width=0.75]    (10.93,-3.29) .. controls (6.95,-1.4) and (3.31,-0.3) .. (0,0) .. controls (3.31,0.3) and (6.95,1.4) .. (10.93,3.29)   ;
\draw    (570,220) -- (532,220) ;
\draw [shift={(530,220)}, rotate = 360] [color={rgb, 255:red, 0; green, 0; blue, 0 }  ][line width=0.75]    (10.93,-3.29) .. controls (6.95,-1.4) and (3.31,-0.3) .. (0,0) .. controls (3.31,0.3) and (6.95,1.4) .. (10.93,3.29)   ;
\draw   (460,200) -- (530,200) -- (530,240) -- (460,240) -- cycle ;
\draw    (460,210) -- (430.7,131.87) ;
\draw [shift={(430,130)}, rotate = 69.44] [color={rgb, 255:red, 0; green, 0; blue, 0 }  ][line width=0.75]    (10.93,-3.29) .. controls (6.95,-1.4) and (3.31,-0.3) .. (0,0) .. controls (3.31,0.3) and (6.95,1.4) .. (10.93,3.29)   ;
\draw    (460,230) -- (430.39,378.04) ;
\draw [shift={(430,380)}, rotate = 281.31] [color={rgb, 255:red, 0; green, 0; blue, 0 }  ][line width=0.75]    (10.93,-3.29) .. controls (6.95,-1.4) and (3.31,-0.3) .. (0,0) .. controls (3.31,0.3) and (6.95,1.4) .. (10.93,3.29)   ;
\draw   (360,110) -- (430,110) -- (430,150) -- (360,150) -- cycle ;
\draw   (300,170) -- (430,170) -- (430,210) -- (300,210) -- cycle ;
\draw   (200,360) -- (430,360) -- (430,400) -- (200,400) -- cycle ;
\draw   (200,170) -- (270,170) -- (270,210) -- (200,210) -- cycle ;
\draw   (200,110) -- (330,110) -- (330,150) -- (200,150) -- cycle ;
\draw    (360,130) -- (332,130) ;
\draw [shift={(330,130)}, rotate = 360] [color={rgb, 255:red, 0; green, 0; blue, 0 }  ][line width=0.75]    (10.93,-3.29) .. controls (6.95,-1.4) and (3.31,-0.3) .. (0,0) .. controls (3.31,0.3) and (6.95,1.4) .. (10.93,3.29)   ;
\draw    (300,190) -- (272,190) ;
\draw [shift={(270,190)}, rotate = 360] [color={rgb, 255:red, 0; green, 0; blue, 0 }  ][line width=0.75]    (10.93,-3.29) .. controls (6.95,-1.4) and (3.31,-0.3) .. (0,0) .. controls (3.31,0.3) and (6.95,1.4) .. (10.93,3.29)   ;
\draw   (130,110) -- (170,110) -- (170,150) -- (130,150) -- cycle ;
\draw    (200,130) -- (172,130) ;
\draw [shift={(170,130)}, rotate = 360] [color={rgb, 255:red, 0; green, 0; blue, 0 }  ][line width=0.75]    (10.93,-3.29) .. controls (6.95,-1.4) and (3.31,-0.3) .. (0,0) .. controls (3.31,0.3) and (6.95,1.4) .. (10.93,3.29)   ;
\draw   (130,170) -- (170,170) -- (170,210) -- (130,210) -- cycle ;
\draw   (130,360) -- (170,360) -- (170,400) -- (130,400) -- cycle ;
\draw    (190,190) -- (200,190) -- (172,190) ;
\draw [shift={(170,190)}, rotate = 360] [color={rgb, 255:red, 0; green, 0; blue, 0 }  ][line width=0.75]    (10.93,-3.29) .. controls (6.95,-1.4) and (3.31,-0.3) .. (0,0) .. controls (3.31,0.3) and (6.95,1.4) .. (10.93,3.29)   ;
\draw    (200,380) -- (172,380) ;
\draw [shift={(170,380)}, rotate = 360] [color={rgb, 255:red, 0; green, 0; blue, 0 }  ][line width=0.75]    (10.93,-3.29) .. controls (6.95,-1.4) and (3.31,-0.3) .. (0,0) .. controls (3.31,0.3) and (6.95,1.4) .. (10.93,3.29)   ;
\draw   (30,240) -- (100,240) -- (100,280) -- (30,280) -- cycle ;
\draw    (30,260) -- (2,260) ;
\draw [shift={(0,260)}, rotate = 360] [color={rgb, 255:red, 0; green, 0; blue, 0 }  ][line width=0.75]    (10.93,-3.29) .. controls (6.95,-1.4) and (3.31,-0.3) .. (0,0) .. controls (3.31,0.3) and (6.95,1.4) .. (10.93,3.29)   ;
\draw    (130,130) -- (100.53,238.07) ;
\draw [shift={(100,240)}, rotate = 285.26] [color={rgb, 255:red, 0; green, 0; blue, 0 }  ][line width=0.75]    (10.93,-3.29) .. controls (6.95,-1.4) and (3.31,-0.3) .. (0,0) .. controls (3.31,0.3) and (6.95,1.4) .. (10.93,3.29)   ;
\draw    (130,190) -- (100.79,258.16) ;
\draw [shift={(100,260)}, rotate = 293.2] [color={rgb, 255:red, 0; green, 0; blue, 0 }  ][line width=0.75]    (10.93,-3.29) .. controls (6.95,-1.4) and (3.31,-0.3) .. (0,0) .. controls (3.31,0.3) and (6.95,1.4) .. (10.93,3.29)   ;
\draw    (460,220) -- (431.41,191.41) ;
\draw [shift={(430,190)}, rotate = 45] [color={rgb, 255:red, 0; green, 0; blue, 0 }  ][line width=0.75]    (10.93,-3.29) .. controls (6.95,-1.4) and (3.31,-0.3) .. (0,0) .. controls (3.31,0.3) and (6.95,1.4) .. (10.93,3.29)   ;
\draw    (130,380) -- (100.53,271.93) ;
\draw [shift={(100,270)}, rotate = 74.74] [color={rgb, 255:red, 0; green, 0; blue, 0 }  ][line width=0.75]    (10.93,-3.29) .. controls (6.95,-1.4) and (3.31,-0.3) .. (0,0) .. controls (3.31,0.3) and (6.95,1.4) .. (10.93,3.29)   ;

\draw (581,206.4) node [anchor=north west][inner sep=0.75pt]    {$\mathsf{Etr}{^{N}}^{h}$};
\draw (671,82.4) node [anchor=north west][inner sep=0.75pt]    {$\mathbb{R}$};
\draw (674,122.4) node [anchor=north west][inner sep=0.75pt]    {$\mathbb{R}$};
\draw (674,352.4) node [anchor=north west][inner sep=0.75pt]    {$\mathbb{R}$};
\draw (661,212.4) node [anchor=north west][inner sep=0.75pt]  [font=\LARGE]  {$\vdots $};
\draw (469,212.4) node [anchor=north west][inner sep=0.75pt]    {$\mathsf{Cpy}_{n}{}_{,}{}_{1}$};
\draw (371,115.4) node [anchor=north west][inner sep=0.75pt]    {$\mathsf{Pwr}_{0}^{q}$};
\draw (331,175.4) node [anchor=north west][inner sep=0.75pt]    {$\mathsf{Pwr}_{1}^{q}$};
\draw (311,365.4) node [anchor=north west][inner sep=0.75pt]    {$\mathsf{Pwr}_{n}^{q}$};
\draw (322,262.4) node [anchor=north west][inner sep=0.75pt]  [font=\LARGE]  {$\vdots $};
\draw (182,262.4) node [anchor=north west][inner sep=0.75pt]  [font=\LARGE]  {$\vdots $};
\draw (238,122.4) node [anchor=north west][inner sep=0.75pt]    {$\mathsf{Tun}$};
\draw (208,182.4) node [anchor=north west][inner sep=0.75pt]    {$\mathsf{Tun}$};
\draw (136,115.4) node [anchor=north west][inner sep=0.75pt]  [font=\scriptsize]  {$\frac{1}{0!} \rhd $};
\draw (141,173.4) node [anchor=north west][inner sep=0.75pt]  [font=\scriptsize]  {$\frac{1}{1!} \rhd $};
\draw (141,363.4) node [anchor=north west][inner sep=0.75pt]  [font=\scriptsize]  {$\frac{1}{n!} \rhd $};
\draw (122,262.4) node [anchor=north west][inner sep=0.75pt]  [font=\LARGE]  {$\vdots $};
\draw (41,250.4) node [anchor=north west][inner sep=0.75pt]    {$\mathsf{Cpy}_{n}{}_{,}{}_{1}$};

\end{tikzpicture}

	\end{center}
	\caption{Diagram of $\mathsf{E}^{N,h,q,\ve}_n$.}
\end{figure}
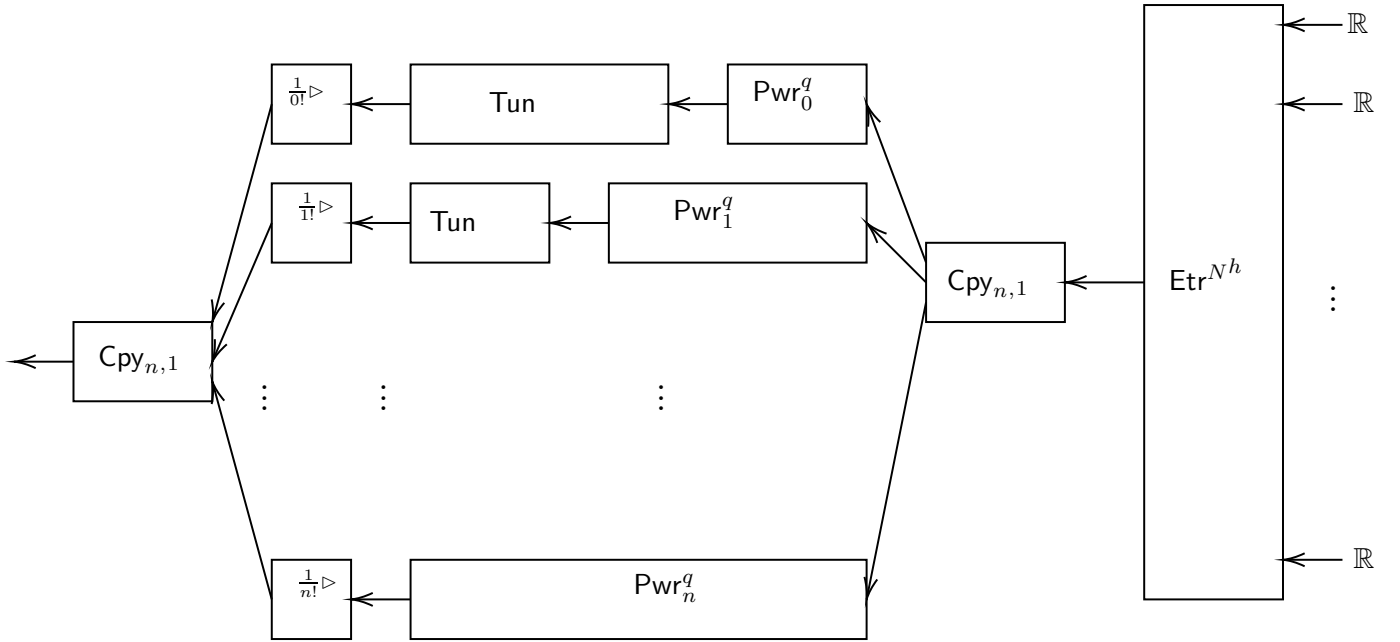

\subsection{Towards a $1$-D Interpolation Scheme}
\subsubsection{$\nrm^d$ Networks}
Note that for the following proof, the authors only contribute the parameter counts, and a more streamlined proof that maximum convolutions do indeed converge, at-least over compact domains, and atleast for Lipschitz functions. Otherwise, the proof follows that of Proposition 4.2.2 in \cite{bigbook}.
\begin{lemma}\label{9.7.2}\label{lem:nrm_prop}
	Let $d \in \N$. It is then the case that:
	\begin{enumerate}
		\item $\lay \lp \nrm^d_1 \rp = \lp d,2d,1 \rp$
		\item $\lp \real_{\rect} \lp \nrm^d_1\rp \rp \lp x \rp \in C \lp \R^d,\R \rp$
		\item that for all $x \in \R^d$ that $\lp \real_{\rect}\lp \nrm^d_1 \rp \rp \lp x \rp = \left\| x \right\|_1$
		\item it holds $\hid\lp \nrm^d_1\rp=1$
		\item it holds that $\param \lp \nrm_1^d \rp \les 7d^2$ 
		\item it holds that $\dep\lp \nrm^d_1\rp =2 $
	\end{enumerate}
\end{lemma}
\begin{proof}
	Note that by observation, it is the case that $\lay\lp \nrm^d_1 \rp = \lp  1,2,1\rp$. This tells us that for all $d \in \{2,3,...\}$ it is the case that $\lay \lp \boxminus_{i=1}^d \nrm^d_1 \rp = \lp d,2d,d\rp$. This, Proposition 2.6 in \cite{grohs2019spacetime}, and Lemma \ref{5.3.2} ensure that for all $d \in \{2,3,4,...\}$ it is the case that $\lay\lp \nrm^d_1 \rp = \lp d,2d,1 \rp$, which in turn establishes Item (i). 
	
	Notice now that (\ref{(9.7.1)}) ensures that:
	\begin{align}
		\lp \real_{\rect} \lp \nrm^d_1 \rp \rp \lp x \rp = \rect \lp x \rp + \rect \lp -x \rp = \max \{x,0 \} + \max \{ -x,0\} = \left| x \right| = \| x \|_1
	\end{align} 
	This along with Lemma 2.18 in \cite {grohs2019spacetime} tells us that for all $d \in \{2,3,4,...\}$ and $x = \lp x_1,x_2,...,x_d\rp \in \R^d$ it is the case that:
	\begin{align}
		\lp \real_{\rect} \lb \boxminus^d_{i=1} \nrm^1_1\rb\rp \lp x \rp = \lp \left| x_1 \right|, \left| x_2\right|,..., \left| x_d \right| \rp
	\end{align}
	This together with Lemma \ref{depthofcomposition} tells us that:
	\begin{align}
		\lp \real_{\rect} \lp \nrm^d_1 \rp \rp &= \lp \real_{\rect} \lp \sm_{d,1} \bullet \lb \boxminus_{i=1}^d \nrm^d_1\rb\rp \rp \lp x \rp \nonumber\\
		&= \lp \real_{\rect} \lp \sm_{d,1} \rp \rp \lp |x_1|,|x_2|,...,|x_d|\rp = \sum^d_{i=1} |x_i| =\|x\|_1
	\end{align}
	Note next that by observation $\hid\lp \nrm^1_1 \rp = 1$. Definition \ref{def:stk} then tells us that since the number of layers remains unchanged under stacking, it is then the case that $\hid \lp \nrm^1_1 \rp = \hid \lp \boxminus_{i=1}^d \nrm_1^1\rp = 1$. Note next that Lemma \ref{5.2.3} then tells us that $\hid \lp \sm_{d,1} \rp = 0$ whence Lemma \ref{comp_prop} tells us that:
	\begin{align}
		\hid \lp \nrm^d_1 \rp &= \hid \lp \sm_{d,1}\bullet \lb \boxminus_{i=1}^d \nrm^1_1 \rb \rp \nonumber \\
		&= \hid \lp \sm_{d,1} \rp + \hid \lp \lb \boxminus_{i=1}^d \nrm^1_1 \rb \rp = 0+1=1
	\end{align}

	Note next that:
	\begin{align}
			\nrm^1_1 = \lp \lp \begin{bmatrix}
				1 \\ -1
			\end{bmatrix}, \begin{bmatrix}
				0 \\ 0
			\end{bmatrix}\rp, \lp \begin{bmatrix}
				1 && 1
			\end{bmatrix}, \begin{bmatrix}
				0
			\end{bmatrix}\rp \rp \in \lp \R^{2 \times 1} \times \R^2 \rp \times \lp \R^{1 \times 2} \times \R^1 \rp
	\end{align}
	and as such $\param\lp \nrm^1_1 \rp = 7$. This, combined with Corolary 2.21 in \cite{grohs2019spacetime}, and the fact that we are stacking identical neural networks then tells us that:
	\begin{align}
		\param \lp \lb \boxminus_{i=1}^d \nrm_1^1 \rb \rp &\les 7d^2
	\end{align}
	Then Corollary 2.9 in \cite{grohs2019spacetime}, Lemma 2.4.7 in \cite{bigbook}, and Proposition 2.6 in \cite{grohs2019spacetime} tells us that:
	\begin{align}
		\param \lp \nrm^d_1 \rp &= \param \lp \sm_{d,1} \bullet \lb \boxminus_{i=1}^d \nrm_1^1 \rb\rp \nonumber \\
		&\les \param \lp \lb \boxminus_{i=1}^d \nrm_1^1 \rb \rp \les 7d^2
		\end{align}
	This establishes Item (v). 
	
	Finally, by observation $\dep \lp \nrm^1_1\rp = 2$, we are stacking the same neural network when we have $\nrm^d_1$. Stacking of equal length neural networks has no effect on depth from Definition \ref{def:stk}, and by Proposition 2.6 from \cite{grohs2019spacetime}, $\dep \lp \sm_{d,1} \bullet \lb \boxminus^d_{i=1} \nrm_1^1\rb \rp = \dep \lp \boxminus \nrm^1_1\rp$. Thus we may conclude that $\dep \lp \nrm^d_1\rp = \dep \lp \nrm_1^1\rp =2$.
	
	This concludes the proof of the lemma.
\end{proof}

\subsubsection{The $\mxm^d$ networks}

\begin{lemma}\label{9.7.4}\label{lem:mxm_prop}
	Let $d \in \N$, it is then the case that:
	\begin{enumerate}
		\item $\hid \lp \mxm^d \rp = \lceil \log_2 \lp d \rp \rceil $
		\item for all $i \in \N$ that $\wid_i \lp \mxm^d \rp \les 3 \left\lceil \frac{d}{2^i} \right\rceil$
		\item $\real_{\rect} \lp \mxm^d\rp \in C \lp \R^d, \R \rp$ and
		\item for all $x = \lp x_1,x_2,...,x_d \rp \in \R^d$ we have that $\lp \real_{\rect} \lp \mxm^d \rp \rp \lp x \rp = \max \{x_1,x_2,...,x_d \}$.
		\item $\param \lp \mxm^d \rp \les \left\lceil \lp \frac{2}{3}d^2+3d\rp \lp 1+\frac{1}{2}^{2\lp \left\lceil \log_2\lp d\rp\right\rceil+1 \rp}\rp + 1 \right\rceil$
		\item $\dep \lp \mxm^d\rp = \left\lceil \log_2 \lp d\rp \right\rceil + 1$
	\end{enumerate}
\end{lemma}
\begin{proof}
	Assume w.l.o.g. that $d > 1$. Note that (\ref{9.7.6}) ensures that $\hid \lp \mxm^d \rp = 1$. This and (\ref{5.2.5}) then tell us that for all $d \in \{2,3,4,...\}$ it is the case that:
	\begin{align}
		\hid \lp \boxminus_{i=1}^d \mxm^2\rp = \hid \lp \lb \boxminus_{i=1}^d \mxm^2 \rb \boxminus \id_1 \rp = \hid \lp \mxm^2 \rp = 1 \nonumber
	\end{align}
	This and Lemma \ref{comp_prop} tells us that for all $d \in \{3,4,5,...\}$ it holds that:
	\begin{align}\label{9.7.7}
		\hid \lp \mxm^d \rp = \hid \lp \mxm^{\left\lceil \frac{d}{2} \right\rceil}\rp + 1
	\end{align}
	And for $d \in \{4,6,8,...\}$ with $\hid \lp \mxm^{\left\lceil \frac{d}{2} \right\rceil} \rp = \left\lceil \log_2 \lp \frac{d}{2} \rp\right\rceil$ it holds that:
	\begin{align}\label{9.7.8}
		\hid \lp \mxm^d \rp = \left\lceil \log_2 \lp \frac{d}{2} \rp\right\rceil + 1 = \left\lceil \log_2 \lp d \rp -1 \right\rceil +1 = \left\lceil \log_2 \lp d \rp \right\rceil
	\end{align}
	Moreover (\ref{9.7.7})  and the fact that for all $d \in \{3,5,7,...\}$ it holds that $\left\lceil \log_2 \lp d+1 \rp \right\rceil = \left\lceil \log_2 \lp d \rp \right\rceil$ ensures that for all $d \in \{3,5,7,...\}$ with $\hid \lp \mxm^{\left\lceil \frac{d}{2}\right\rceil}\rp = \left\lceil \log_2 \lp \left\lceil \frac{d}{2} \right\rceil\rp \right\rceil$ it holds that:
	\begin{align}
		\hid \lp \mxm^d\rp &= \left\lceil \log_2 \lp \left\lceil \frac{d}{2} \right\rceil\rp \right\rceil + 1 = \left\lceil \log_2 \lp \left\lceil \frac{d+1}{2} \right\rceil\rp \right\rceil + 1 \nonumber\\
		&= \left\lceil \log_2 \lp d+1\rp-1 \right\rceil + 1 = \left\lceil \log_2 \lp d+1 \rp \right\rceil = \left\lceil \log_2 \lp d \rp \right\rceil 
	\end{align}
	This and (\ref{9.7.8}) demonstrate that for all $d \in \{3,4,5,...\}$ with $\forall k \in \{2,3,...,d-1\}: \hid \lp \mxm^d\rp = \left\lceil \log_2 \lp k \rp \right\rceil$ it holds htat $\hid \lp \mxm^d \rp = \left\lceil \log_2 \lp d \rp \right\rceil$. The fact that $\hid \lp \mxm^2 \rp =1$ and induction establish Item (i). 
	
	We next note that $\lay \lp \mxm^2 \rp = \lp 2,3,1 \rp$. This then indicates that for all $i\in \N$ that:
	\begin{align}\label{9.7.10}
		\wid_i \lp \mxm^2 \rp \les 3 = 3 \left\lceil \frac{2}{2^i} \right\rceil. 
	\end{align}
	Note then that Proposition 2.6 in \cite{grohs2019spacetime} tells us that:
	\begin{align}\label{9.7.11}
		\wid_i \lp \mxm^{2d} \rp = \begin{cases}
			3d &:i=1 \\
			\wid_{i-1}\lp \mxm^d \rp &:i\ges 2
		\end{cases}
	\end{align}
	And:
	\begin{align}\label{9.7.12}
		\wid_i \lp \mxm^{2d-1}\rp = \begin{cases}
			3d-1 &:i=1 \\
			\wid_{i-1}\lp \mxm^d \rp &:i \ges 2
		\end{cases}
	\end{align}
	This in turn assures us that for all $d \in \{ 2,4,6,...,\}$ it holds that:
	\begin{align}\label{9.7.13}
		\wid_1 \lp \mxm^d \rp = 3\lp \frac{d}{2} \rp \les 3 \left\lceil \frac{d}{2} \right\rceil
	\end{align}
	Moreover, note that (\ref{9.7.12}) tells us that for all $d \in \{3,5,7,...\}$ it holds that:
	\begin{align}
		\wid_1 \lp \mxm^d \rp = 3\left\lceil \frac{d}{2}\right\rceil -1 \les 3 \left\lceil \frac{d}{2} \right\rceil
	\end{align}
	This and (\ref{9.7.13}) shows that for all $d \in \{2,3,...\}$ it holds that:
	\begin{align}\label{9.7.15}
		\wid_1 \lp \mxm^d\rp \les 3 \left\lceil \frac{d}{2}\right\rceil
	\end{align}
	Additionally note that (\ref{9.7.11}) demonstrates that for all $d \in \{ 4,6,8,...\}$, $i \in \{2,3,...\}$ with $\wid_{i-1} \lp \mxm^{\frac{d}{2}} \rp \les 3 \left\lceil \lp \frac{d}{2}\rp \frac{1}{2^{i-1}}\right\rceil$ it holds that:
	\begin{align}\label{9.7.16}
		\wid_i \lp \mxm^d \rp = \wid_{i-1}\lp \mxm^{\frac{d}{2}}\rp \les 3 \left\lceil \lp \frac{d}{2}\rp \frac{1}{2^{i-1}} \right\rceil = 3 \left\lceil \frac{d}{2^i} \right\rceil 
	\end{align}
	Furthermore note also the fact that for all $d \in \{3,5,7,...\}$, $i \in \N$ it holds that $\left\lceil \frac{d+1}{2^i} \right\rceil = \left\lceil \frac{d}{2^i}\right\rceil$ and (\ref{9.7.12}) assure that for all $d \in \{3,5,7,...\}$, $i\in \{2,3,...\}$ with $\wid_{i-1} \lp \mxm^{\left\lceil \frac{d}{2}\right\rceil}\rp \les 3 \left\lceil \left\lceil \frac{d}{2}\right\rceil \frac{1}{2^{i-1}}\right\rceil$ it holds that:
	\begin{align}
		\wid_i \lp \mxm^d \rp = \wid_{i-1} \lp \mxm^{\left\lceil \frac{d}{2}\right\rceil}\rp \les 3 \left\lceil \left\lceil \frac{d}{2} \right\rceil \frac{1}{2^{i-1}} \right\rceil = 3 \left\lceil \frac{d+1}{2^i}\right\rceil = 3 \left\lceil \frac{d}{2^i} \right\rceil 
	\end{align}
	This and (\ref{9.7.16}) tells us that for all $d \in \{3,4,...\}$, $i \in \{2,3,...\}$ with $\forall k \in \{2,3,...,d-1\}$, $j \in \{1,2,...,i-1\}: \wid_j \lp \mxm^k \rp \les 3 \left\lceil \frac{k}{2^j} \right\rceil$ it holds that:
	\begin{align}
		\wid_i \lp \mxm^d \rp \les 3 \left\lceil \frac{d}{2^i}\right\rceil
	\end{align}
	This, combined with (\ref{9.7.10}), (\ref{9.7.15}), with induction establishes Item (ii). 
	
	Next observe that (\ref{9.7.6}) tells that for $x = \begin{bmatrix}
		x_1 \\ x_2
	\end{bmatrix} \in \R^2$ it becomes the case that:
	\begin{align}
		\lp\real_{\rect} \lp \mxm^2 \rp \rp \lp x \rp &= \max \{x_1-x_2,0\} + \max\{x_2,0 \} - \max\{ -x_2,0\} \nonumber \\
		&= \max \{x_1-x_2,0\} + x_2 = \max\{x_1,x_2\}
	\end{align}
	Note next that Lemma 2.2.7 in \cite{bigbook}, Proposition 2.6 in \cite{grohs2019spacetime}, and Proposition~2.19 in \cite{grohs2019spacetime} then imply for all $d \in \{2,3,4,...\}$, $x = \{x_1,x_2,...,x_d\} \in \R^d$ it holds that $\lp \real_{\rect} \lp \mxm^d \rp \rp \lp x \rp \in C \lp \R^d,\R \rp$. and $\lp \real_{\rect} \lp \mxm^d \rp \rp \lp x \rp = \max\{ x_1,x_2,...,x_d \}$. This establishes Items (iii)-(iv).
	
	Consider now the fact that Item (ii) implies that the layer architecture forms a geometric series whence we have that the number of bias parameters is bounded by:
	\begin{align}
		\frac{\frac{3d}{2} \lp 1 - \lp \frac{1}{2} \rp^{\left\lceil \log_2 \lp d\rp\right\rceil +1} \rp }{\frac{1}{2}} &= 3d \lp 1 -  \frac{1}{2}^{\left\lceil \log_2 \lp d \rp \right\rceil +1}\rp \nonumber \\
		&\les \left\lceil 3d \lp 1 -  \frac{1}{2}^{\left\lceil \log_2 \lp d \rp \right\rceil +1}\rp \right\rceil
	\end{align}
	For the weight parameters, consider the fact that our widths follow a geometric series with ratio $\frac{1}{2}$, and considering that we have an upper bound for the number of hidden layers, and the fact that $\wid_0 \lp \mxm^d\rp = d$, would then tell us that the number of weight parameters is bounded by: 
	\begin{align}
		&\sum^{\left\lceil \log_2\lp d\rp\right \rceil}_{i=0} \lb \lp \frac{1}{2}\rp ^i \cdot \wid_0\lp \mxm^d\rp \cdot \lp \frac{1}{2}\rp^{i+1}\cdot \wid_0 \lp \mxm^d\rp\rb \nonumber \\
		&= \sum^{\left\lceil \log_2\lp d\rp\right\rceil}_{i=0} \lb \lp \frac{1}{2}\rp^{2i+1}\lp \wid_0 \lp \mxm^d\rp\rp^2\rb \nonumber \\
		&= \frac{1}{2} \sum^{\left\lceil \log_2 \lp d\rp \right\rceil}_{i=0} \lb \lp \lp  \frac{1}{2}\rp^{i} \wid_0\lp \mxm^d\rp\rp^2\rb 
		= \frac{1}{2} \sum^{\left\lceil \log_2\lp d\rp\right\rceil}_{i=0} \lb \lp \frac{1}{4}\rp^id^2\rb
	\end{align}
	Notice that this is a geometric series with ratio $\frac{1}{4}$, which would then reveal that:
	\begin{align}
		\frac{1}{2} \sum^{\left\lceil \log_2\lp d\rp\right\rceil}_{i=0} \lb \lp \frac{1}{4}\rp^id^2\rb \les \frac{2}{3} d^2\lp 1- \frac{1}{2}^{2\lp \left\lceil \log_2(d)\right\rceil + 1\rp}\rp 
	\end{align}
	
	Thus, we get that:
	\begin{align}
		\param \lp \mxm^d\rp &\les \frac{2}{3} d^2\lp 1- \frac{1}{2}^{2\lp \left\lceil \log_2(d)\right\rceil \rp + 1}\rp + \left\lceil 3d \lp 1 -  \frac{1}{2}^{\left\lceil \log_2 \lp d \rp \right\rceil +1}\rp \right\rceil \nonumber\\
		&\les \frac{2}{3} d^2\lp 1- \frac{1}{2}^{2\lp \left\lceil \log_2(d)\right\rceil \rp + 1}\rp + \left\lceil 3d \lp 1 -  \frac{1}{2}^{2\lp\left\lceil \log_2 \lp d \rp \right\rceil +1\rp}\rp\right\rceil\\
		&\les \left\lceil \lp \frac{2}{3}d^2+3d\rp \lp 1+\frac{1}{2}^{2\lp \left\lceil \log_2\lp d\rp\right\rceil+1 \rp}\rp + 1 \right\rceil
	\end{align}	
	
	This proves Item (v).

	Item (vi) is a straightforward consequence of Item (i). 	This completes the proof of the lemma.
\end{proof}

\begin{remark}
	Diagrammatically, this can be represented as in Figure \ref{mxm_fig}.
	\begin{figure}[h]\label{mxm_fig}
		\begin{center}

\tikzset{every picture/.style={line width=0.75pt}} 

\begin{tikzpicture}[x=0.75pt,y=0.75pt,yscale=-1,xscale=1]

\draw   (560,138) -- (630,138) -- (630,178) -- (560,178) -- cycle ;
\draw   (560,206) -- (630,206) -- (630,246) -- (560,246) -- cycle ;
\draw   (562,274) -- (632,274) -- (632,314) -- (562,314) -- cycle ;
\draw   (565,350) -- (635,350) -- (635,390) -- (565,390) -- cycle ;
\draw   (568,425) -- (638,425) -- (638,465) -- (568,465) -- cycle ;
\draw   (438,175) -- (508,175) -- (508,215) -- (438,215) -- cycle ;
\draw   (439,310) -- (509,310) -- (509,350) -- (439,350) -- cycle ;
\draw   (302,234) -- (372,234) -- (372,274) -- (302,274) -- cycle ;
\draw    (437,196.5) -- (374.51,251.18) ;
\draw [shift={(373,252.5)}, rotate = 318.81] [color={rgb, 255:red, 0; green, 0; blue, 0 }  ][line width=0.75]    (10.93,-3.29) .. controls (6.95,-1.4) and (3.31,-0.3) .. (0,0) .. controls (3.31,0.3) and (6.95,1.4) .. (10.93,3.29)   ;
\draw    (559,155.5) -- (509.66,188.88) ;
\draw [shift={(508,190)}, rotate = 325.92] [color={rgb, 255:red, 0; green, 0; blue, 0 }  ][line width=0.75]    (10.93,-3.29) .. controls (6.95,-1.4) and (3.31,-0.3) .. (0,0) .. controls (3.31,0.3) and (6.95,1.4) .. (10.93,3.29)   ;
\draw    (558,224.5) -- (512.81,203.35) ;
\draw [shift={(511,202.5)}, rotate = 25.08] [color={rgb, 255:red, 0; green, 0; blue, 0 }  ][line width=0.75]    (10.93,-3.29) .. controls (6.95,-1.4) and (3.31,-0.3) .. (0,0) .. controls (3.31,0.3) and (6.95,1.4) .. (10.93,3.29)   ;
\draw    (560,290.5) -- (510.66,323.88) ;
\draw [shift={(509,325)}, rotate = 325.92] [color={rgb, 255:red, 0; green, 0; blue, 0 }  ][line width=0.75]    (10.93,-3.29) .. controls (6.95,-1.4) and (3.31,-0.3) .. (0,0) .. controls (3.31,0.3) and (6.95,1.4) .. (10.93,3.29)   ;
\draw    (566,372.5) -- (511.72,340.03) ;
\draw [shift={(510,339)}, rotate = 30.89] [color={rgb, 255:red, 0; green, 0; blue, 0 }  ][line width=0.75]    (10.93,-3.29) .. controls (6.95,-1.4) and (3.31,-0.3) .. (0,0) .. controls (3.31,0.3) and (6.95,1.4) .. (10.93,3.29)   ;
\draw    (437,331.5) -- (373.38,264.95) ;
\draw [shift={(372,263.5)}, rotate = 46.29] [color={rgb, 255:red, 0; green, 0; blue, 0 }  ][line width=0.75]    (10.93,-3.29) .. controls (6.95,-1.4) and (3.31,-0.3) .. (0,0) .. controls (3.31,0.3) and (6.95,1.4) .. (10.93,3.29)   ;
\draw   (445,425) -- (515,425) -- (515,465) -- (445,465) -- cycle ;
\draw   (301,367) -- (371,367) -- (371,407) -- (301,407) -- cycle ;
\draw   (163,296) -- (233,296) -- (233,336) -- (163,336) -- cycle ;
\draw    (302,252.5) -- (236.47,313.14) ;
\draw [shift={(235,314.5)}, rotate = 317.22] [color={rgb, 255:red, 0; green, 0; blue, 0 }  ][line width=0.75]    (10.93,-3.29) .. controls (6.95,-1.4) and (3.31,-0.3) .. (0,0) .. controls (3.31,0.3) and (6.95,1.4) .. (10.93,3.29)   ;
\draw    (299,390.5) -- (235.38,323.95) ;
\draw [shift={(234,322.5)}, rotate = 46.29] [color={rgb, 255:red, 0; green, 0; blue, 0 }  ][line width=0.75]    (10.93,-3.29) .. controls (6.95,-1.4) and (3.31,-0.3) .. (0,0) .. controls (3.31,0.3) and (6.95,1.4) .. (10.93,3.29)   ;
\draw    (162,315.5) -- (108,315.5) ;
\draw [shift={(106,315.5)}, rotate = 360] [color={rgb, 255:red, 0; green, 0; blue, 0 }  ][line width=0.75]    (10.93,-3.29) .. controls (6.95,-1.4) and (3.31,-0.3) .. (0,0) .. controls (3.31,0.3) and (6.95,1.4) .. (10.93,3.29)   ;
\draw    (567,447.5) -- (518,447.5) ;
\draw [shift={(516,447.5)}, rotate = 360] [color={rgb, 255:red, 0; green, 0; blue, 0 }  ][line width=0.75]    (10.93,-3.29) .. controls (6.95,-1.4) and (3.31,-0.3) .. (0,0) .. controls (3.31,0.3) and (6.95,1.4) .. (10.93,3.29)   ;
\draw    (443,448.5) -- (373.61,397.68) ;
\draw [shift={(372,396.5)}, rotate = 36.22] [color={rgb, 255:red, 0; green, 0; blue, 0 }  ][line width=0.75]    (10.93,-3.29) .. controls (6.95,-1.4) and (3.31,-0.3) .. (0,0) .. controls (3.31,0.3) and (6.95,1.4) .. (10.93,3.29)   ;
\draw    (437,342) -- (373.67,383.41) ;
\draw [shift={(372,384.5)}, rotate = 326.82] [color={rgb, 255:red, 0; green, 0; blue, 0 }  ][line width=0.75]    (10.93,-3.29) .. controls (6.95,-1.4) and (3.31,-0.3) .. (0,0) .. controls (3.31,0.3) and (6.95,1.4) .. (10.93,3.29)   ;

\draw (574,150.4) node [anchor=north west][inner sep=0.75pt]    {$\mathsf{Mxm}^{2}$};
\draw (574,214.4) node [anchor=north west][inner sep=0.75pt]    {$\mathsf{Mxm}^{2}$};
\draw (576,283.4) node [anchor=north west][inner sep=0.75pt]    {$\mathsf{Mxm}^{2}$};
\draw (579,358.4) node [anchor=north west][inner sep=0.75pt]    {$\mathsf{Mxm}^{2}$};
\draw (585,428.4) node [anchor=north west][inner sep=0.75pt]    {$\mathsf{Mxm}^{2}$};
\draw (453,185.4) node [anchor=north west][inner sep=0.75pt]    {$\mathsf{Mxm}^{2}$};
\draw (456,322.4) node [anchor=north west][inner sep=0.75pt]    {$\mathsf{Mxm}^{2}$};
\draw (316,242.4) node [anchor=north west][inner sep=0.75pt]    {$\mathsf{Mxm}^{2}$};
\draw (470,434.4) node [anchor=north west][inner sep=0.75pt]    {$\mathsf{Id}_{1}$};
\draw (317,377.4) node [anchor=north west][inner sep=0.75pt]    {$\mathsf{Mxm}^{2}$};
\draw (177,305.4) node [anchor=north west][inner sep=0.75pt]    {$\mathsf{Mxm}^{2}$};

\end{tikzpicture}
		\end{center}
		\caption{Neural network diagram for $\mxm^5$.}
	\end{figure}
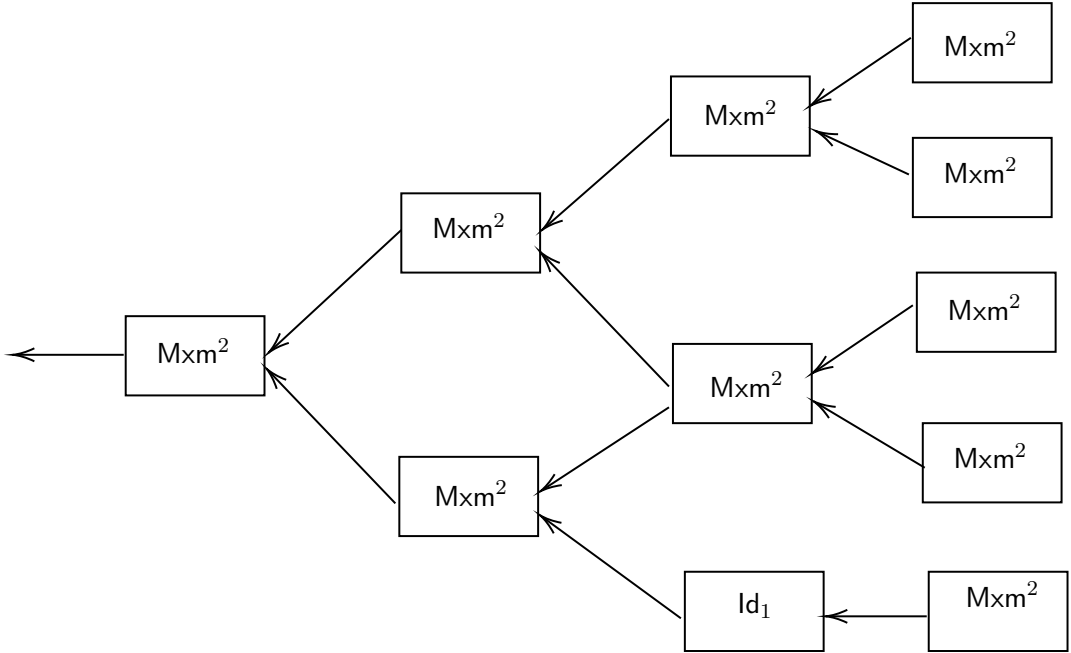
\end{remark}

\begin{lemma}\label{(9.7.5)}\label{lem:mc_prop}
	Let $d,N\in \N$, $L\in \lb 0,\infty \rp$, $x_1,x_2,\hdots, x_N \in \R^d$, $y = \lp y_1,y_2,\hdots,y_N  \rp \in \R^N$ and $\mathsf{MC} \in \neu$ satisfy that:
	\begin{align}\label{9.7.20}
		\mathsf{MC}^{N,d,L}_{x,y} = \mxm^N \bullet \aff_{-L\mathbb{I}_N,y} \bullet \lp \boxminus_{i=1}^N \lb \nrm^d_1  \bullet \aff_{\mathbb{I}_d,-x_i} \rb \rp \bullet \cpy_{N,d}
	\end{align}
	It is then the case that:
	\begin{enumerate}
		\item $\inn \lp \mathsf{MC} \rp = d$
		\item $\out\lp \mathsf{MC}^{N,d,L}_{x,y} \rp = 1$
		\item $\hid \lp \mathsf{MC}^{N,d,l}_{x,y} \rp = \left\lceil \log_2 \lp N \rp \right\rceil +1$
		\item $\wid_1 \lp \mathsf{MC}^{N,d,l}_{x,y} \rp = 2dN$
		\item for all $i \in \{ 2,3,...\}$ we have $\wid_1 \lp \mathsf{MC} \rp \les 3 \left\lceil \frac{N}{2^{i-1}} \right\rceil$
		\item it holds for all $x \in \R^d$ that $\lp \real_{\rect} \lp \mathsf{MC}^{N,d,l}_{x,y} \rp \rp \lp x \rp = \max_{i \in \{1,2,...,N\}} \lp y_i - L \left\| x-x_i \right\|_1\rp$ 
		\item it holds that $\param \lp  \mathsf{MC}^{N,d,L}_{x,y} \rp  \les \lp \frac{4}{3}N^2+3N\rp \lp 1+\frac{1}{2}^{\left\lceil \log_2\lp N\rp\right\rceil+1}\rp + 7N^2d^2 + 3\left\lceil \frac{N}{2}\right\rceil \cdot 2dN$   
	\end{enumerate}
\end{lemma}
\begin{proof}
	Throughout this proof let $\mathsf{S}_i \in \neu$ satisfy for all $i \in \{1,2,...,N\}$  that $\mathsf{S}_i = \nrm_1^d \bullet \aff_{\mathbb{I}_d,-x_i}$ and let $\mathsf{X} \in \neu$ satisfy:
	\begin{align}
		\mathsf{X} = \aff_{-L\mathbb{I}_N,y} \bullet \lp \lb \boxminus_{i=1}^N \mathsf{S}_i \rb \rp \bullet \cpy_{N,d}
	\end{align}
	Note that (\ref{9.7.20}) and Proposition 2.6 in \cite{grohs2019spacetime} tells us that $\out \lp \mathsf{MC}^{N,d,l}_{x,y} \rp = \out \lp \mxm^N \rp = 1$ and $\inn \lp \mathsf{MC} \rp = \inn \lp \cpy_{N,d} \rp =d $. This proves Items (i)\textemdash(ii). Next observe that since it is the case that $\hid \lp \cpy_{N,d} \rp$ and $\hid \lp \nrm^d_1 \rp = 1$, Proposition 2.6 in \cite{grohs2019spacetime} then tells us that:
	\begin{align}
		\hid \lp \mathsf{X} \rp = \hid \lp\aff_{-L\mathbb{I}_N,y} \rp + \hid \lp \boxminus_{i=1}^N \mathsf{S}_i\rp + \hid \lp \cpy_{N,d} \rp = 1
	\end{align}
	Thus Proposition 2.6 in \cite{grohs2019spacetime} and Lemma \ref{9.7.4} then tell us that:
	\begin{align}
		\hid \lp \mathsf{MC} \rp = \hid \lp \mxm^N \bullet \mathsf{X}\rp = \hid \lp \mxm^N \rp + \hid \lp \mathsf{X}\rp = \left\lceil \log_2 \lp N \rp \right\rceil +1
	\end{align}
	Which in turn establishes Item (iii).
	
	Note next that Proposition 2.6 in \cite{grohs2019spacetime} and Proposition 2.20 in \cite{grohs2019spacetime} tells us that:
	\begin{align}\label{8.3.33}
		\wid_1 \lp \mathsf{MC} \rp = \wid_1 \lp \mathsf{X} \rp = \wid_1 \lp \boxminus^N_{i=1} \mathsf{S}_i\rp = \sum^N_{i=1} \wid_1 \lp \mathsf{S}_i \rp = \sum^N_{i=1} \wid_1 \lp \nrm^d_1 \rp = 2dN
	\end{align}
	This establishes Item (iv). 
	
	Next observe that the fact that $\hid \lp \mathsf{X} \rp=1$, Lemma \ref{comp_prop} and Lemma \ref{9.7.4} tells us that for all $i \in \{2,3,...\}$ it is the case that:
	\begin{align}
		\wid_i \lp \mathsf{MC} \rp = \wid_{i-1} \lp \mxm^N \rp \les 3 \left\lceil \frac{N}{2^{i-1}} \right\rceil
	\end{align}
	This establishes Item (v).
	
	Next observe that Lemma \ref{lem:nrm_prop} tells us that for all $x \in \R^d$, $i \in \{1,2,...,N\}$ it holds that:
	\begin{align}
		\lp \real_{\rect} \lp \mathsf{MC} \rp \rp \lp x \rp - \lp \real_{\rect}\lp \nrm^d_1 \rp \circ \real_{\rect}\lp \aff_{\mathbb{I}_d,-x_i} \rp \rp \lp x \rp = \left\| x-x_i \right\|_1
	\end{align}
	This and Proposition 2.20 in \cite{grohs2019spacetime} combined establishes that for all $x \in \R^d$ it holds that:
	\begin{align}
		\lp \real_{\rect} \lp \lb \boxminus_{i=1}^N \mathsf{S}_i \rb \bullet \cpy_{N,d} \rp \rp \lp x \rp = \lp \| x-x_1 \|_1, \|x-x_2\|_1,...,\|x-x_N\|_1\rp \nonumber \\
	\end{align}
	This Proposition 2.6 in. \cite{grohs2019spacetime} and Lemma 2.3.2 in \cite{bigbook} establishes that for all $x \in \R^d$ it holds that:
	\begin{align}
		\lp \real_{\rect}\lp \mathsf{X}\rp \rp \lp x \rp &= \lp \real_{\rect}\lp \aff_{-L\mathbb{I}_N,y}\rp\rp \circ \lp\real_{\rect} \lp \lb \boxminus_{i=1}^N \mathsf{S}_i\rb \bullet \cpy_{N,d}\rp \rp \lp x \rp \nonumber\\
		&= \lp y_1-L \|x-x_1 \|, y_2-L\|x-x_2\|,...,y_N-L \| x-x_N \|_1\rp
	\end{align}
	Then Proposition 2.6 in \cite{grohs2019spacetime} and Lemma \ref{9.7.4} tells us that for all $x\in \R^d$ it holds that:
	\begin{align}
		\lp \real_{\rect} \lp \mathsf{MC} \rp \rp \lp x \rp &= \lp \real_{\rect}\lp \mxm^N \rp \circ \lp \real_{\rect}\lp \mathsf{X} \rp \rp \rp \lp x \rp \nonumber \\
		&= \lp \real_{\rect}\lp \mxm^N \rp \rp \lp y_1-L \|x-x_1\|_1,y_2-L\|x-x_2\|_1,...,y_N-L\|x-x_N\|_1\rp  \nonumber\\
		&=\max_{i\in \{1,2,...,N\} } \lp y_i - L \|x-x_i\|_1\rp
	\end{align}
	This establishes Item (vi).
	
	For Item (vii) note that Lemma \ref{lem:nrm_prop}, Definition \ref{def:stk}, Lemma \ref{lem:nrm_prop}, and Corollary 2.9 in \cite{grohs2019spacetime} tells us that for all $d\in \N$ and $x \in \R^d$ it is the case that:
	\begin{align}
		\param \lp \nrm^d_1\bullet \aff_{\mathbb{I}_d, -x}\rp \les \param \lp \nrm_1^d\rp \les 7d^2
	\end{align} 
	This, along with Corollary 2.21 in \cite{grohs2019spacetime}, and because we are stacking identical neural networks, then tells us that for all $N \in \N$, it is the case that:
	\begin{align}
		\param \lp \boxminus_{i=1}^N \lb \nrm^d_1\bullet \aff_{\mathbb{I}_d, -x} \rb\rp \les 7N^2d^2
	\end{align}
	Observe next that Corollary 2.9 in \cite{grohs2019spacetime} tells us that for all $d,N \in \N$ and $x \in \R^d$ it is the case that:
	\begin{align}\label{8.3.38}		
	\param \lp \lp \boxminus^N_{i=1} \lb \nrm^d_1 \bullet \aff_{\mathbb{I}_d,-x}\rb\rp \bullet \cpy_{N,d}\rp \les \param \lp \boxminus_{i=1}^N \lb \nrm^d_1\bullet \aff_{\mathbb{I}_d, -x} \rb\rp \les 7N^2d^2  
	\end{align}
	Now, let $d,N \in \N$, $L \in [0,\infty)$, let $x_1,x_2,\hdots, x_N \in \R^d$ and let $y = \{y_1,y_2,\hdots, y_N \} \in \R^N$. Observe that again, Corollary 2.9 in \cite{grohs2019spacetime}, and (\ref{8.3.38}) tells us that:
	\begin{align}
		\param\lp \aff_{-L\mathbb{I}_N,y} \bullet \lp \boxminus_{i=1}^N \lb \nrm^d_1  \bullet \aff_{\mathbb{I}_d,-x_i} \rb \rp \bullet \cpy_{N,d}\rp \nonumber\\ \les \param \lp \boxminus_{i=1}^N \lb \nrm^d_1\bullet \aff_{\mathbb{I}_d, -x} \rb\rp \les 7N^2d^2 \nonumber
	\end{align}
	Finally Proposition 2.6 in \cite{grohs2019spacetime}, (\ref{8.3.33}), and Lemma \ref{lem:mxm_prop} yields that:
	\begin{align}
		\param(\mathsf{MC}) &= \param \lp \mxm^N \bullet \aff_{-L\mathbb{I}_N,y} \bullet \lp \boxminus_{i=1}^N \lb \nrm^d_1  \bullet \aff_{\mathbb{I}_d,-x_i} \rb \rp \bullet \cpy_{N,d} \rp  \nonumber\\
		&\les \param \lp \mxm^N \bullet \lp \boxminus_{i=1}^N \lb \nrm^d_1\bullet \aff_{\mathbb{I}_d, -x} \rb \rp \rp \nonumber\\
		&\les \param \lp \mxm^N \rp + \param \lp \lp \boxminus_{i=1}^N \lb \nrm^d_1\bullet \aff_{\mathbb{I}_d, -x} \rb\rp \rp + \nonumber\\ &\wid_1\lp \mxm^N\rp \cdot \wid_{\hid \lp \boxminus_{i=1}^N \lb \nrm^d_1\bullet \aff_{\mathbb{I}_d, -x} \rb\rp} \lp \boxminus_{i=1}^N \lb \nrm^d_1\bullet \aff_{\mathbb{I}_d, -x} \rb\rp \nonumber \\
		&\les \left\lceil \lp \frac{2}{3}d^2+3d\rp \lp 1+\frac{1}{2}^{2\lp \left\lceil \log_2\lp d\rp\right\rceil+1 \rp}\rp + 1 \right\rceil+ 7N^2d^2 + 3\left\lceil \frac{N}{2}\right\rceil \cdot 2dN
	\end{align}
\end{proof}

\begin{remark}
 	We may represent the neural network diagram for $\mxm^d$ as:
\end{remark}
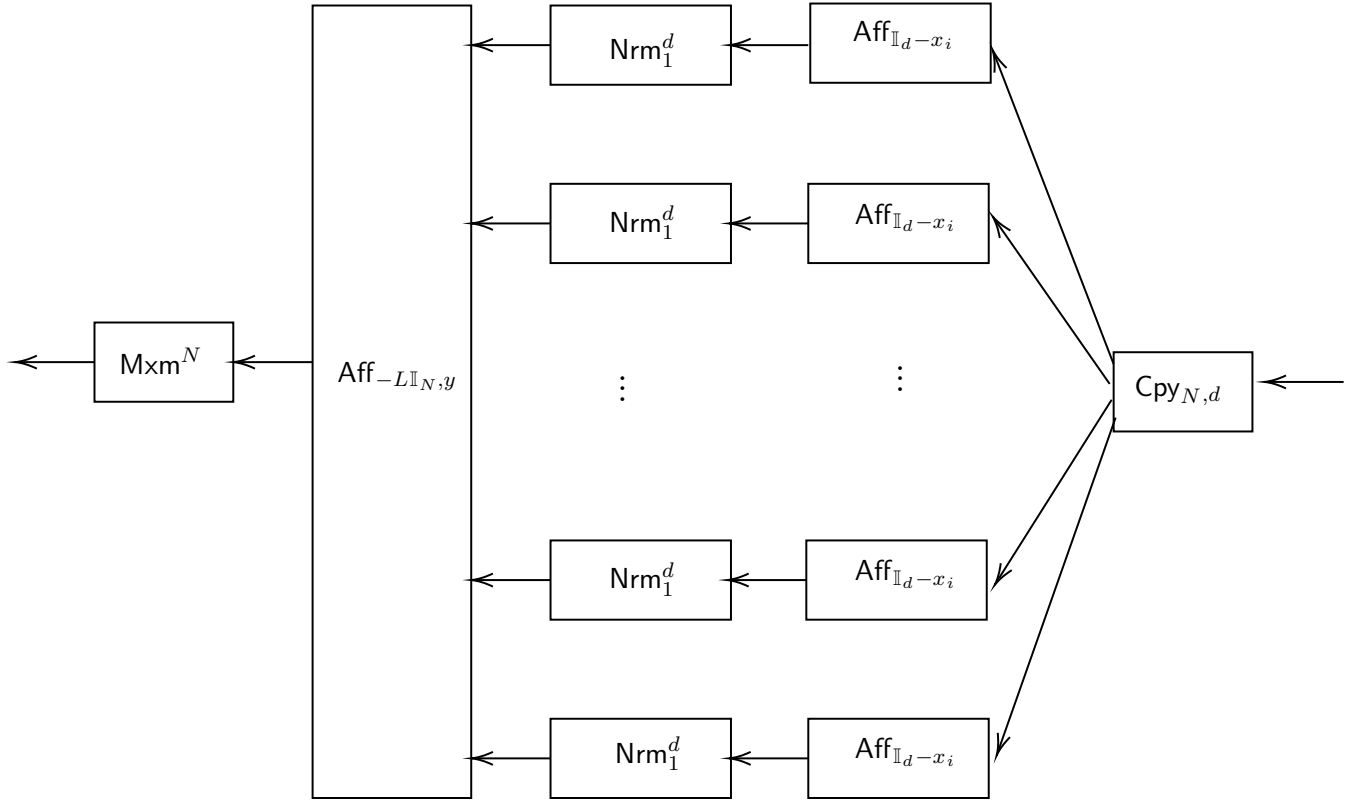
\begin{figure}[h]
	\begin{center}

\tikzset{every picture/.style={line width=0.75pt}} 

\begin{tikzpicture}[x=0.75pt,y=0.75pt,yscale=-1,xscale=1]

\draw   (574,235) -- (644,235) -- (644,275) -- (574,275) -- cycle ;
\draw    (574,241) -- (513.72,84.37) ;
\draw [shift={(513,82.5)}, rotate = 68.95] [color={rgb, 255:red, 0; green, 0; blue, 0 }  ][line width=0.75]    (10.93,-3.29) .. controls (6.95,-1.4) and (3.31,-0.3) .. (0,0) .. controls (3.31,0.3) and (6.95,1.4) .. (10.93,3.29)   ;
\draw    (572,251) -- (514.14,168.14) ;
\draw [shift={(513,166.5)}, rotate = 55.08] [color={rgb, 255:red, 0; green, 0; blue, 0 }  ][line width=0.75]    (10.93,-3.29) .. controls (6.95,-1.4) and (3.31,-0.3) .. (0,0) .. controls (3.31,0.3) and (6.95,1.4) .. (10.93,3.29)   ;
\draw    (573,259) -- (515.07,350.81) ;
\draw [shift={(514,352.5)}, rotate = 302.25] [color={rgb, 255:red, 0; green, 0; blue, 0 }  ][line width=0.75]    (10.93,-3.29) .. controls (6.95,-1.4) and (3.31,-0.3) .. (0,0) .. controls (3.31,0.3) and (6.95,1.4) .. (10.93,3.29)   ;
\draw    (575,268) -- (515.66,436.61) ;
\draw [shift={(515,438.5)}, rotate = 289.39] [color={rgb, 255:red, 0; green, 0; blue, 0 }  ][line width=0.75]    (10.93,-3.29) .. controls (6.95,-1.4) and (3.31,-0.3) .. (0,0) .. controls (3.31,0.3) and (6.95,1.4) .. (10.93,3.29)   ;
\draw   (421,59) -- (512,59) -- (512,99) -- (421,99) -- cycle ;
\draw   (419,330) -- (510,330) -- (510,370) -- (419,370) -- cycle ;
\draw   (420,150) -- (511,150) -- (511,190) -- (420,190) -- cycle ;
\draw   (420,420) -- (511,420) -- (511,460) -- (420,460) -- cycle ;
\draw   (290,60) -- (381,60) -- (381,100) -- (290,100) -- cycle ;
\draw   (290,150) -- (381,150) -- (381,190) -- (290,190) -- cycle ;
\draw   (290,330) -- (381,330) -- (381,370) -- (290,370) -- cycle ;
\draw   (290,420) -- (381,420) -- (381,460) -- (290,460) -- cycle ;
\draw    (420,80) -- (401,80) -- (382,80) ;
\draw [shift={(380,80)}, rotate = 360] [color={rgb, 255:red, 0; green, 0; blue, 0 }  ][line width=0.75]    (10.93,-3.29) .. controls (6.95,-1.4) and (3.31,-0.3) .. (0,0) .. controls (3.31,0.3) and (6.95,1.4) .. (10.93,3.29)   ;
\draw    (420,170) -- (401,170) -- (382,170) ;
\draw [shift={(380,170)}, rotate = 360] [color={rgb, 255:red, 0; green, 0; blue, 0 }  ][line width=0.75]    (10.93,-3.29) .. controls (6.95,-1.4) and (3.31,-0.3) .. (0,0) .. controls (3.31,0.3) and (6.95,1.4) .. (10.93,3.29)   ;
\draw    (419,350) -- (400,350) -- (381,350) ;
\draw [shift={(379,350)}, rotate = 360] [color={rgb, 255:red, 0; green, 0; blue, 0 }  ][line width=0.75]    (10.93,-3.29) .. controls (6.95,-1.4) and (3.31,-0.3) .. (0,0) .. controls (3.31,0.3) and (6.95,1.4) .. (10.93,3.29)   ;
\draw    (420,440) -- (401,440) -- (382,440) ;
\draw [shift={(380,440)}, rotate = 360] [color={rgb, 255:red, 0; green, 0; blue, 0 }  ][line width=0.75]    (10.93,-3.29) .. controls (6.95,-1.4) and (3.31,-0.3) .. (0,0) .. controls (3.31,0.3) and (6.95,1.4) .. (10.93,3.29)   ;
\draw   (170,60) -- (250,60) -- (250,460) -- (170,460) -- cycle ;
\draw    (170,240) -- (132,240) ;
\draw [shift={(130,240)}, rotate = 360] [color={rgb, 255:red, 0; green, 0; blue, 0 }  ][line width=0.75]    (10.93,-3.29) .. controls (6.95,-1.4) and (3.31,-0.3) .. (0,0) .. controls (3.31,0.3) and (6.95,1.4) .. (10.93,3.29)   ;
\draw    (290,80) -- (252,80) ;
\draw [shift={(250,80)}, rotate = 360] [color={rgb, 255:red, 0; green, 0; blue, 0 }  ][line width=0.75]    (10.93,-3.29) .. controls (6.95,-1.4) and (3.31,-0.3) .. (0,0) .. controls (3.31,0.3) and (6.95,1.4) .. (10.93,3.29)   ;
\draw    (290,170) -- (252,170) ;
\draw [shift={(250,170)}, rotate = 360] [color={rgb, 255:red, 0; green, 0; blue, 0 }  ][line width=0.75]    (10.93,-3.29) .. controls (6.95,-1.4) and (3.31,-0.3) .. (0,0) .. controls (3.31,0.3) and (6.95,1.4) .. (10.93,3.29)   ;
\draw    (290,350) -- (252,350) ;
\draw [shift={(250,350)}, rotate = 360] [color={rgb, 255:red, 0; green, 0; blue, 0 }  ][line width=0.75]    (10.93,-3.29) .. controls (6.95,-1.4) and (3.31,-0.3) .. (0,0) .. controls (3.31,0.3) and (6.95,1.4) .. (10.93,3.29)   ;
\draw    (290,440) -- (252,440) ;
\draw [shift={(250,440)}, rotate = 360] [color={rgb, 255:red, 0; green, 0; blue, 0 }  ][line width=0.75]    (10.93,-3.29) .. controls (6.95,-1.4) and (3.31,-0.3) .. (0,0) .. controls (3.31,0.3) and (6.95,1.4) .. (10.93,3.29)   ;
\draw   (60,220) -- (130,220) -- (130,260) -- (60,260) -- cycle ;
\draw    (60,240) -- (22,240) ;
\draw [shift={(20,240)}, rotate = 360] [color={rgb, 255:red, 0; green, 0; blue, 0 }  ][line width=0.75]    (10.93,-3.29) .. controls (6.95,-1.4) and (3.31,-0.3) .. (0,0) .. controls (3.31,0.3) and (6.95,1.4) .. (10.93,3.29)   ;
\draw    (690,250) -- (652,250) ;
\draw [shift={(650,250)}, rotate = 360] [color={rgb, 255:red, 0; green, 0; blue, 0 }  ][line width=0.75]    (10.93,-3.29) .. controls (6.95,-1.4) and (3.31,-0.3) .. (0,0) .. controls (3.31,0.3) and (6.95,1.4) .. (10.93,3.29)   ;

\draw (583,245.4) node [anchor=north west][inner sep=0.75pt]    {$\mathsf{Cpy}_{N}{}_{,d}$};
\draw (441,66.4) node [anchor=north west][inner sep=0.75pt]    {$\mathsf{Aff}_{\mathbb{I}}{}_{_{d}}{}_{-x}{}_{_{i}}$};
\draw (442,158.4) node [anchor=north west][inner sep=0.75pt]    {$\mathsf{Aff}_{\mathbb{I}}{}_{_{d}}{}_{-x}{}_{_{i}}$};
\draw (442,338.4) node [anchor=north west][inner sep=0.75pt]    {$\mathsf{Aff}_{\mathbb{I}}{}_{_{d}}{}_{-x}{}_{_{i}}$};
\draw (442,428.4) node [anchor=north west][inner sep=0.75pt]    {$\mathsf{Aff}_{\mathbb{I}}{}_{_{d}}{}_{-x}{}_{_{i}}$};
\draw (318,72.4) node [anchor=north west][inner sep=0.75pt]    {$\mathsf{Nrm}_{1}^{d}$};
\draw (318,159.4) node [anchor=north west][inner sep=0.75pt]    {$\mathsf{Nrm}_{1}^{d}$};
\draw (318,339.4) node [anchor=north west][inner sep=0.75pt]    {$\mathsf{Nrm}_{1}^{d}$};
\draw (321,427.4) node [anchor=north west][inner sep=0.75pt]    {$\mathsf{Nrm}_{1}^{d}$};
\draw (322,237.4) node [anchor=north west][inner sep=0.75pt]  [font=\LARGE]  {$\vdots $};
\draw (462,232.4) node [anchor=north west][inner sep=0.75pt]  [font=\LARGE]  {$\vdots $};
\draw (181,238.4) node [anchor=north west][inner sep=0.75pt]    {$\mathsf{Aff}_{-L}{}_{\mathbb{I}}{}_{_{N} ,y}$};
\draw (71,231.4) node [anchor=north west][inner sep=0.75pt]    {$\mathsf{Mxm}^{N}$};

\end{tikzpicture}

	\end{center}
	\caption{Neural network diagramfor the $\mathsf{MC}_{x,y}^{N,d,l}$ network}
\end{figure}

\begin{lemma}\label{capF_lemma}
	Let $\lp E,d \rp$ be a metric space. Let $L \in \lb 0,\infty \rp$, $D \subseteq E$, $\emptyset \neq C \subseteq D$. Let $f:D \rightarrow \R$ satisfy for all $x\in D$, $y \in C$ that $\left| f(x) -f(y)\right| \les L d \lp x,y \rp$, and let $F:E \rightarrow \R \cup \{\infty\}$ satisfy for all $x\in E$ that:
	\begin{align}\label{9.7.30}
		F\lp x \rp = \sup_{y\in C} \lb f\lp y \rp - Ld\lp x,y \rp \rb
	\end{align}
	It is then the case that:
	\begin{enumerate}
		\item for all $x \in C$ that $F(x) = f(x)$
		\item it holds for all $x \in D$, that $F(x) \les f(x)$
		\item it holds for all $x\in E$ that $F\lp x \rp < \infty$
		\item it holds for all $x,y \in E$ that $\left| F(x)-F(y)\right| \les Ld\lp x,y \rp$ and,
		\item it holds for all $x \in D$ that:
		\begin{align}\label{9.7.31}
			\left| F\lp x \rp - f \lp x \rp \right| \les 2L \lb \inf_{y\in C} d \lp x,y \rp\rb
		\end{align}
	\end{enumerate}
\end{lemma} 
\begin{proof}
	The assumption that $\forall x \in D, y \in C: \left| f(x) - f(y)\right| \les Ld\lp x,y \rp$ ensures that:
	\begin{align}\label{9.7.32}
		f(y) - Ld\lp x,y\rp \les f\lp x \rp \les f(y) + Ld\lp x,y \rp 
	\end{align}
	For $x\in D$, it then renders as:
	\begin{align}\label{9.7.33}
		f(x) \ges \sup_{y \in C} \lb f(y) - Ld\lp x,y \rp \rb 
	\end{align}
	This establishes Item (i). Note that (\ref{9.7.31}) then tells us that for all $x\in C$ it holds that:
	\begin{align}
		F\lp x \rp \ges f(x) - Ld\lp x,y \rp = f\lp x \rp
	\end{align}
	This with (\ref{9.7.33}) then yields Item (i). 
	
	Note next that (\ref{9.7.32}, with $x \curvearrowleft y \text{ and } y \curvearrowleft z)$ and the triangle inequality ensure that for all $x \in E$, $y,z \in C$ it holds that:
	\begin{align}
		f(y) - Ld\lp x,y\rp \les f(z)+Ld\lp y,z \rp - Ld\lp x,y \rp \les f(z) + Ld\lp x,z \rp
	\end{align}
	We then obtain for all $x\in E, z\in C$ it holds that:
	\begin{align}
		F\lp x \rp = \sup_{y\in C} \lb f(y) - Ld\lp x,y \rp \rb \les f\lp x \rp + Ld\lp x,z \rp < \infty
	\end{align}
	This proves Item (iii). Item (iii), (\ref{9.7.30}), and the triangle inequality then shows that for all $x,y \in E$, it holds that:
	\begin{align}
		F(x) - F(y) &= \lb \sup_{v \in C} \lp f(v) - Ld\lp x,v \rp \rp \rb - \lb \sup_{w\in C} \lp f(w)-Ld\lp y,w \rp \rp\rb \nonumber \\
		&= \sup_{v \in C}\lb f(v) - Ld\lp x,v \rp -\sup_{w\in C} \lp f(w) - L d\lp  y,w  \rp \rp\rb \nonumber\\
		&\les \sup_{v \in C}\lb f(v) - Ld\lp x,v \rp - \lp f(v) - Ld\lp y,w \rp \rp\rb \nonumber\\
		&= \sup_{v\in C} \lp Ld\lp y,v \rp + Ld\lp x,v \rp -Ld\lp x,v\rp  \rp = Ld \lp x,y \rp 
	\end{align} 
	This establishes Item (v). Finally, note that Items (i) and (iv), the triangle inequality, and the assumption that $\forall x \in D, y\in C: \left| f(x) - f(y) \right| \les Ld\lp x,y \rp$ ensure that for all $x\in D$ it holds that:
	\begin{align}
		\left| F(x) - f(x) \right| &= \inf_{y\in C} \left| F(x) - F(y) +f(y) - f(x)\right| \nonumber \\
		&\les \inf_{y\in C} \lp \left| F(x) - F(y) \right| + \left| f(y) - f(x) \right|\rp \nonumber\\
		&\les \inf_{y\in C} \lp 2Ld\lp x,y \rp\rp = 2L \lb \inf_{y\in C} d \lp x,y \rp \rb
 	\end{align}
 	This establishes Item (v) and hence establishes the Lemma.
\end{proof}
\begin{corollary}\label{9.7.6.1}
	Let $\lp E,d \rp$ be a metric space, let $L \in \lb 0,\infty \rp$, $\emptyset \neq C \subseteq E$, let $f: E \rightarrow \R$ satisfy for all $x\in E$, $y \in C$ that $\left\| f(x) - f(y) \right| \les Ld \lp x,y \rp$, and let $F:E \rightarrow \R \cup \{\infty\}$ satisfy for all $x\in E$ that:
	\begin{align}
		F \lp x \rp = \sup_{y\in C} \lb f(y) - Ld \lp x,y \rp\rb
	\end{align}
	It is then the case that:
	\begin{enumerate}
		\item for all $x\in C$ that $F(x) = f(x)$
		\item for all $x\in E$ that $F(x) \les f(x)$
		\item for all $x,y \in E$ that $\left| F(x) - f(y) \right| \les L d \lp x,y \rp$ and
		\item for all $x\in E$ that: \begin{align}
			\left| F\lp x \rp - f\lp x \rp \right| \les 2L \lb \inf_{y\in C} d \lp x,y \rp \rb
		\end{align}
	\end{enumerate}
\end{corollary}
\begin{proof}
	Note that Lemma \ref{capF_lemma} establishes Items (i)\textemdash(iv).
\end{proof}

\begin{lemma}\label{max_conv_converges}
	Let $d,N \in \N$, $L \in \lb 0,\infty \rp$. Let $E \subseteq \R^d$. Let $x_1,x_2,...,x_N \in E$, let $f:E \rightarrow \R$ satisfy for all $x_1,y_1 \in E$ that $\left| f(x_1) -f(y_1)\right| \les L \left\| x_1-x_2 \right\|_1$ and let $\mathsf{MC} \in \neu$ and $y = \lp f\lp x_1 \rp, f \lp x_2 \rp,...,f\lp x_N \rp\rp$ satisfy:
	\begin{align}
		\mathsf{MC} = \mxm^N \bullet \aff_{-L\mathbb{I}_N,y} \bullet \lb \boxminus^N_{i=1} \nrm^d_1 \bullet \aff_{\mathbb{I}_d,-x_i} \rb \bullet \cpy_{N,d}
	\end{align}
	It is then the case that:
		\begin{align}\label{(9.7.42)}
			\sup_{x\in E} \left| \lp \real_{\rect}\lp \mathsf{MC} \rp \rp \lp x \rp -f\lp x \rp  \right| \les 2L \lb \sup _{x\in E} \lp \min_{i\in \{1,2,...,N\}} \left\| x-x_i\right\|_1\rp\rb
		\end{align}
\end{lemma}
\begin{proof}
	Throughout this proof let $F: \R^d \rightarrow \R$ satisfy that:
	\begin{align}\label{9.7.43}
		F\lp x \rp = \max_{i \in \{1,2,...,N\}} \lp f\lp x_i \rp- L \left\| x-x_i \right\|_1 \rp
	\end{align}
	Note then that Corollary \ref{9.7.6.1}, (\ref{9.7.43}), and the assumption that for all $x,y \in E$ it holds that $\left| f(x) - f(y)\right| \les L \left\|x-y \right\|_1$ assures that:
	\begin{align}\label{(9.7.44)}
		\sup_{x\in E} \left| F(x) - f(x) \right| \les 2L \lb \sup_{x\in E} \lp \min_{i \in \{1,2,...,N\}} \left\| x-x_i\right\|_1\rp\rb 
	\end{align}
	Then Lemma \ref{(9.7.5)} tells us that for all $x\in E$ it holds that $F(x) = \lp \real_{\rect} \lp \mathsf{MC} \rp \rp \lp x \rp$. This combined with (\ref{(9.7.44)}) establishes (\ref{(9.7.42)}).
\end{proof}
\begin{remark}
	It now follows quite straightforwardly that for a compact connected $ E \subsetneq \R$, i.e. $[a,b]\subsetneq \R$, with $N\in \N$ uniformly spaced meshpoints, Lemma \ref{max_conv_converges} implies that the supremum of the 1-norm difference over $\lb a,b\rb$, $\sup_{x\in \lb a,b\rb} \left| \lp \real_{\rect}\lp \mathsf{MC} \rp \rp \lp x \rp -f\lp x \rp  \right| \rightarrow 0$ as $N \rightarrow 0$. Analogously given $x_1,x_2,\hdots, x_N$,  where each $x_i \sim \unif \lp \lb a,b\rb\rp$, i.i.d. we see convergence in probability.
	\end{remark}

\bibliographystyle{apalike}
\bibliography{main.bib}

\end{document}